\PassOptionsToPackage{table,dvipsnames}{xcolor}
\documentclass[nonblindrev]{informs3custom}
\usepackage{lmodern}      % optional
\usepackage{amsmath}
\usepackage{amssymb}
\usepackage{bm}
\usepackage{mathtools}
\usepackage{mathrsfs}     % only if you really use \mathscr
\usepackage{color} 
\usepackage{float}
\usepackage{pgfplots}
\pgfplotsset{compat=newest} % Sets the compatibility level
\usepackage{bbm}
\usepackage{amsfonts}
\usepackage{comment}
\usepackage{import}
\usepackage{parent}
\makeatletter
\let\INFORMS@proof\proof
\let\INFORMS@endproof\endproof

\renewcommand{\proof}{\@ifnextchar[{\INFORMS@proof@i}{\INFORMS@proof@i[Proof]}}
\def\INFORMS@proof@i[#1]{\INFORMS@proof{#1}}

\renewcommand{\endproof}{\Halmos\INFORMS@endproof}
\makeatother

\TheoremsNumberedThrough     
\ECRepeatTheorems
\EquationsNumberedThrough

\makeatletter
\def\footnoterule{\kern-3\p@
  \hrule \@width 6.5in \kern 2.6\p@} % the \hrule is .4pt high
\makeatother

\usepackage[
  backend=bibtex,
  bibstyle=alphabetic,   % keep alphabetic style in the bibliography, if you like
  citestyle=authoryear,  % <-- names in citations
  maxcitenames=2,
  mincitenames=1,
  maxbibnames=99
]{biblatex}
\begin{document}
\RUNAUTHOR{Wen, Chen, Sun and Zhang}
\RUNTITLE{Policy-Aware Design of Large-Scale Factorial Experiments}
\TITLE{\bf \vspace{-2.7cm} \\Policy-Aware Design of Large-Scale Factorial Experiments \\ \vspace{-.5cm} }

\ARTICLEAUTHORS{

\AUTHOR{Xin Wen}
\AFF{Stern School of Business, New York University,  New York, NY 10012,
\EMAIL{xin.wen@stern.nyu.edu}}

\AUTHOR{Xi Chen}
\AFF{Stern School of Business, New York University,  New York, NY 10012,
\EMAIL{xc13@stern.nyu.edu}}

\AUTHOR{Will Wei Sun}
\AFF{Daniels School of Business, Purdue University,  West Lafayette, IN 47907,
\EMAIL{sun244@purdue.edu}}

\AUTHOR{Yichen Zhang}
\AFF{Daniels School of Business, Purdue University,  West Lafayette, IN 47907,
\EMAIL{zhang@purdue.edu}}
}

\date{}
\ABSTRACT{ \vspace{.5cm}

Digital firms routinely run many online experiments on shared user populations. When product decisions are compositional---such as combinations of interface elements, flows, messages, or incentives---the number of feasible interventions grows combinatorially, while available traffic remains limited. Overlapping experiments can therefore generate interaction effects that are poorly handled by decentralized A/B testing. We study how to design large-scale factorial experiments when the objective is not to estimate every treatment effect, but to identify a high-performing policy under a fixed experimentation budget. We propose a two-stage design that centralizes overlapping experiments into a single factorial problem and models expected outcomes as a low-rank tensor. In the first stage, the platform samples a subset of intervention combinations, uses tensor completion to infer performance on untested combinations, and eliminates weak factor levels using estimated marginal contributions. In the second stage, it applies sequential halving to the surviving combinations to select a final policy. We establish gap-independent simple-regret bounds and gap-dependent identification guarantees showing that the relevant complexity scales with the degrees of freedom of the low-rank tensor and the separation structure across factor levels, rather than the full factorial size. In an offline evaluation based on a product-bundling problem constructed from 100 million Taobao interactions, the proposed method substantially outperforms one-shot tensor completion and unstructured best-arm benchmarks, especially in low-budget and high-noise settings. These results show how centralized, policy-aware experimentation can make combinatorial product design operationally feasible at platform scale. 

}

\vspace{0.5cm}

\KEYWORDS{online experimentation; factorial experiments; adaptive experimentation; low-rank tensor completion;  simple regret\\}
\maketitle

\vspace{-2.1cm} 

% !TEX root = /main_INFORMS.tex 

\section{Introduction}
The traditional experimental design literature centered on agricultural and biomedical settings, where researchers typically worked with a modest number of fixed units (e.g., patients or plots of land) and a small number of treatments (e.g., active drug versus placebo) \parencites{fisher1935design, yates1937design, box2005statistics, wu2011experiments, gerber2012field}. Experimental designs were primarily evaluated by how precisely they estimated average treatment effects under these relatively low-dimensional, static conditions.

Over the past two decades, however, the practice of experimentation has shifted from the field to the digital environments. Today, large-scale experiments on websites, mobile apps, wearables, and other IT artifacts are routinely used to answer scientific questions, validate design choices, and generate empirical insights in business \parencites{ 
farias2022synthetically,
chen2023frontiers,
chen2024new, 
zhao2025pigeonhole,
zhang2025comparing,
	peng2025synthesizing}, social science \parencites{athey2017digital,brynjolfsson2019using,allcott2020welfare,levy2021social,athey2024heterogeneous, offer-westort2024battling}, and healthcare \parencites{han2018effectiveness, nahum-shani2018justintime,auerbach2019evaluating,ghose2022empowering,bundorf2024how}. This infrastructure has become particularly pervasive at major technology and e-commerce firms, which now run tens of thousands of experiments per year \parencites{gupta2019top,kohavi2020trustworthy}.
Ultimately, this shift supports data-driven decision-making, reducing reliance on “expert opinion” by allowing teams to learn directly from observed user behavior.

Modern digital experimentation differs from classical settings in several key dimensions \parencites{shi2019challenges,bojinov2022online,bojinov2023design,eichhorn2024loworder,zhao2024experimental}. First, units arrive sequentially into the experiment rather than being observed all at once. Second, outcomes are typically observed quickly after exposure. Third, the number of treatment arms can be large, reflecting the many design choices available in digital products. Fourth, the primary goal is often to find a good treatment or policy to deploy, rather than to obtain precise estimates for all competing arms \parencite{gupta2019top}. These features expand the design space and create an opportunity to adapt the experiment over time to improve its value for decision making, subject to a fixed budget of experimental units. To illustrate, consider an online marketer choosing one out of ten advertising creatives for a campaign. A traditional static experiment would allocate equal traffic to all ten arms until the sample size is reached, then pick the arm with the highest sample average outcome and compute standard errors in the usual way. Achieving reasonable precision for all arms in such a design typically requires a large number of experimental units, and substantial traffic may be spent on arms that quickly appear inferior.

Adaptive experiments address this inefficiency by updating assignment probabilities as data accrue \parencites{lai1978adaptive,lai1985asymptotically,lai1987adaptive}. Initially, units are assigned to each arm with equal probability and outcomes are observed. As evidence accumulates, the platform reduces allocation to arms that are unlikely to be optimal and reallocates traffic toward better-performing alternatives. If the goal is to identify the best treatment, sending a large fraction of additional units to clearly underperforming arms is wasteful. By progressively de-emphasizing such arms, adaptive designs can narrow the set of plausible candidates and improve both the rate of learning and the welfare of participants. At the same time, adaptive algorithms raise subtle issues about bias, variance, and inference, especially in complex, high-dimensional environments \parencites{cortez2022staggered,zhao2024experimental,xiong2024optimal,keyvanshokooh2025contextual}.

A second challenge in digital settings is the sheer number of potential interventions. Digital products are often composed of multiple components, each of which can vary along several dimensions. For example, an email or text-message campaign can vary in style, tone, theme, imagery, behavioral tactics, subject line, and timing. It is typically cheap to generate a very large number of intervention variants by combining different component choices. In such settings, the binding resource is not the number of possible designs, but the amount of user traffic available for experimentation. The design space is high-dimensional and combinatorial, while the experimental budget (number of users) is limited.

This scarcity of traffic creates a bottleneck at the industry level, where the demand for estimating treatment effects often leads to the same user being enrolled in multiple experiments concurrently \parencites{tang2010overlapping,kohavi2017surprising}. Consider a standard e-commerce interface. A single ``checkout'' page is not one decision, but a composition of button colors, payment flows, promotional popups, and layout options. Instead of a simple A/B test, the real-world design space is combinatorial: varying just 10 colors, 5 flows, 6 popups, and 4 layouts creates a grid of $10 \times 5 \times 6 \times 4$ yielding 1,200 potential combinations \parencites{tang2010overlapping, kohavi2017surprising}. When multiple teams want to test features that affect similar UI components, their experiments may need to be serialized to avoid interference, which delays learning and product iteration. When experiments are run concurrently because interactions are believed to be weak, interaction effects nonetheless arise and can be common in practice. This raises fundamental questions: How should we handle interaction between treatments in concurrent experiments? How can we aggregate and share information across experiments to improve decisions \parencite{gupta2019top}?

Recent empirical work shows that co-occurring experiments can exhibit statistically significant interactions and that ignoring them can materially bias estimated average treatment effects (ATEs). \cite{abbasi2025critical}, for example, use a case study from a major e-commerce company to document that a large proportion of concurrent experiments have significant interactions, often making focal ATEs appear more positive than they actually are when interactions are accounted for. Measured ATEs for a given test can change substantially, including sign reversals, as the surrounding experimental ecosystem changes \parencites{gupta2019top, abbasi2025critical}. Despite running many experiments, firms still report high failure rates: more than half of ideas do not produce meaningful improvements, and a substantial fraction of experiments fail to meet their predefined success criteria \parencites{kohavi2020trustworthy, kohavi2022testing}. Because A/B testing is often used for one-shot “ship or kill” decisions, such failures delay learning and product iteration. Practitioners therefore seek decision-support tools that can leverage information across experiments and support prediction and design before launch \parencites{kohavi2020trustworthy, kohavi2022testing}.

Several methodological traditions speak to these challenges. Since the seminal work of \cite{fisher1935design} and \cite{yates1937design}, factorial designs have been widely used in agricultural, industrial, biomedical, and social science applications because they allow simultaneous estimation of multiple factors and their interactions \parencite{wu2011experiments}. Classical factorial experiments, however, typically involve a small number of factors ($K \leq 4$), so that the $2^K-1$ main and interaction effects can be estimated with feasible sample sizes. As the number of factors grows, the number of treatment combinations explodes, motivating forward-selection strategies and fractional factorial designs that target a limited subset of effects \parencite{shi2025forward}. These approaches rely on sparsity assumptions: only a small number of factorial effects are important. In digital product settings, where many subtle interactions may matter, such sparsity assumptions are often difficult to justify. An alternative approach places functional form restrictions, such as additivity, on how outcomes depend on product characteristics. Under such assumptions, researchers can randomize multiple dimensions independently and model expected outcomes as a function of the resulting features. These “factorial” experiments enable exploration of a wide range of treatments by treating variation in other dimensions as background noise. However, allowing many dimensions to vary increases the variance of estimated effects for any single dimension. 

Classical factorial designs thus clarify why interactions matter but are operationally infeasible at platform scale and misaligned with the primary decision objective. For firms, the goal of experimentation is rarely the precise estimation of all main and interaction effects; it is to improve decisions in a rapidly evolving environment. In such decision-focused settings, the common emphasis on $p$-values and statistical significance is often inappropriate \parencite{wasserstein2016asa}. Instead, platforms need experimental designs that explicitly target the selection of high-performing policies under resource constraints and pervasive interactions. In the digital economy, the limiting factor for innovation is no longer the number of ideas, but the sampling cost of evaluating their combinations. These practical tensions motivate our central research question: 
\[
\begin{gathered}
\text{\itshape How can firms design large-scale factorial experiments that move beyond parameter estimation}\\
\text{\itshape to directly target optimal policy selection under resource constraints?}
\end{gathered}
\]

Taken together, this paper makes four main contributions.

First, we shift the mechanism of learning and introduce a policy-aware perspective on large-scale factorial experimentation in digital platforms, arguing that overlapping A/B tests on shared components should be centralized and modeled as a single high-dimensional design space. This reframing makes interaction effects a design feature rather than a nuisance or confounder and connects platform experimentation to the literature on low-rank tensor representations.

Second, we propose a two-stage ``centralize and then randomize'' design that exploits low-rank structure to screen factor levels using tensor completion methods and then refines the choice among surviving combinations via a standard best-arm identification procedure. It enables inference on unobserved treatments, predicting the performance of treatment combinations that were never fielded in the experiment with a theoretical guarantee. The design uses factor-level marginal contributions as a natural object for elimination and explicitly targets simple regret rather than conventional hypothesis testing. Unlike static designs, our approach is adaptive: the decision to sample or eliminate interventions is not fixed in advance but updates dynamically as structure is learned.

Third, we establish two complementary theoretical guarantees for the proposed policy selection
procedure: a gap-independent bound that delivers a worst-case performance baseline, and a gap-dependent bound that captures how quickly the algorithm improves when the data exhibit
clear performance separations. Both guarantees scale with the effective degrees of freedom of
the low-rank tensor, rather than the full $d^m$ design space, and the gap-dependent guarantee depends on
a factor-level separability profile: how strongly the top few levels within each factor are
distinguished by their best achievable outcomes. This distinction clarifies when structure-exploiting,
centralized experimentation yields large returns: when only a small subset of levels in each factor is
competitive, Stage~I can eliminate most of the design space with limited traffic, leaving a
dramatically smaller candidate set for final selection. In contrast, when many levels are nearly tied,
the theory correctly predicts that no method can do substantially better than unstructured search.

Fourth, we provide an empirical validation of our framework using a large-scale product bundling problem on Alibaba’s Taobao platform, demonstrating the practical utility of our ``centralize and then randomize'' approach. By modeling 1,680 potential item combinations as a high-dimensional tensor, we show that our policy-aware design overcomes the ``exploration overhead'' that causes traditional bandit algorithms to fail in resource-constrained environments. Specifically, our method exploits latent cross-category correlations to identify high-performing bundles even when the experimental budget is insufficient to sample the vast majority of combinations. These results illustrate how e-commerce managers can bypass the prohibitive costs of combinatorial testing, offering a scalable tool for decision-making in digital marketplaces.

\subsection{Relationship to Closely Related Literature}
Recent approaches to combinatorial experimentation generally manage the explosion of the design space through either explicit sparsity assumptions or non-parametric approximation.  
One stream of research \parencites{banerjee2021selecting, shyamal2025probabilistic, shi2025forward} tackles the high dimensional problem by assuming that interactions are effectively sparse, specifically, that outcomes are driven solely by main effects and bounded low-order (e.g., pairwise) interactions. While these methods allow for efficient estimation of specific marginal effects via regression or Fourier basis expansions, they typically optimize for parameter recovery (minimizing mean squared error) as a proxy for decision quality, rather than directly minimizing the opportunity cost of the final policy. Conversely, deep learning approaches \parencites{ye2023deeplearningbased, chernozhukov2018double} circumvent explicit structural constraints in favor of universal approximation, using neural networks to model arbitrary interaction surfaces. However, this flexibility comes at a cost: theoretical guarantees for these ``Double Machine Learning'' frameworks rely on the nuisance parameters (the outcome and propensity models) converging at a rate of $o(n^{-1/4})$. As noted by \cite{ye2023deeplearningbased}, achieving this rate faces severe identifiability challenges in exploration-limited settings. When the experimental budget prohibits sampling the vast majority of combinations, the data is insufficient to ground a high-capacity neural network, leading to nuisance estimation errors that propagate to the final decision.  Finally, while bandit frameworks have been adapted to factorial settings \parencites{offer-westort2024battling, muralidharan2025factorial}, they often treat the correlation structure as secondary to the selection rule. Our work bridges these gaps by treating the interaction not as a sparse nuisance or a black box, but as a structured low-rank tensor, allowing for policy optimization that is both sample-efficient and exploits latent dependencies.

Our paper also contributes to the fixed-budget pure-exploration literature. In pure exploration, the experimenter adaptively collects data and then recommends a final action for deployment. Two canonical formulations are studied in the literature: fixed-confidence \parencites{chernoff1958sequential,garivier2016optimal} and fixed-budget \parencites{audibert2010best,wang2024best}. In the fixed-confidence setting, the target error probability is specified in advance and the objective is to minimize the expected number of samples required to meet that target. In the fixed-budget setting, which is the focus here, the experimental horizon is fixed and the goal is to maximize the quality of the final recommendation. This problem is intrinsically challenging: universally asymptotically optimal policies need not exist in the fixed-budget regime, even in simple bandit models \parencite{degenne2023existence}. Our contribution is therefore not a universally optimal policy for arbitrary bandits, but a theoretically grounded design for combinatorial action spaces with latent low-rank structure. This perspective is also distinct from the existing low-rank bandit literature \parencites{kveton2017stochastic,jun2019bilinear,jang2021improved,lu2021lowrank,kang2022efficient,stojanovic2023spectral,bayati2022speed,jedra2024lowrank,lee2025gllowpopart}, which mainly studies cumulative regret during learning. In contrast, our objective is not to optimize rewards during the learning process, but to allocate the fixed sampling budget so as to maximize the probability of making a high-quality final recommendation. Thus, low-rank structure enters our analysis not through cumulative-regret minimization, but through its role in improving terminal decision quality in pure exploration.

\section{Our Solution: Centralize and Then Randomize}
\label{sec:solution}

Our core proposal is a shift from decentralized, component-wise testing to centralized, structural experimentation. In most digital firms, different product teams (e.g., the Payments Team versus the Growth Team) run parallel A/B tests on a shared user base. While logistically simple, this approach treats interaction effects as ``noise'' or ``nuisance parameters'' that can bias individual results. These interactions might be the most valuable signals for platform optimization. By centralizing these overlapping experiments, we transform the design space into a single high-dimensional intervention tensor, allowing us to exploit the structural dependencies between design choices.

\subsection{Centralizing Overlapping Experiments}

Consider checkout page optimization involving three design factors: Button Color, Payment Flow, and Coupon Placement. In current practice, these factors are often tested in isolation or serial batches. However, a user assigned to a ``Blue'' button may simultaneously encounter a ``Two-Step'' payment flow from a concurrent experiment. Their resulting behavior (e.g., a completed purchase) is a joint response to both factors. Standard A/B testing assumes these factors are independent, but in reality, they are often synergistic or antagonistic. For instance, a highly salient ``Red'' button might increase click-through rates but create friction if the subsequent payment flow is cumbersome. Centralization allows the firm to observe these cross-feature synergies explicitly. Rather than asking ``Which color is best?'' or ``Which flow is best?'', a centralized design asks: ``Which combination of color, flow, and coupon placement maximizes the platform objective?''

We formalize this decision space as an $m$-mode tensor $\mathcal{T}^\star \in \mathbb{R}^{d_1 \times \dots \times d_m}$ \parencite{kolda2006multilinear,kolda2009tensor,bi2021tensors}, where each mode represents a design factor and each index represents a level of that factor. In the checkout example, the physical axes of the experiment correspond to the factors: Axis 1 (Button Color), Axis 2 (Payment Flow), and Axis 3 (Coupon Placement). Every point in this grid is a specific intervention. In a fully flexible world, every single point is unique and independent; to know the value of any combination, one must visit it. This leads to a combinatorial explosion where the number of possible interventions grows exponentially with the number of factors ($K = \prod d_k$). Since user traffic is a finite and expensive resource, a ``brute-force'' A/B test, where every combination is tested equally, is operationally challenging. To navigate this space, the firm must rely on a structural inductive bias: the assumption that the design space is a structured landscape governed by a few underlying drivers rather than a collection of independent islands.

\subsection{Low-Rank Structure: From Additivity to Structural Synergies}\label{sec:lowrank structure}
To operationalize high-dimensional experimentation, we exploit the algebraic structure of the decision space via the concept of tensor rank \parencite{kolda2009tensor}. Formally, an $m$-mode tensor has rank 1 if it can be expressed as the outer product of $m$ vectors (e.g., $\mathcal{T}=\mathbf{u} \circ \mathbf{v} \circ \mathbf{w}$ ). In a management context, the rank of the tensor represents the number of latent behavioral mechanisms that drive performance across the platform. A low-rank approximation implies that the complex surface of customer behavior is actually generated by a small number of fundamental drivers interacting across the design factors. Below, we contrast the traditional additive model with more flexible tensor representations.

A common starting point in experimental design is the linearly additive model. For a 3-mode intervention space, this model posits:
\begin{align}
    \label{eq:additive}
    \cT^\star_{ijk} \;=\; \mu + \alpha_i + \beta_j + \gamma_k,
\end{align}
where $\mu$ is a baseline, and $\alpha, \beta, \gamma$ are independent main effects for each factor. Geometrically, this corresponds to a sum of rank-1 tensors and assumes the outcome is simply the sum of the coordinates. It projects the high-dimensional tensor onto $m$ independent 1D lines (the ``main effects''). This implies that the ``Button Color'' axis has a fixed slope that never changes, regardless of the level of Payment Flow. From a decision-making perspective, this model assumes independence of factors. It implies the value of a ``Blue button'' is a universal constant. If a ``Red Button'' increases conversion by 1\% and ``Free Shipping'' increases it by 2\%, then the model assumes combining them must yield exactly 3\%. By ignoring the psychological reality that these signals might be redundant or conflicting, a manager risks missing ``hidden gems'': combinations where the synergy creates value far beyond the sum of the parts.

The CP decomposition \parencites{carroll1970analysis} generalizes the additive model by allowing $R$ independent latent components (rank-one tensor) to interact:
\begin{align}\label{equ:cp}
    \cT^\star_{ijk} \;=\; \sum_{\ell=1}^{R} \lambda_\ell \cdot a_{i\ell} b_{j\ell} c_{k\ell}
\end{align}
In this representation, each component $\ell$ represents a global behavioral theme (e.g., ``Urgency'' or ``Trust'') and all the effects are driven by these $R$ themes.  Geometrically, the rank $R$ is the number of global latent axes that run diagonally through the data. The coordinate system is rotated so that all $m$ physical factors align along $R$ hidden directions. Instead of $d$ colors and $d$ flows, the manager considers $R$ global themes. Moving along the Urgency Axis, for example, simultaneously changes the Color (to Red), the Flow (to Fast), and the Coupon (to Limited-Time). However, the model cannot move off-axis; it assumes Color only interacts with Flow because they sit on the same shared latent coordinate. From a decision-making perspective, CP is appropriate when performance is driven by a few dominant, independent psychological channels or recipes for success. If $R=2$, the platform might be driven by two themes: an Urgency Theme (High-salience colors and fast-track checkout) and a Deliberation Theme (Subdued colors and detailed info pages). Every design factor must align with one of these global personas. If a specific feature does not fit into one of these pre-set recipes, the model effectively treats its contribution as noise.

The Tucker model \parencites{tucker1966mathematical} allows for complex, many-to-many interactions between latent factors across different modes:
\begin{align}
\label{eq:tucker}
\cT^\star_{ijk} \;=\; \sum_{p=1}^{r_1}\sum_{q=1}^{r_2}\sum_{r=1}^{r_3} \cG_{pqr} \cdot U_1(i,p) U_2(j,q) U_3(k,r)
\end{align}
Here, $U_1, U_2, U_3$ provide mode-wise bases (latent profiles), while the core tensor $\cG$ explicitly encodes how any latent profile in one mode interacts with any latent profile in another. Geometrically, the Tucker model changes the basis of each factor mode to its own internal latent space. The physical colors are projected onto $r_1$ Latent Color Axes (e.g., Visibility and Aesthetics), and payment flows are projected onto $r_2$ Latent Flow Axes (e.g., Speed and Security). The core tensor $\cG$ serves as a Rotation Matrix—an Interaction Engine—that defines how the Visibility axis of Color interacts with the Speed axis of Flow. From a decision-making perspective, Tucker is the preferred model for general experimentation because it allows for asymmetric synergies. The rank $r_n$ represents the latent dimensionality of each specific factor. For example, a Visible button might help a Speedy flow but actually hurt a Secure flow where users might perceive high visibility as untrustworthy. Unlike CP, where factors are locked into global components, Tucker allows the same latent signal to interact uniquely with different backend processes.

By exploiting these low-rank structures, a manager can explore an exponential design space ($\prod d_k$) using a linear sampling budget ($\sum d_k r_k$). This effectively de-risks innovation, allowing firms to learn the structural landscape of consumer behavior without testing every possible combination. If the tensor is low-rank, the outcome of any combination is a deterministic function of these latent levers. Even with 50 million combinations, there are only a handful of degrees of freedom. By identifying the latent coordinates of each factor using a small number of samples, a manager can mathematically predict the performance of combinations that have never been tested.

\subsection{Sampling Model and Noisy Observations}
In a high-dimensional design space, the binding constraint is the scarcity of user traffic relative to the number of possible interventions. A firm cannot afford to ``purchase'' information about every cell in the tensor. Consequently, 
the manager faces a subset selection problem: choosing an exploration portfolio $\Omega \subseteq\left[d_1\right] \times \cdots \times\left[d_m\right]$ of strategic combinations to field to users. For each selected combination $\bidx\in\Omega$, we do not observe the ground truth value $\cT^\star_{\bidx}$ directly. Instead, we observe a noisy realization based on a randomized experiment with a finite number of users. We model the observations as:
\begin{align}\label{equ:tensor_completion}
Y_{\bidx} = \cT_{\bidx}^\star + \xi_{\bidx},
\end{align}
where $\mathcal{T}^{\star}$ represents the true underlying performance tensor and $\xi_{\bidx}$ encapsulates the measurement uncertainty inherent in digital experimentation.  Since outcomes (e.g., clicks, purchases) are stochastic and users are heterogeneous, the observed performance of any design is an imperfect signal of its true long-run value. We assume the noise terms $\xi_{\bidx}$ are independent sub-Gaussian random variables with mean zero and variance bounded by $\sigma^2$. From a managerial perspective, $\sigma^2$ captures the volatility of the metric. A higher $\sigma^2$ (e.g., a noisy metric like ``Revenue per Visitor'') requires a larger sample size per cell to distinguish signal from noise.   The challenge of the experimental design is to choose $\Omega$ such that we can reconstruct the unobserved counterfactuals (the empty cells) with high fidelity, minimizing the traffic ``spent'' on exploration.

\subsection{Decision Objective: Minimizing Opportunity Cost}

A natural approach is to use the experimental budget to estimate $\cT^\star$ as accurately as possible (minimizing mean squared error) or to test for significant differences between arms. However, the ultimate goal for firms is not precise estimation of population parameters; it is to maximize the value of the final deployed design. Point estimation and hypothesis testing, the usual focus in the treatment-effects literature, are not decision rules and are often misaligned with this objective \parencites{manski2000identification,manski2002treatment,manski2004statistical, dehejia2005program}. Null hypothesis significance testing answers the wrong question (``Are these designs different?'') rather than the managerial question (``Which design should I ship?''), and often demands infeasible sample sizes to detect trivial effects \parencites{sawyer1983significance,lewis2015unfavorable,wasserstein2016asa,feit2019test,amrhein2019scientists,amrhein2019inferential,mcshane2019abandon,wasserstein2019moving, mcshane2024statistical}.

Instead, we adopt a decision-theoretic objective aligned with the firm's bottom line: simple regret \parencites{bubeck2009pure,bubeck2011pure} , also known as the Expected Opportunity Cost of the decision \parencites{chen2000simulation,chick2001new}. 
Formally, we consider a sequential decision process over a fixed horizon of $N$ rounds (the traffic budget). An experimental policy $\pi$ consists of two components:
\begin{itemize}
    \item \textbf{Allocation Rule:} At each round $t=1, \dots, N$, based on the history of previous observations $\mathcal{H}_{t-1} = \{(\bidx_s, Y_s)\}_{s=1}^{t-1}$, the policy selects a design $\bidx_t \in [d_1] \times \cdots \times [d_m]$ to test. The policy allocates traffic to acquire information. 
    \item \textbf{Recommendation Rule:} After the budget $N$ is exhausted, the policy uses the full history $\mathcal{H}_N$ to return a single design candidate $J_N$ for permanent deployment.
\end{itemize}
Let $\mu^\star := \max_{\bidx} \mathcal{T}_{\bidx}^\star$ denote the value of the optimal design, and let $\mu_{J_N} := \mathcal{T}_{J_N}^\star$ be the true expected value of the recommended design. The Simple Regret is defined as the expected optimality gap of the recommendation:
\begin{align}\label{equ:obj-1}
\SR_N
\;:=\;
\mathbb{E}_\pi\left[\mu_\star - \mu_{J_N}\right],
\end{align}
where the expectation is taken over the randomness in the sampling process and the noise outcomes.
This metric quantifies the expected revenue or welfare lost if the platform fails to identify the best design. We depart from standard cumulative regret in bandit literature, which prioritizes earning rewards during the experiment. While cumulative regret is appropriate when the experiment itself is the product (e.g., ad targeting), it conflates the cost of learning with the value of deployment. In our setting, the experiment is a transient information-gathering phase (e.g., a pilot study) intended to guide a permanent launch. We therefore optimize for post-experiment performance, accepting higher ``exploration costs'' during the pilot to ensure the long-term deployment is optimal.

Furthermore, we prefer Simple Regret over the Probability of Correct Selection (PCS) \parencite{glynn2004large}, which minimizes the binary error $\mathbb{P}(\hat{\imath} \neq i^\star)$ that only distinguishes correct vs.\ incorrect identification. While intuitive, PCS creates a misalignment in high-dimensional spaces: if the second-best design is only marginally worse than the optimum, selecting it is a statistical error but practically optimal. Simple Regret treats such errors as negligible, penalizing mistakes strictly in proportion to their economic suboptimality. With finite samples and a combinatorial design space, exact optimality is often unattainable; our goal is therefore to find a policy that minimizes the opportunity cost $\SR_N$ given a fixed traffic budget $N$.

\section{Algorithm}
A purely randomized experimental design that selects a subset of feature combinations and applies tensor completion is inherently non-adaptive: it allocates the same exploration budget to demonstrably poor variations as it does to promising ones. Operationally, this static approach is highly inefficient. If a specific feature level (e.g., a ``Red'' checkout button) is uniformly inferior across all contexts, a static design wastes a fixed fraction of user traffic on combinations containing that button until the experiment concludes.

This highlights the critical distinction between statistical estimation and managerial decision-making. While a static design might minimize global reconstruction error, a product manager's primary goal is to identify the optimal policy while minimizing regret. Given that many interventions in digital experiments are systematically inferior---with failure rates often exceeding 50\% \parencite{kohavi2020trustworthy,kohavi2022testing}---an effective algorithm must be capable of ``failing fast.'' It must discard underperforming factor levels early so that the remaining traffic can be concentrated on high-potential candidates.

A natural alternative would be an optimistic algorithm based on upper confidence bounds (UCB). We do not pursue that route for two reasons. First, our objective is fixed-budget simple regret rather than cumulative regret, so the design problem is to maximize the quality of the final recommendation rather than to optimize rewards during the learning process. Second, in our setting the reward surface is represented by a partially observed low-rank tensor under adaptive sampling. A UCB-style method would therefore require high-probability confidence sets for predicted rewards over a nonconvex tensor parameterization. Constructing confidence sets that are both sufficiently tight to guide exploration and computationally tractable at platform scale is, to our knowledge, not currently available in this setting. Our two-stage elimination design avoids this difficulty by relying instead on structural prediction in Stage I (\TStage) to screen out weak levels and on direct empirical comparison in Stage II (\VS) to identify the final policy.

\subsection{Two-stage design: tensor stage and vector stage}
In the \TStage, we begin without prior knowledge of the treatment effects. We allocate a small fraction of traffic to a randomized subset of interventions to ``cover'' the factor space. The goal here is not perfect predictive accuracy, but rather sufficient precision to confidently screen out sub-optimal factor levels. 

To formalize this elimination criterion, we define the Factor Level Marginal Contribution (FLMC). For a factor $k$ and level $i$, the FLMC is defined as the maximum estimated effect among all combinations in the current design space that contain level $i$:
\begin{align}
	\mu^{(\ell)}_{k,i}\;:=\;\max_{\bj\in \prod_{k'\neq k}\cA^{(\ell)}_{k'}} \cT^{(\ell)}_{(i,\bj)}
\end{align}
where $\cA_k^{(1)} = \{1,\dots,d\}$ and $\Aell := \cA_1^{(\ell)}\times\cdots\times\cA_m^{(\ell)}$ denotes the design space at round $\ell$.

Managerially, the FLMC represents the ``performance ceiling'' of a specific feature. For example, the FLMC of a ``Blue Button'' is the highest conversion rate achievable by any valid combination that includes it. This metric provides a safety net: it ensures a product manager does not prematurely discard a high-potential feature simply because it was randomly paired with poor-performing counterparts in early testing phases. We visualize this concept in the first row of Figure \ref{fig:tensor_matricizations} using an $8 \times 8 \times 8$ tensor matricized along each mode. The value at the left of each row (e.g., 0.601 in Mode-1) represents the maximum potential of that specific level across all valid combinations. FLMC is defined for each factor separately. Figure \ref{fig:tensor_matricizations} (a-c) displays the unfoldings for all three modes (factors) of the initial tensor.

\begin{figure}[h!]
    \centering
    % --- FIRST ROW: 3 Figures ---
    \begin{subfigure}[b]{0.32\textwidth}
        \centering
        \includegraphics[width=\textwidth]{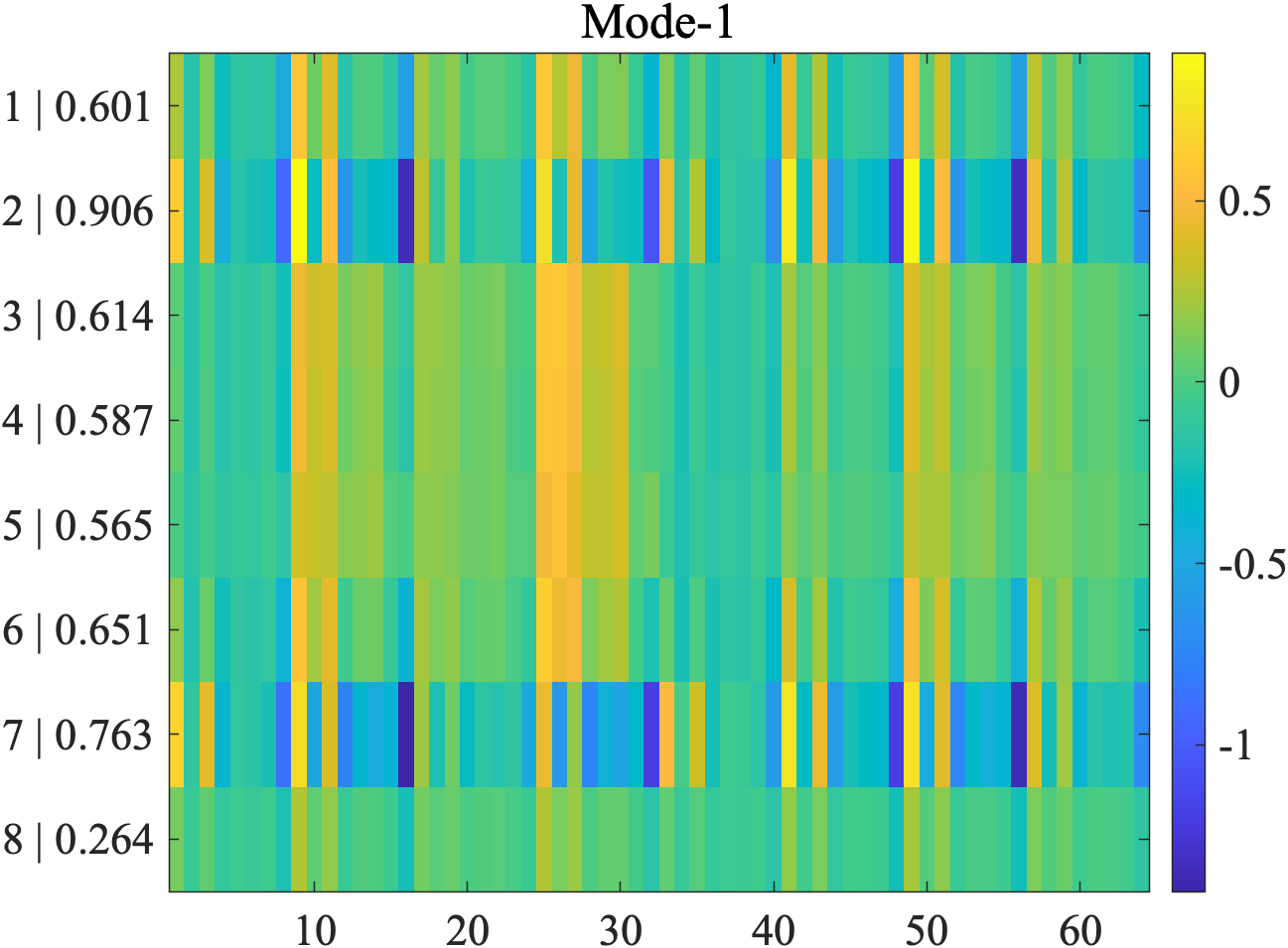}
        \caption{}
        \label{fig:round1_mode1}
    \end{subfigure}
    \hfill
    \begin{subfigure}[b]{0.32\textwidth}
        \centering
        \includegraphics[width=\textwidth]{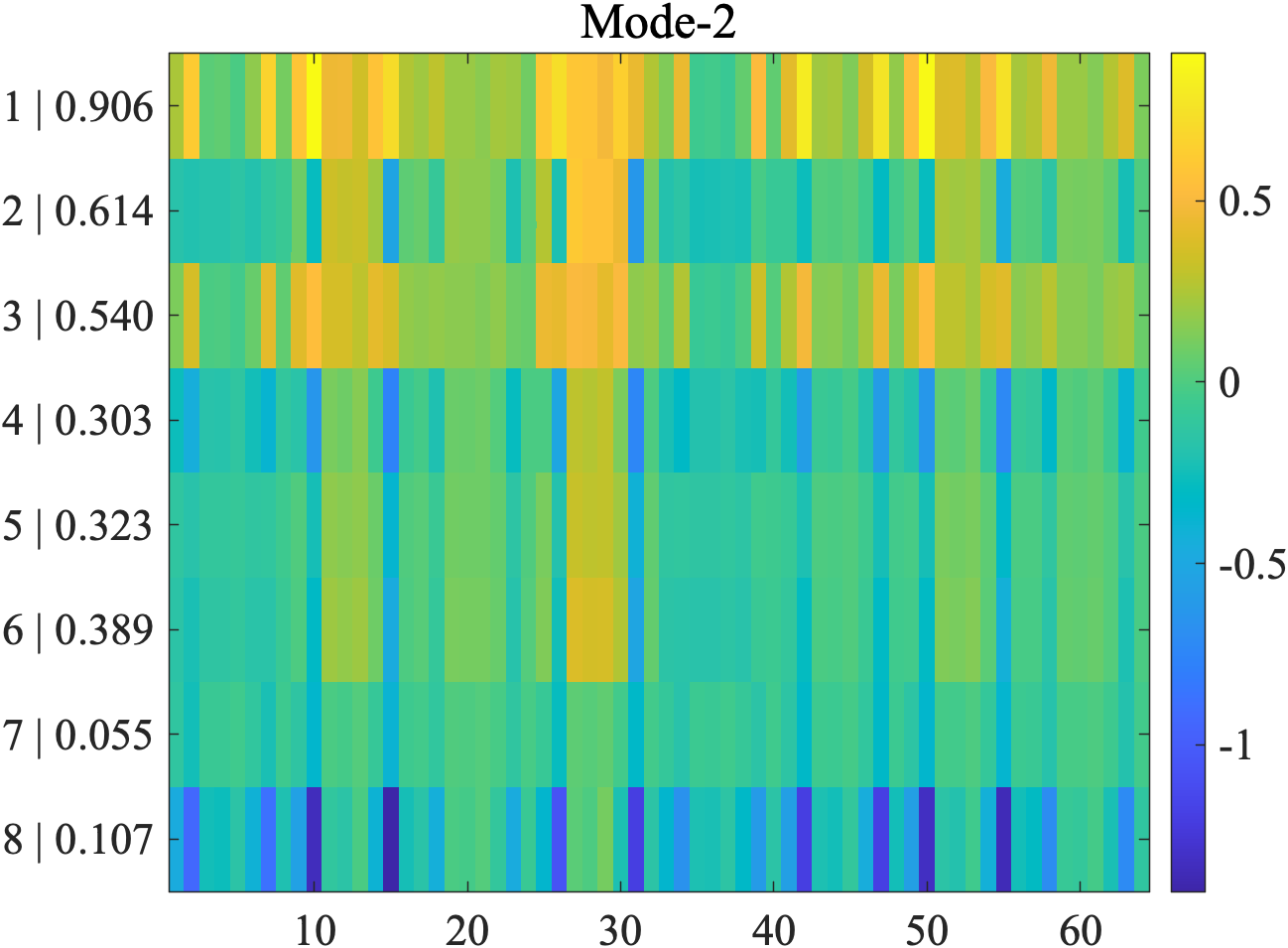}
        \caption{}
        \label{fig:round1_mode2}
    \end{subfigure}
    \hfill
    \begin{subfigure}[b]{0.32\textwidth}
        \centering
        \includegraphics[width=\textwidth]{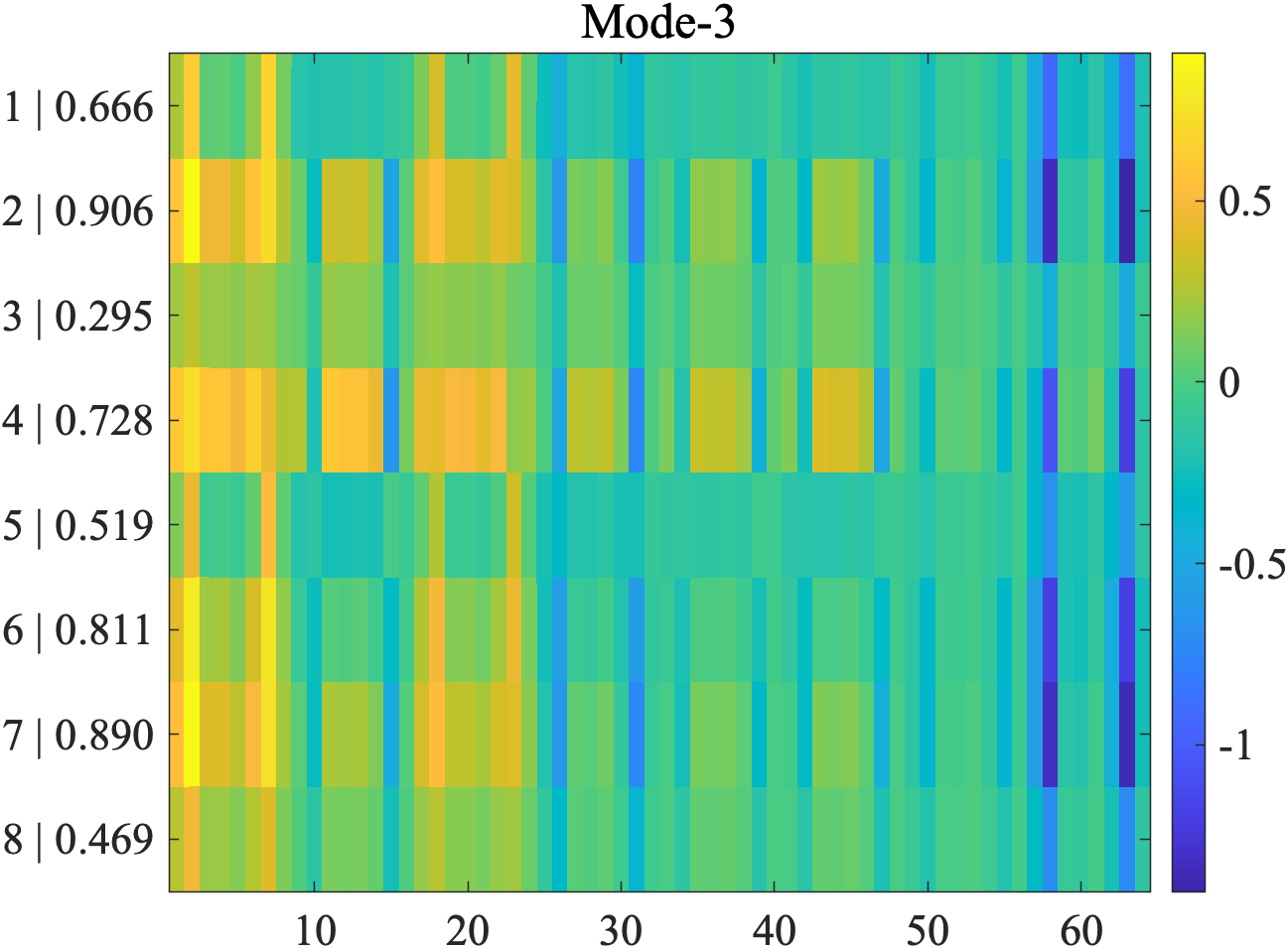}
        \caption{}
        \label{fig:round1_mode3}
    \end{subfigure}

    \vspace{1.5em} % Adds vertical space between the rows

    % --- SECOND ROW: 2 Figures ---
    \begin{subfigure}[b]{0.45\textwidth}
        \centering
        \includegraphics[width=\textwidth]{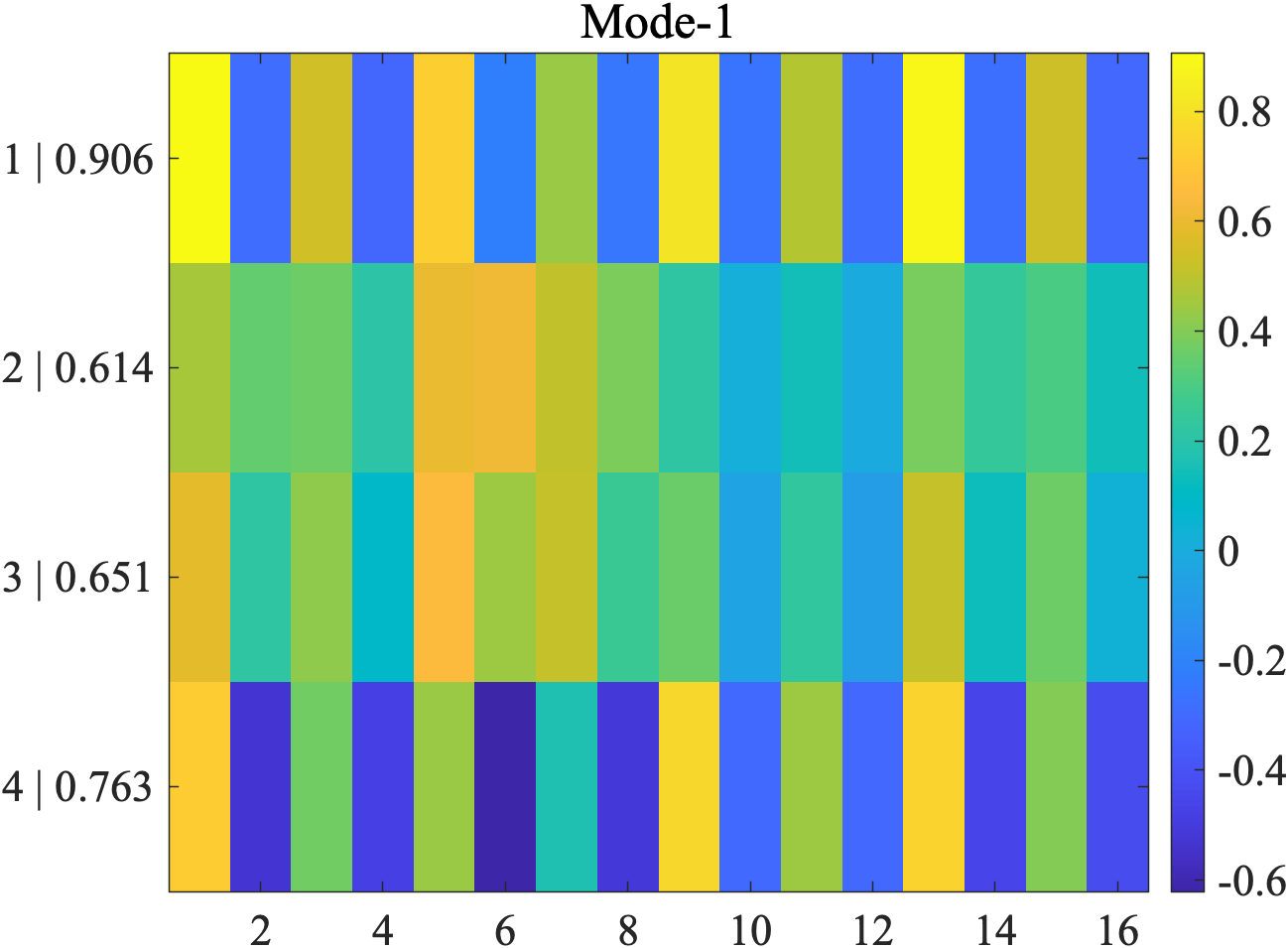}
        \caption{}
        \label{fig:round2_mode1}
    \end{subfigure}
    \hspace{0.05\textwidth} % Keeps the bottom two centered but slightly spaced apart
    \begin{subfigure}[b]{0.45\textwidth}
        \centering
        \includegraphics[width=\textwidth]{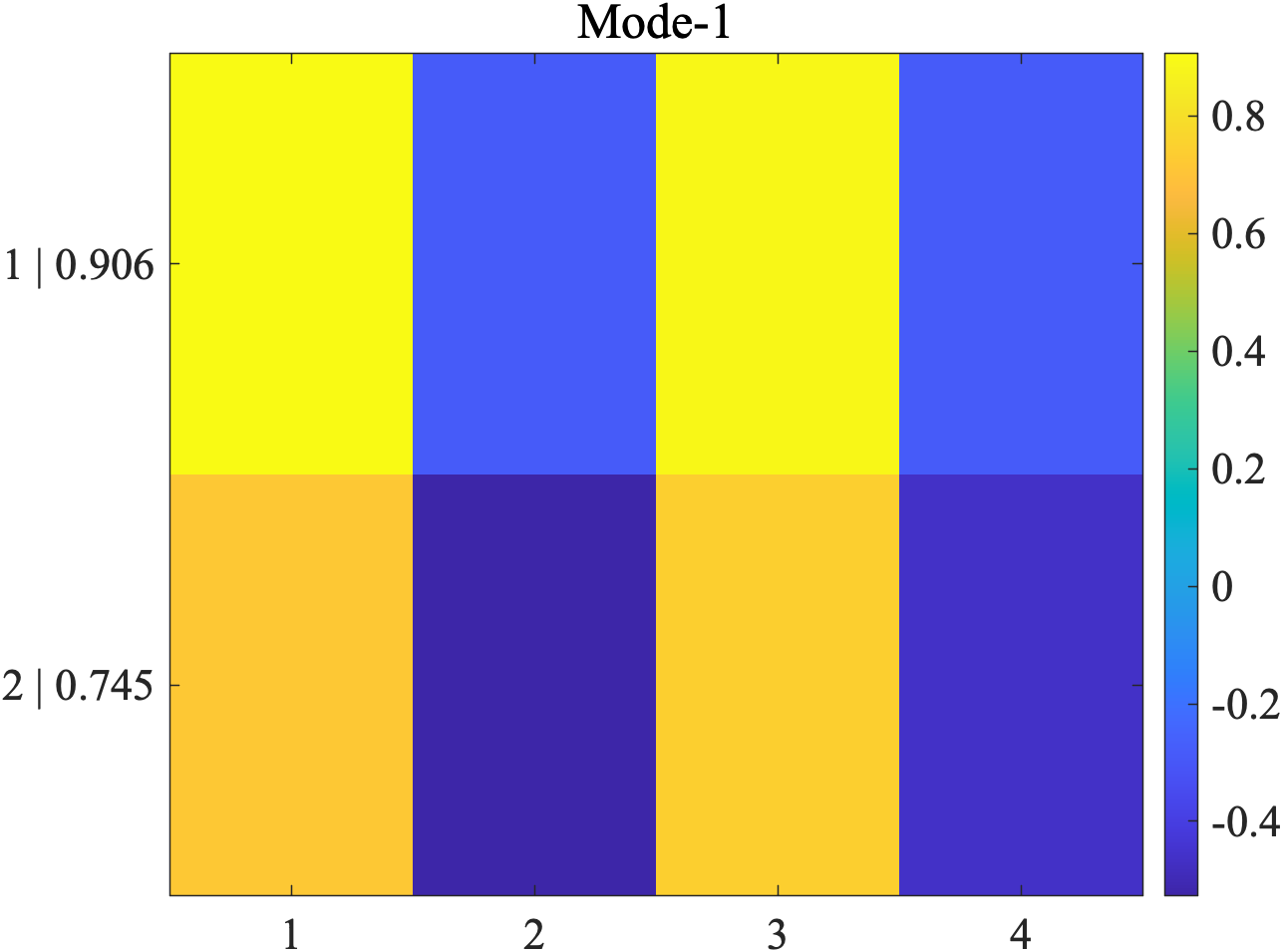}
        \caption{}
        \label{fig:round3_mode1}
    \end{subfigure}

    \caption{Evolution of row-wise maxima during tensor halving. The top row (a-c) displays the unfoldings for Modes 1, 2, and 3 of the original tensor. The bottom row (d-e) tracks the Mode-1 unfolding across the subsequent first and second halving rounds.}
    \label{fig:tensor_matricizations}
\end{figure}
At the end of each round $\ell$, we rank the levels within each factor based on their FLMC. We then prune the design space by eliminating a fraction (e.g., the bottom 50\%) of levels for each factor. It is important to note that a 50\% elimination is not a rigid requirement; rather, it is an instance of a broader level-wise elimination procedure designed to rapidly reduce the dimensionality of the design space. Managerially, this strategy exponentially cuts wasted budget on poor performers while maintaining enough breadth to avoid missing complex synergies between features. Computationally, if the sub-tensors satisfy the theoretical conditions, this level-wise elimination reduces the time-to-insight from a polynomial timeline ($d^m$) to a logarithmic one ($\log d$), enabling managers to find the optimal policy in a fraction of the time. This elimination is processed for each factor individually and in parallel.

In our $8 \times 8 \times 8$ example in Figure \ref{fig:tensor_matricizations}, rows 1, 4, 5, and 8 of Mode-1 fall into the bottom half during the first round and are eliminated, yielding the reduced tensor shown in Figure \ref{fig:tensor_matricizations} (d). A critical phenomenon to observe during this sequential halving is the shifting of maximum values. For instance, the FLMC of a specific level might be 0.7629 in round two, but drop to 0.7449 in round three. The value 0.7629 disappears because the specific cross-factor combination that generated it was deleted when Mode-2 or Mode-3 was sliced. The new maximum, 0.7449, is simply the highest-scoring ``runner-up'' surviving within the remaining space. Discarding a sub-optimal level in one factor can accidentally eliminate the necessary synergies that drove the peak performance of another factor.

This elimination strategy concentrates the search on a denser, high-value sub-tensor. 

\begin{lemma}\label{lem:low-rank}
If $\cT^\star$ has multilinear rank $(r_1,\dots,r_m)$, then any subtensor obtained by restricting each mode to a subset of its indices has multilinear rank at most $(r_1,\dots,r_m)$. In particular, eliminating levels (reducing $d$) does not increase rank.
\end{lemma}
This stage relies on the tensor completion algorithm to infer values for unobserved arms. The validity of this approach is guaranteed by the property that restricting a low-rank tensor to a subset of indices does not increase its rank (Lemma \ref{lem:low-rank}). We continue this elimination process until a predetermined ``switch round'' $\Lsw$. Managerially, $\Lsw$ represents the threshold where the manager believes the remaining design space is no longer ``low-rank''---meaning the surviving features are so highly competitive and nuanced that a small number of latent factors can no longer capture their complex effects across the sub-tensor. (We discuss the theoretical conditions guiding the selection of $\Lsw$ in Remark \ref{rem:conditions}).

After $\Lsw$ rounds, the design space has been successfully distilled to a highly concentrated set of top-performing candidates. However, continuing to rely purely on the tensor model's estimations carries a significant risk of model misspecification due to the breakdown of this low-rank assumption. Treating a dense, high-performing sub-tensor as low-rank might artificially smooth over the true, idiosyncratic differences between the absolute best arms.

To mitigate this, we switch to a \VS, where we treat the surviving interventions as distinct, standalone variations in a multi-armed bandit framework. In this stage, we utilize the Sequential Halving (SH) algorithm (Algorithm \ref{alg:VectorSH}). This is a multi-round procedure where every surviving arm (intervention) receives an equal share of the user traffic to establish an empirical mean estimation. As users enter the digital platform, they are randomly assigned to one of these surviving arms. After a set period, we evaluate the empirical average treatment effect of each arm. We then discard the bottom half of the empirical performers and funnel all remaining traffic into the next round. This process continues until only the single best policy remains.

\subsection{Two-stage algorithm}

We summarize the full procedure in Algorithm \ref{alg:TwoStage_TensorVector}. Let the total budget be $N = \Nsw + \Nps$, where $\Nsw$ is allocated to the tensor stage and $\Nps$ to the vector stage. In \TStage, $\Nsw$ is further divided into per-round budgets $\{N_\ell\}_{\ell=1}^{\Lsw}$. \TStage\ uses low-rank tensor completion to aggressively shrink the design space by eliminating weak factor levels based on estimated marginal contributions, while \VS\ treats the surviving combinations as distinct arms and refines the choice using a standard best-arm identification procedure.

The proposed design is a general framework rather than one tied to a particular tensor-structure assumption. In practice, the tensor-stage algorithm should be chosen to match the structural assumptions that are most appropriate for the application. In Sections \ref{sec:theory} and \ref{sec:real_data}, we specialize to Tucker low-rank structure to develop theoretical guarantees and to implement the design in a real-world setting. The corresponding tensor completion method is described in Appendix~\ref{sec:appx_alg}, where we present a representative nonconvex approach based on Riemannian gradient descent. Randomization in \TStage\ reduces the experimental burden by avoiding exhaustive evaluation of all treatment combinations, but it also introduces estimation error. This trade-off motivates the sequential, elimination-based structure of the design.

\begin{algorithm}[!ht]
\caption{Two-stage algorithm}
\label{alg:TwoStage_TensorVector}
\DontPrintSemicolon
\KwIn{Total budget $N = \Nsw + \Nps$; switch round $\Lsw$; per-round budgets $\{N_\ell\}_{\ell=1}^{\Lsw}$}
\KwOut{A selected intervention index $\widehat{\boldsymbol{i}}$}
\BlankLine
\textbf{Initialize:} For each factor $k\in\{1,\dots,m\}$, set $\cA_k^{(1)}\gets \{1,\dots,d\}$; set $\ell\gets 1$.\;
\BlankLine
\textbf{\TStage\  ($\ell=1,\dots,\Lsw$):}\;
\While{$\ell \le \Lsw$ \textbf{and} $\prod_{k=1}^m |\cA_k^{(\ell)}| > 1$}{
  Let $\Aell \gets \cA_1^{(\ell)}\times\cdots\times\cA_m^{(\ell)}$.\;
  Sample $N_\ell$ indices uniformly at random from $\Aell$ with replacement and observe outcomes $\{(\bidx_t,Y_t)\}_{t=1}^{N_\ell}$.\;
  Fit a tensor completion model on $\{(\bidx_t,Y_t)\}_{t=1}^{N_\ell}$ to obtain $\widehat{\cT}^{(\ell)}$. \;
  \For{$k=1$ \KwTo $m$}{
     For each $i \in \cA_k^{(\ell)}$, compute the marginal contribution
     \[
     \hat\mu^{(\ell)}_{k,i}=\max_{\bj\in\prod_{k'\neq k}\cA_{k'}^{(\ell)}} \widehat{\cT}^{(\ell)}_{(i,\bj)}.
     \]
     Update the active level set by pruning by median:
     \[
     \cA_k^{(\ell+1)} \gets \Big\{i\in\cA_k^{(\ell)}:\ \hat\mu^{(\ell)}_{k,i} \ge \mathrm{median}\big\{\hat\mu^{(\ell)}_{k,i}: i\in\cA_k^{(\ell)}\big\}\Big\}.
     \]
     At each pruning step, we keep exactly the top $\left\lceil\left|\cA_k^{(\ell)}\right| / 2\right\rceil$ indices by empirical score, breaking ties arbitrarily so that $\left|\cA_k^{(\ell+1)}\right|=d_{\ell+1}$.
  }
  $\ell\gets \ell+1$.\;
}
\BlankLine
\textbf{\VS\ :}\;
Let $\cA^{(\Lsw+1)} \gets \cA_1^{(\Lsw+1)}\times\cdots\times\cA_m^{(\Lsw+1)}$ denote the remaining arms and let $\Nps \gets N - \sum_{\ell=1}^{\Lsw} N_\ell$.\;
Run Algorithm~\ref{alg:VectorSH} (sequential halving) on arm set $\cA^{(\Lsw+1)}$ with budget $\Nps$.\;
Let $\widehat{\boldsymbol{i}}$ be the arm returned by the sequential halving stage.\;
\Return $\widehat{\boldsymbol{i}}$.\;
\end{algorithm}

\begin{algorithm}[!ht]
\caption{Sequential Halving (SH)}
\label{alg:VectorSH}
\DontPrintSemicolon
\KwIn{Budget $\Nps$, surviving arms $\mathcal{S}_1$ with size $K = |\mathcal{S}_1|$}
\KwOut{Selected arm $J_N$}
\BlankLine
Let $L \gets \left\lceil \log_2 K \right\rceil$.\;
\For{$\ell=1$ \KwTo $L$}{
  For each arm $i\in\mathcal S_\ell$, sample it
    \[
    T_\ell \;\gets\; \left\lfloor \frac{\Nps}{|\mathcal S_\ell|\,L} \right\rfloor
    \]
  times and compute its empirical mean reward $\hat\mu_i$.\;
  Let $\mathcal S_{\ell+1}$ be the $\left\lceil |\mathcal S_\ell|/2 \right\rceil$ arms in $\mathcal S_\ell$ with the largest $\hat\mu_i$.\;
}
\Return $J_N$ as the unique arm in $\mathcal S_{L+1}$.\;
\end{algorithm}

\section{Theoretical Results}\label{sec:theory}
This section states the assumptions under which our two-stage design is provably effective and
summarizes its performance guarantees. 
The theory formalizes two complementary guarantees: (i) a worst-case simple-regret bound that holds
without any separation assumptions, and (ii) an instance-specific identification bound that quantifies
how quickly the algorithm improves when the data exhibit clear performance gaps.
Throughout, boldface lowercase letters denote vectors (e.g., $\mathbf{x}$), uppercase letters denote
matrices, and calligraphic letters denote tensors.

\subsection{Assumptions}
We first establish structural properties of the true effect tensor $\cT^\star$ and introduce the
tensor notation used in our guarantees. Let $\mathbb{O}_{d,r}:=\{U\in\mathbb R^{d\times r}:U^\top U=I_r\}$
denote the Stiefel manifold. For a tensor $\cT\in\mathbb{R}^{d_1\times\cdots\times d_m}$ and a matrix
$V\in\mathbb R^{d_j\times r_j}$, we write $\cT\times_j V$ for the mode-$j$ multilinear product.
We also write $\mathcal M_j(\cT)$ for the mode-$j$ matricization of $\cT$, which stacks mode-$j$ fibers
as columns.

\begin{assumption}[Low-rank structure]\label{assump:lowrank}
Let $\cT^\star \in \mathbb{R}^{d_1 \times \cdots \times d_m}$ be an order-$m$ tensor representing the true effect tensor over all $\prod_kd_k$ intervention combinations. We assume that $\cT^\star$ admits a low-rank representation in Tucker form. Specifically, there exist integers $r_1,\dots,r_m$ and orthogonal factor matrices $U_k \in \mathbb{O}_{d_k, r_k}$ for $k = 1,\dots,m$, together with a core tensor $\cG \in \mathbb{R}^{r_1 \times \cdots \times r_m}$ such that:
\[
\cT^\star = \cG \times_1 U_1 \times_2 U_2 \cdots \times_m U_m
\]
meaning $\cT^\star$ has Tucker multilinear rank at most $(r_1,\dots,r_m)$. Assume the true tensor $\cT^\star$ satisfies standard $\mu_0$-incoherence condition that
   $\max_{k\in[m]}(d_k/r_k)\|U_k\|_{2,\infty}^2 \leq \mu_0$, where
  $\|U\|_{2,\infty}:=\max_{i\in[d]}\|U_{i,\cdot}\|_2$.
\end{assumption}
Let $\sigma_r(\cdot)$ denote the $r$-th largest singular value of a matrix. The signal strength of $\cT$ is defined by the smallest positive singular value across all of its matricizations:
\(
\lambda_{\min } := \min_{k\in[m]}\sigma_{r_k}\left(\mathcal{M}_k(\cT^\star)\right).
\)
Similarly, define $\lambda_{\max }:= \max _j \sigma_1\left(\mathcal{M}_j(\cT^\star)\right)$. The condition number of $\cT^\star$ is defined by $\kappa:= \lambda_{\max } \lambda_{\min }^{-1}$. For simplicity of exposition, we focus on the balanced case
$d_1=\cdots=d_m=d$ and $r_1=\cdots=r_m=r$.
Thus, $\cT^\star \in \mathbb{R}^{d\times\cdots\times d}$, and the effective complexity of the tensor is governed by a single rank parameter $r$. While we focus on Tucker low-rankness to establish our theoretical guarantees, our methodology is not strictly bound to this specific assumption. The choice of low-rank structure ultimately depends on the manager's behavioral assumptions regarding the feature space. These results can be naturally extended to other low-rank paradigms, such as CP \parencites{carroll1970analysis} and Tensor-Train decompositions \parencites{fannes1992finitely,oseledets2011tensortrain}.
Without loss of generality, we additionally assume that $d$ is a power of two so that the Stage-I
``keep-the-top-half'' pruning rule yields exact halving at each round.
Define \( d_\ell := d/2^{\ell-1}, \ell=1,2,\ldots, \)
and recall that $\cA_k^{(\ell)}\subseteq[d]$ denotes the active level set
in mode $k$ at the beginning of Stage-I round $\ell$.
With deterministic tie-breaking to keep exactly the top half,
the algorithm maintains $|\cA_k^{(\ell)}| = d_\ell$ for all
$ k\in[m]$ and all $\ell\le \Lsw+1$, 
and hence the active Cartesian set has cardinality
\(
|\cA^{(\ell)}|
=
\prod_{k=1}^m |\cA_k^{(\ell)}|
=
d_\ell^{\,m}.
\)
If $d$ is not a power of two, the same statements hold with
$d_\ell:=\lceil d/2^{\ell-1}\rceil$ (and corresponding floors/ceilings in the pruning),
which only changes constants and does not affect the rates; we omit these
bookkeeping details throughout.

\begin{assumption}[Sampling and noise model across rounds]\label{assump:sampling-noise}
Let $\{\mathcal F_\ell\}_{\ell\ge1}$ be the natural filtration generated by the algorithm up to the
beginning of round $\ell$. In particular, the active design set
$\cA^{(\ell)}=\cA_1^{(\ell)}\times\cdots\times\cA_m^{(\ell)}$ is $\mathcal F_\ell$-measurable. In round $\ell$, conditional on $\mathcal F_\ell$, the algorithm collects $N_\ell$ independent samples as follows.
There exist i.i.d.\ queried intervention indices
\(
\bidx_{\ell,1},\ldots,\bidx_{\ell,N_\ell}\ \in\ \cA^{(\ell)}
\)
such that:
\begin{enumerate}
\item Conditional on $\mathcal F_\ell$,
\(
\bidx_{\ell,t}\ \stackrel{\text{i.i.d.}}{\sim}\ \mathrm{Unif}\big(\cA^{(\ell)}\big), t=1,\ldots,N_\ell.
\)

\item  Each observation satisfies
\(
Y_{\ell,t}
=
\cT^\star_{\bidx_{\ell,t}}\;+\;\xi_{\ell,t}, t=1,\ldots,N_\ell.
\)
Conditional on $\mathcal F_\ell$, the noise variables
$\{\xi_{\ell,t}\}_{t=1}^{N_\ell}$ are independent, mean-zero, and $\sigma^2$-sub-Gaussian, i.e. for all
$s\in\mathbb R$,
\(
\mathbb E\!\left[\exp\!\left(s\,\xi_{\ell,t}\right)\,\middle|\,\mathcal F_\ell\right]
\;\le\;
\exp\!\left({\sigma^2 s^2/2}\right).
\)
Moreover, the noise variables are independent across rounds conditional on the corresponding filtrations.
\end{enumerate}
\end{assumption}
\noindent
Assumption~\ref{assump:sampling-noise} matches standard experimental practice: conditional on the platform state (and the algorithm's history), users are randomized uniformly within the active design set, and outcomes provide unbiased noisy measurements of the underlying treatment effect tensor.

\begin{assumption}[Tensor completion tail bound]\label{assump:completion-tail}
Let $\widehat{\cT}$ denote a tensor-completion estimator of $\cT^\star$ (defined in Assumption~\ref{assump:lowrank}) constructed from partial, noisy observations with total traffic budget $N = |\Omega|$. Define the sampling rate by $p := N/d^m$, the fraction of all possible intervention combinations that are experimentally evaluated. We assume there exist absolute constants $c,C_m>0$ such that for all $t>0$, supposing that $r, \kappa, \mu_0 = O(1)$:
\[
\bP\Big(\infnorm{\widehat{\cT}-\cT^{\star}} \ge t \Big) \le C_m \exp\!\Bigg\{-\,c\,t^2\,\frac{\lambda_{\min}^2}{\sigma^2}\,\frac{p}{\df\,\log d}\frac{1}{\infnorm{\cT^\star}^2}  \Bigg\},
\]
where $\infnorm{\cX}$ denotes the maximum absolute entry of tensor $\cX$, $\sigma^2$ is the variance proxy of the noise, and $\df = m(dr-r^2)+r^m$ denotes the effective degrees of freedom of the low-rank tensor model.
\end{assumption}
Assumption~\ref{assump:completion-tail} can be achieved by a range of computationally efficient algorithms, such as vanilla gradient descent \parencite{cai2022nonconvex}, Riemannian gradient descent \parencite{wang2023implicit}, or online Riemannian gradient descent (RGM)\parencite{li2024onlinea}, under standard conditions on incoherence, signal-to-noise ratio, and sample size. 
Most tensor-completion results are stated in Frobenius norm, which controls average reconstruction error. Our decision problem is different: elimination and final recommendation depend on comparing individual entries and row-wise maxima, so uniform entrywise control is the relevant object. For this reason, we formulate the completion requirement through an $\ell_\infty$-type tail bound. Such guarantees are available for several nonconvex tensor-completion methods under incoherence, sufficient sampling, and signal-strength conditions. See Appendix~\ref{sec:appx_alg} for one representative nonconvex tensor-completion route, based on Riemannian gradient descent, that yields entrywise error guarantees under standard incoherence, sample-size, and signal-to-noise conditions. The key requirement for our analysis is uniform control over the active tensor entries, not merely good average recovery. We now summarize these conditions and interpret them in our experimental setting.

\begin{remark}[Incoherence, signal-to-noise ratio, and unbiased randomization]\label{rem:conditions}
\leavevmode
\begin{itemize}
    \item \textbf{Incoherence:} 
  Small incoherence $\mu_0$ (see Assumption \ref{assump:lowrank}) means that the “energy” of the singular vectors is spread relatively evenly across coordinates. Incoherence prevents the subspace from aligning too closely with the canonical basis (i.e., from having a few extremely “spiky” entries), which would otherwise make completion impossible unless those specific entries are observed. Managerially, incoherence implies that the success of the platform is driven by broad behavioral themes rather than a single ``needle-in-a-haystack'' combination. If a platform's value was entirely concentrated in one specific, unobserved combination of color, flow, and coupon, no algorithm could ``predict'' it without testing it directly. Small $\mu_0$ guarantees that by observing a strategic subset of combinations, we can reliably infer the rest of the design landscape.
  
    \item \textbf{Signal-to-noise ratio (SNR):} In tensor completion literature, the underlying low-rank signal must be sufficiently strong relative to noise and sampling rate so that recovery is both information-theoretically and algorithmically feasible. A typical requirement takes the form $\lambda_{\min }/\sigma \gg \sqrt{(d^m)^{3/2}/N }$ and \( N \gg (d^m)^{1/2}, \) where $\lambda_{\min}$ denotes the smallest nonzero singular value of the matricization of the tensor and $\sigma$ is the noise level. Such conditions are common across various tensor models \parencites{zhang2018tensor,xia2021statistically,xia2022inference,li2024onlinea} and are believed to be essentially necessary for polynomial-time recovery algorithms. By contrast, the minimax-optimal statistical rate can be substantially weaker, $\lambda_{\min }/\sigma \gg \sqrt{d^{m+1}/N}$ and 
\( N \gg d,\) but achieving recovery at this regime is known to be NP-hard \parencite{barak2022noisy}. 
 In digital experiments, $\lambda_{\min}$ captures the strength of latent behavioral factors, while $\sigma$ represents the inherent ``noise'' or variance of user behavior. A high SNR means the underlying consumer preferences are strong enough to be detected despite this volatility. This condition warns managers that for highly noisy metrics, a larger ``exploration portfolio'' is required to ground the tensor model before elimination can safely begin.
\item \textbf{Sampling rate:}  
Tensor-completion theory typically assumes a random sampling model in which each tensor entry is included in the observation set $\Omega$ independently with probability $p$. In our setting, $p = N/d^m$, where $N = |\Omega|$ is the number of experimentally evaluated combinations and $d^m$ is the total number of possible combinations. A common requirement is
\(
p \gg d^{-m/2},
\)
which implies the budget condition
\(
N \gg d^{m/2}.
\)
Thus, successful recovery does not require testing all $d^m$ interventions. Instead, it is sufficient to allocate traffic on the order of the square root of the design-space size, because the low-rank structure allows information to be shared across related combinations. From a managerial perspective, this condition provides a concrete traffic threshold for when tensor-based experimentation is feasible. Suppose the platform has traffic budget $B$ in a given round, and each intervention requires $b$ user-level observations to estimate its treatment effect with the desired noise level. Then the number of intervention combinations that can be evaluated in that round is
\(
N = B/b,
\)
so the sampling condition becomes a corresponding requirement on platform traffic:
\(
B \gg b\, d^{m/2}.
\)
This interpretation makes clear how statistical recovery conditions translate into operational experimentation constraints. In particular, even when individual intervention-level measurements are noisy, accurate recovery remains possible as long as the total traffic is large enough relative to the size of the design space.
    \item \textbf{Unbiased Randomization:}   In each round, we randomly sample over the set $\Aell$ and use these sample to get the tensor completion estimation $\widehat{\cT}^{(\ell)}$. In the experimental design language, $\Aell$ is the active intervention set and when the user comes to the platform, he or she will be randomly assigned to one of the intervention in the active set $\Aell$. The change of $\Aell$ doesnot affect the underlying distribution of the user population. Thus, the samples over different rounds are independent. 
\end{itemize}
\end{remark}

\subsection{Simple regret guarantees}
We analyze the theoretical performance of the Two-Stage Sequential Elimination algorithm along two complementary dimensions:
\begin{itemize}
    \item \textbf{Gap-independent guarantees,} which provide worst-case bounds as a function of the experimental budget $N$ and problem dimensions (e.g., $d,m,r$), without assuming any separation between the best and second-best interventions;
    \item \textbf{Gap-dependent guarantees,} which provide instance-specific bounds that become stronger when the best intervention is better separated from its competitors.
\end{itemize}

The gap-independent analysis provides a uniform benchmark for experimental design. Because it does not rely on unknown features of the instance, it yields a budget-to-performance guarantee that is useful for ex ante planning. The gap-dependent analysis complements this worst-case benchmark by capturing the environments in which adaptive experimentation is especially effective. In many platform applications, many factor levels are clearly inferior and only a small subset remains competitive. In such cases, the tensor stage can safely eliminate large portions of the design space early, so the algorithm identifies a high-performing policy using substantially less traffic than the worst-case analysis alone would imply.

Recall that the total budget $N$ is split into $\Nsw$ samples for the tensor rounds and $\Nps$ samples for the vector rounds. Let $\Lsw$ denote the number of tensor rounds before switching. After $\Lsw$ rounds of median pruning, the number of remaining arms is $\Kps := |\cA^{(\Lsw+1)}| = (d/2^{\Lsw})^m$, and the number of subsequent sequential-halving rounds is $\Lps := \lceil \log_2 \Kps \rceil = m(\log_2 d - \Lsw)$.

\subsubsection{Gap-Independent (Worst-Case) Performance}
\begin{theorem}[Two-stage gap-independent bound]\label{thm:TwoStage-gapind}
Consider the two-stage procedure in Algorithms~\ref{alg:TwoStage_TensorVector}
and~\ref{alg:VectorSH} with a switch after $\Lsw$ tensor rounds. 
For each Stage-I round $\ell \in [L_I]$, define the roundwise sampling rate
\(
p_\ell := N_\ell/|A^{(\ell)}|, 
\)
and
$p_{\min}:=\min_{\ell\in[L_I]} p_\ell.$
Let
\(
N_I := \sum_{\ell=1}^{L_I} N_\ell
\) be the total Stage-I budget, and let $\Nps:=N-N_I$ be the Stage-II budget. Suppose the true effect tensor satisfies Assumption~\ref{assump:lowrank}, the sampling process satisfies Assumption~\ref{assump:sampling-noise}, the completion algorithm satisfies Assumption~\ref{assump:completion-tail}, and the restricted sub-tensors satisfy the analogous round-wise properties of Assumption~\ref{assump:completion-tail} (formalized in Assumption \ref{assump:concentration-appendix} in Appendix \ref{sec:proof}).  In particular, there exist absolute constants $c,C_m>0$ and a deterministic constant $\lambda_{\min}>0$ such that for every Stage-I round $\ell\in[\Lsw]$,
\(
\lambda_{\min}^{(\ell)}:=\min_{k\in[m]}\sigma_r\!\big(\mathcal M_k(\cT^{(\ell)})\big)\ge\lambda_{\min},
\)
and the conditional restricted sup-norm tail bound in
Assumption~\ref{assump:concentration-appendix} holds on $\cA^{(\ell)}$ with this $\lambda_{\min}$ and $p_\ell$. 
Then the simple regret satisfies:
\begin{align}
\mathrm{SReg}_N = O\!\left(\sigma\,\left[\underbrace{\sqrt{\tfrac{\df\,\Lsw^2\,\log d}{p_{\min}}\,\log \Lsw} \; \frac{\infnorm{\cT^\star}}{\lambda_{\min}}}_{\TStage} \;+\; \underbrace{\sqrt{\tfrac{\Kps\,\Lps}{\Nps}\,\log \Lps}}_{\VS}\right]\right),
\label{eq:twostage-sreg-equal-fixed}
\end{align}
where $\df = m(dr-r^2)+r^m$ is the degrees of freedom of the low-rank tensor model. 
\end{theorem}
The proof is provided in Section~\ref{sec:proof of theorem gapind}.

The bound in~\eqref{eq:twostage-sreg-equal-fixed} decomposes into a contribution from the \TStage\ and a contribution from the \VS. The \TStage\ term reflects how accurately we can recover the low-rank tensor from partial observations; it scales with $\sqrt{\df}$ and is controlled by the sampling rate and the number of tensor rounds. The \VS\ term matches the usual behavior of simple regret over $\Kps$ arms using sequential halving.

To understand the practical value of this result, we compare it against a standard unstructured approach. If a manager were to treat every intervention as a standalone option (a pure vector baseline), the simple regret scales as $\Theta\big(\sqrt{d^m/N}\big)$ \parencite{lattimore2020bandit}, and the required traffic budget to guarantee optimal identification must vastly exceed the total number of combinations ($N \gg d^m$). In contrast, our Two-stage algorithm exploits the low-rank tensor structure, reducing the regret scaling to $\Theta\big(\sqrt{\df/N}\big)$. Because the effective degrees of freedom are exponentially smaller than the total space ($\df \ll d^m$), our algorithm only requires a budget that exceeds the square root of the design space ($N \gg d^{m/2}$). Managerially, this translates to massive savings in experimental traffic while maintaining guarantees for identifying the best policy.

It is important to note that the unstructured/vector ($\Delta_{\min}/\sigma \gtrsim \sqrt{(d^m\log d)/N}$)  and structured/tensor ($\lambda_{\min}/\sigma \gg \sqrt{(d^m)^{3/2}/N}$) guarantees are driven
by different notions of signal. The vector baseline is governed by the minimum reward gap
$\Delta_{\min}$ between the best and second-best interventions, while the tensor stage is governed by a structural signal parameter such as $\lambda_{\min}$ (the smallest nonzero singular value across matricizations). Consequently, neither condition uniformly dominates the other across all instances: when many arms have tiny gaps but the overall response surface exhibits strong latent structure, tensor completion can still be reliable even though $\Delta_{\min}$ is small. A practical implication is that per-user outcomes can be very noisy relative to the effect of a single intervention combination, as long as the aggregate experimental traffic is large enough to reveal the shared low-dimensional structure. In managerial terms, the algorithm does not need each treatment cell to have high signal-to-noise ratio; it only needs the platform to generate enough randomized exposure overall for the latent behavioral drivers to become identifiable.

Because the proposed two-stage algorithm adaptively shrinks the design space, the Stage-I tensor-completion guarantees must hold not only for the original tensor but also for the restricted sub-tensors encountered up to the switch round \(\Lsw\). Lemma~\ref{lem:low-rank} shows that restricting to an active sub-tensor does not increase multilinear rank, so the latent structural complexity does not worsen as the design space contracts. To extend recovery guarantees uniformly over Stage~I, Theorem~\ref{thm:TwoStage-gapind} further imposes round-wise regularity conditions on the active sub-tensors. In particular, we assume
\(
\max_{\ell\le \Lsw}\mu_0^{(\ell)} = O(1), \min_{\ell\le \Lsw}\lambda_{\min}^{(\ell)} > 0.
\)
The quantity \(\mu_0^{(\ell)}\) is the incoherence parameter of the active sub-tensor at round \(\ell\); it measures whether the latent factor directions remain sufficiently diffuse across the surviving levels in each mode. The quantity \(\lambda_{\min}^{(\ell)}\) is the smallest nonzero singular value across the matricizations of the active sub-tensor and captures the strength of the weakest retained latent direction. Under the Tucker model,
\(
\cT^{(\ell)}
=
\cG\times_1(S_1^{(\ell)}U_1)\times\cdots\times_m(S_m^{(\ell)}U_m),
\)
so \(\lambda_{\min}^{(\ell)}\) is controlled by the conditioning of the restricted factor matrices \(S_k^{(\ell)}U_k\). Taken together, these conditions require that Stage~I pruning remove weak levels while preserving a sufficiently diffuse and well-conditioned latent structure. This is plausible in our setting because the underlying effect tensor reflects systematic interactions across multiple factor levels, and the algorithm switches to the vector stage before the active set becomes too small for stable tensor recovery.

Finally, the predetermined switch point $\Lsw$ acts as a managerial safeguard against model misspecification. Rather than aggressively pruning the design space until the dimensions collapse entirely to the underlying rank $r$—which risks accidentally eliminating the true optimal combination if the low-rank assumption is only an approximation—the manager has the conservative option to halt the tensor phase early. By switching to the assumption-free vector phase for the final selection, the algorithm balances the rapid search efficiency of tensor structures with the robust, model-agnostic guarantees of traditional bandit selection.

\subsubsection{Gap-Dependent (Instance-Specific) Performance}
In practice, some factor levels are clearly inferior. Gap-dependent (instance-specific) bounds characterize how the error probability scales with the difficulty of a particular problem instance, rather than in the worst case. When all arms are nearly indistinguishable (very small gaps), the problem is intrinsically hard. In many realistic digital experimentation problems, however, some interventions are clearly better than others and the gaps are not tiny; in such cases, a good algorithm should exploit this structure and achieve substantially smaller error.

We next introduce instance-specific quantities that describe how many factor levels are competitive
and how strongly they are separated. For each mode $k\in[m]$ and level $i\in[d]$, define the marginal contribution
\(
\mu_{k,i}:=\max_{\bj\in[d]^{m-1}} \cT^\star_{(i,\bj)},
\)
the best achievable performance when level $i$ is used in factor $k$ and all other factors are
optimized. Let $\mu_{k,(1)}\ge\cdots\ge \mu_{k,(d)}$ be the order statistics of
$\{\mu_{k,i}\}_{i=1}^d$ and define mode-wise gaps
\(
\Delta_{k,t}:=\mu_{k,(1)}-\mu_{k,(t)}, t=1,\ldots,d.
\)
We aggregate across modes via the row-wise gap sequence
\(
\Delta_t:=\min_{k\in[m]}\Delta_{k,t},
\)
so that $\Delta_t$ measures separation in the ``hardest'' mode.

For a tolerance $\eps>0$, define the number of $\eps$-good levels in each mode,
\(
g_k(\eps):=\bigl|\{i\in[d]:\ \mu_{k,i}\ge \mu_{k,(1)}-\eps\}\bigr|,
g_{\max}(\eps):=\max_{k\in[m]} g_k(\eps).
\)
Finally, define the mode-wise pivot rounds
\(
\ell_k'(\eps):=\max\bigl\{1\le \ell\le \log_2(d)-2:\ \Delta_{k,d_{\ell+2}}>\eps\bigr\},
\ell'(\eps):=\min_{k\in[m]}\ell_k'(\eps).
\)
Intuitively, $\ell'(\eps)$ is the last Stage-I round at which every mode still has an
$\eps$-margin separating its top quarter $d_{\ell+2}=d_\ell/4$ rows from the rest; this ``top-quarter
cushion'' is what makes elimination robust to estimation error under median pruning. Define also the smallest active mode size during Stage~I:
\(
d_{\min}\ :=\ \min_{1\le \ell\le \Lsw} d_\ell \ (=d_{\Lsw}\ \text{under exact halving}).
\)

\begin{theorem}[Gap-dependent guarantee]
\label{thm:gap-dep}
Assume the assumptions of Theorem \ref{thm:TwoStage-gapind} hold, and the Stage-I switch round satisfies
\(
1\ \le\ \Lsw\ \le\ \min\bigl\{\ell'(\eps),\ \lfloor \log_2(d/r)\rfloor\bigr\}.
\)
Define the Stage-I (row-gap) hardness
\(
H_{2,\mathrm{I}}\ :=\ \max_{2\le t\le d}\ t\Delta_t^{-2},
\)
and the row-based Stage-II hardness surrogates
\[
M_{\mathrm{row}}(\eps)
\ :=\
\max_{t\ \ge\ \left\lceil\left(g_{\max }(\eps / 2)+1\right)^{1 / m}\right\rceil}\ 
\frac{t^m}{\Delta_t^{\,2}},
\qquad
H_{\mathrm{row}}(\eps)
\ :=\
\frac{1}{g_{\max}(\eps/2)}\,M_{\mathrm{row}}(\eps).
\] 
Then there exist universal constants $K_0,C_0,c_0>0$ such that if
\[
p_{\min}
\ \ge\
K_0\,\Big(\frac{\sigma}{\lambda_{\min}}\Big)^2\,
\bigl(\log d\bigr)\,
\infnorm{\cT^\star}\,
H_{2,\mathrm{I}}\,
\Bigl(mr+\frac{r^m}{d_{\min}}\Bigr)
\qquad\text{and}\qquad
\Nps\ \ge\ C_0\,M_{\mathrm{row}}(\eps),
\]
the output $J_N$ of the two-stage algorithm satisfies the gap-dependent tail bound
\begin{align*}
\Pr\!\left(\mu_\star-\mu_{J_N}>\eps\right)
\ \le\ 
C_m\,\Lsw\,
\exp\!\Bigg(
-\,\widetilde{\Theta}\!\Big(
\frac{p_{\min}\,\lambda_{\min}^{\,2}}
{\sigma^2\,\infnorm{\cT^\star}\,H_{2,\mathrm{I}}}
\Big)
\Bigg)
\;+\;
\exp\!\Bigg(
-\,\widetilde{\Theta}\!\Big(
\frac{\Nps}{H_{\mathrm{row}}(\eps)}
\Big)
\Bigg),
\end{align*}
where $\mu_\star=\max_{\bidx\in[d]^m}\cT^\star_{\bidx}$ is the optimal mean reward and
$\mu_{J_N}=\cT^\star_{J_N}$ is the (true) mean reward of the selected arm.
The notation $\widetilde{\Theta}(\cdot)$ hides only absolute constants and logarithmic factors in $d$.
\end{theorem}
The proof is provided in Section~\ref{sec:proof of gap-dep thm}.

Theorem~\ref{thm:gap-dep} provides a high-probability identification guarantee and decomposes the
failure probability into two terms corresponding to the two stages. The Stage-I term controls the
risk of eliminating globally $\eps/2$-good factor levels during tensor screening; the Stage-II term
controls the risk that Sequential Halving fails to select an $\eps$-optimal intervention from the
surviving set.

The key instance-specific quantity in Stage~I is
\(
H_{2,\mathrm I}=\max_{t\ge2} t\,\Delta_t^{-2},
\)
a row-wise analogue of the classical vector complexity $H_2=\max_{i\ge2} i/\delta_i^2$, but defined
over the $d$ ordered {row} gaps rather than the full $K=d^m$ arm gaps. Managerially, $H_{2,\mathrm I}$
captures how quickly the algorithm can identify a small set of promising levels within each factor:
when most levels are clearly dominated (large $\Delta_t$ for moderate $t$), Stage~I can prune
aggressively with low risk. Stage~II is governed by the row-based surrogates $M_{\mathrm{row}}(\eps)$ and $H_{\mathrm{row}}(\eps)$, which translate row-wise competitiveness into an effective number of near-optimal {arms} of order $t^m$. When only a few levels per factor remain competitive, the number of plausible near-optimal combinations collapses from $d^m$ to a much smaller subset, and Sequential Halving benefits accordingly.

\section{Semi-Synthetic Evaluation on E-Commerce Bundling Data}
\label{sec:real_data}

This section provides a data-driven evaluation of the proposed method in a large-scale product-bundling problem in online retail. Product bundling refers to the practice of offering two or more distinct items as a single package, often at a promotional price. A familiar example is a platform offering pasta, pasta sauce, and grated cheese together as an ``Italian dinner kit,'' or a fast-food chain selling a burger, fries, and a drink as a meal combo. In each case, the seller does not merely price products separately, but designs a bundle intended to increase demand, improve convenience, or raise basket size.

Product bundling has a long history in economics and marketing. A large classical literature studies when bundling is profitable, how firms should choose between pure components, pure bundles, and mixed bundling, and how these choices depend on demand heterogeneity, complementarity or substitutability across products, and welfare considerations \parencites{adams1976commodity,schmalensee1982commodity,venkatesh2003optimal,derdenger2013dynamic}. Much of this literature, however, is developed for a monopolist or an integrated firm that produces the component goods and then decides how to package and price them. By contrast, modern e-commerce often presents a different environment: a downstream retailer or online platform may bundle products supplied by different firms, and may do so for promotional, recommendation, or basket-expansion purposes rather than as a pure manufacturing or pricing decision.

This distinction is important because, despite the breadth of the bundling literature, there is still limited guidance on how a retailer should search over a massive set of feasible promotional bundles and identify high-performing ones under realistic traffic constraints. In online retail, the marginal cost of creating a candidate bundle is close to zero, but the design space is enormous. For example, selecting only two items from a catalog of $100{,}000$ products already yields $\binom{100{,}000}{2} \approx 5\times 10^9$ possible pairs, and the number grows even faster for larger bundles. Identifying an attractive bundle in such a space is therefore a fundamentally combinatorial experimentation problem. Our proposed design addresses this challenge by providing a practically implementable way for retailers and platforms to use limited experimental traffic to screen a vast bundle space and efficiently identify promising combinations.

\subsection{Data Description and Tensor Construction}
We utilize a large-scale dataset from Alibaba’s Taobao platform, China's largest e-commerce marketplace \footnote{\url{https://tianchi.aliyun.com/dataset/649}}, containing approximately 100 million interaction records across 4.16 million items and 9,439 unique categories. We use these data to build an offline proxy environment rather than a live randomized experiment. To construct a representative ground-truth tensor $\mathcal{X}$, we first identify the top 100 most popular items based on total user interactions and isolate the three categories with the highest density of these products. These categories contain 21, 10, and 8 of the top-tier items, respectively, collectively accounting for 39\% of the platform's most popular products. We model the interactions among these items as a 3-mode tensor of dimension $21 \times 10 \times 8$, where each entry $\mathcal{X}_{i,j,k}$ represents the popularity of a specific three-item bundle, quantified by the number of unique user interactions for that combination. To visualize the interaction structure in the primary category, Figure~\ref{fig:heatmap} reports a heatmap of the mode-1 matricization $\mathcal M_1(\cX)$ (the unfolding of the original tensor along mode 1), which highlights clustering patterns among items in the first category.

\begin{figure}[h!]
\centering
\includegraphics[width=.8\textwidth]{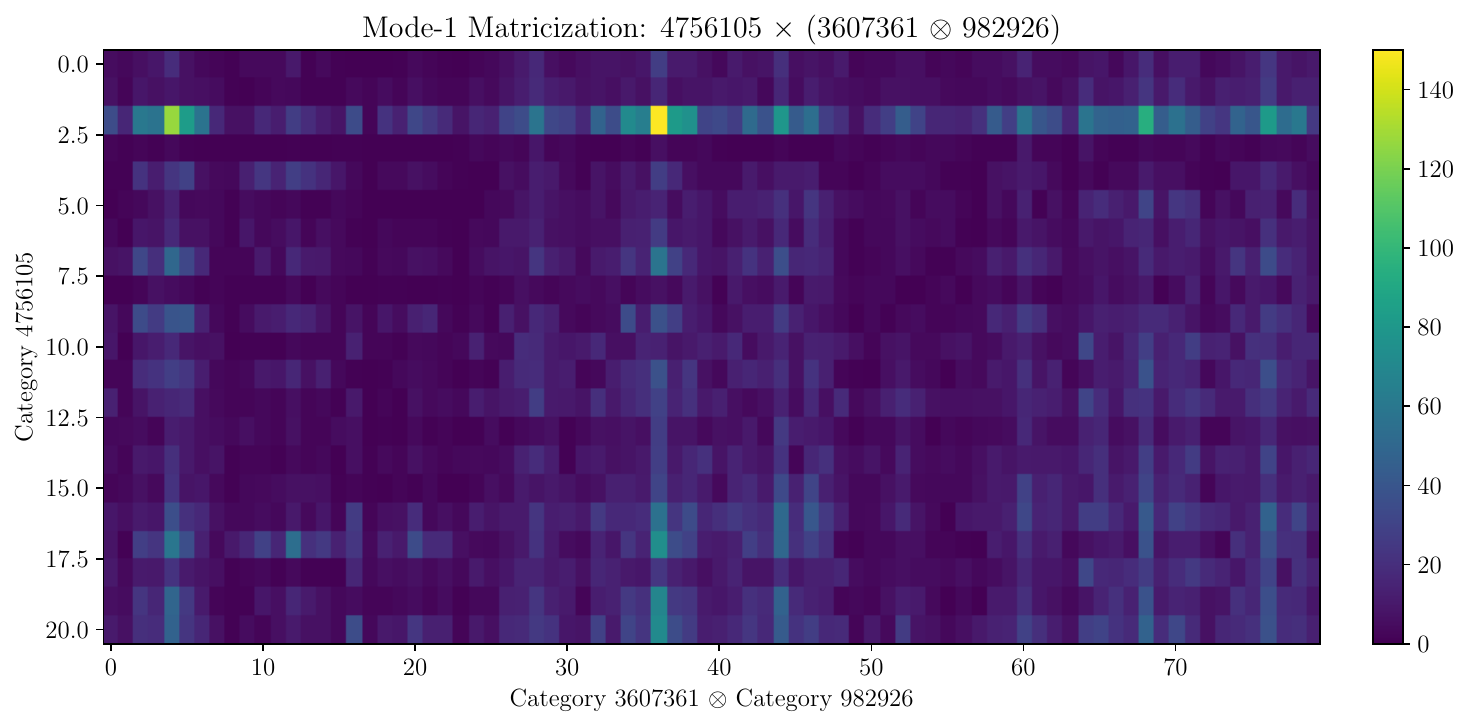}
\caption{Heatmap of the mode-1 matricization $\mathcal M_1(\cX)$ (unfolding of the original tensor along mode 1).}
\label{fig:heatmap}
\end{figure}

The multilinear rank $(r_1, r_2, r_3)$ of the tensor $\mathcal{X} \in \mathbb{R}^{21 \times 10 \times 8}$ is estimated via mode-$n$ unfolding. For each mode $n \in \{1, 2, 3\}$, the tensor is flattened into a matrix $\mathbf{X}_{(n)}$ and decomposed via Singular Value Decomposition (SVD): $\mathbf{X}_{(n)} = \mathbf{U}^{(n)} \mathbf{\Sigma}^{(n)} (\mathbf{V}^{(n)})^\top$. To determine the truncated rank $r_n$, we employ the Cumulative Percentage of Variance (CPV) criterion. For a target energy threshold $\eta = 0.95$, we select the smallest $r_n$ such that: $\sum_{i=1}^{r_n} (\sigma_i^{(n)})^2 / \sum_{j=1}^{I_n} (\sigma_j^{(n)})^2 \geq \eta$. 
As shown in Figure \ref{fig:singular_values}, the singular values exhibit a sharp decay. Using $\eta = 0.95$, we identify the multilinear rank as $(2, 2, 2)$. We normalize the ground-truth tensor to the range $[0, 1]$, resulting in a minimum singular value $\lambda_{\min} = 0.7309$.

\begin{figure}[h!]
\centering
\includegraphics[width=\textwidth]{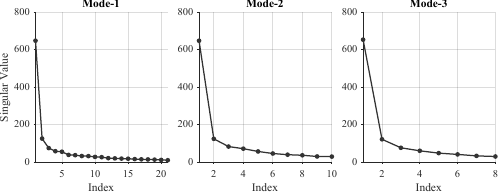}
\caption{Plot of the singular values for each mode unfolding. The rapid decay justifies the low-rank assumption, with the first two components capturing over 95\% of the total energy.}
\label{fig:singular_values}
\end{figure}

\subsection{Experimental Setup and Results}
We simulate a decision-making environment where a retail manager conducts A/B tests to evaluate the profitability of various bundling configurations. In this setting, we treat the constructed tensor as the underlying ground truth of consumer preferences. The estimation noise, $\xi \sim \mathcal{N}(0, \sigma^2)$, encapsulates the uncertainties of the digital marketplace. Specifically, it accounts for market volatility, such as fluctuations in consumer demand driven by seasonal trends or external shocks, and consumer heterogeneity, where the variance in response rates across diverse user segments may not be fully captured within a finite-time A/B test. Furthermore, it reflects observational biases resulting from tracking limitations or the discrepancy between raw clickstream data and actual conversion. 

Because the Taobao data are used to construct an offline proxy environment rather than a live randomized field experiment, we view this exercise as a semi-synthetic evaluation. The real data determine the combinatorial bundle structure and cross-category interaction patterns, while the injected Gaussian noise emulates the finite-sample uncertainty inherent in platform experimentation. This design allows us to evaluate whether the proposed method can exploit realistic interaction structure to improve final policy selection under controlled budget and noise regimes.

We compare three policies. \texttt{Two-stage} is the proposed procedure in Algorithm~\ref{alg:TwoStage_TensorVector}. \texttt{One-shot} uses the entire budget for a single tensor-completion step and then recommends the empirical maximizer of the completed tensor. \texttt{Vector SH} vectorizes the tensor and runs Sequential Halving by treating each bundle as an independent arm. The tensor-completion routine used in both \texttt{Two-stage} and \texttt{One-shot} is Riemannian gradient descent; see Appendix~\ref{sec:appx_alg}.

\begin{figure}[!htp]
\centering
\includegraphics[width=.6\textwidth]{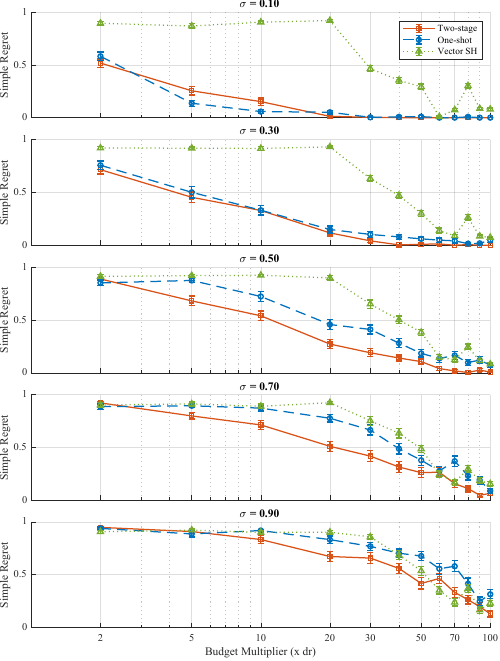}
\caption{Comparative performance across varying noise levels and budgets. The \texttt{Two-stage} algorithm consistently outperforms the \texttt{One-shot} approach, particularly in high-noise and low-budget regimes.}
\label{fig:performance_results}
\end{figure}

The experimental results are summarized over five noise levels, $\sigma \in \{0.1,0.3,0.5,0.7,0.9\}$, and eight budget levels ranging from $2$ to $100$ times the degrees of freedom, where $\df=122$. Varying $\sigma$ from $0.1$ to $0.9$ allows us to evaluate performance across environments ranging from relatively stable markets to highly uncertain product categories. For the \texttt{Two-stage} policy, the Stage-I budget fraction is set to
\(0.3, 0.5, 0.7, 0.8,\) and \(0.9\) at noise levels
\(\sigma=0.1, 0.3, 0.5, 0.7,\) and \(0.9\), respectively. In other words, as \(\sigma\) increases, we allocate a larger share of the total budget to Stage I. This choice reflects the fact that reliable elimination requires a stronger initial structural estimate when observations are more variable. Within Stage I, the budget is split equally across rounds for simplicity and ease of implementation. If an entry is observed multiple times across rounds, we average the corresponding outcomes. Because the design preserves the underlying user population and does not alter the data-generating mechanism across rounds, information collected earlier can be reused in later rounds. All reported results are averaged over 50 independent Monte Carlo trials. Figure~\ref{fig:performance_results} reports the resulting mean regret for the three policies.

A primary finding is the significant performance gap between the tensor-based methods (\texttt{Two-stage} and \texttt{One-shot}) and the \texttt{Vector SH} baseline in low-budget regimes. The total action space contains $21 \times 10 \times 8 = 1,680$ possible bundles (arms). For budgets below $20 \times \df$ (e.g., $T=244, 610, 1220$), the \texttt{Vector SH} method exhibits near-random performance with regret levels exceeding $0.9$. This failure is fundamentally due to the ``exploration overhead'' of sequential halving: SH requires a budget large enough to sample every individual arm multiple times through successive rounds of elimination. In these resource-constrained settings, the budget is insufficient to even visit each arm once, let alone differentiate between them amidst noise. In contrast, by leveraging the low-rank structure, our \texttt{Two-stage} method effectively ``shares'' information across related bundles, allowing it to identify high-performing regions of the action space even when the vast majority of combinations have never been sampled.

The \texttt{Two-stage} algorithm consistently outperforms the \texttt{One-shot} completion policy, particularly as the estimation noise $\sigma$ increases. While \texttt{One-shot} completion utilizes the low-rank assumption to estimate all entries, it is highly sensitive to the realization of noise in the single batch of samples. In high-noise environments (e.g., $\sigma=0.9$), the \texttt{One-shot} estimate often leads to a suboptimal ``greedy'' selection. The \texttt{Two-stage} approach mitigates this by using the first stage to perform a ``principled filtering'' of the combinatorial space, narrowing down the potential candidates to a small subset. The second stage then concentrates the remaining budget on this subset using a sequential halving strategy.

Interestingly, we also observe a phase transition in the performance of \texttt{Vector SH}. When the budget becomes very large, around $70$--$100\times \df$ (for example, $T=12{,}200$), \texttt{Vector SH} begins to approach the performance of the tensor-based methods in some high-noise settings. Intuitively, once the sample size is sufficiently large, the advantage of structural borrowing from the low-rank model becomes less pronounced, while any model misspecification in the tensor approximation becomes relatively more important. In practical retail experimentation, however, managers typically operate in the low-to-medium budget regime, where experimentation is costly and the feasible budget is far below the total number of possible bundles. It is precisely in this practically relevant region that the proposed \texttt{Two-stage} policy delivers the largest gains.

Our results suggest that for e-commerce managers, the ability to exploit cross-category correlations is more valuable than exhaustive testing. By utilizing our algorithm at the product level, firms can generate micro-level insights (e.g., which specific item pairings resonate) and macro-level category strategies without the prohibitive cost of testing the entire combinatorial catalog. This data-driven filtering allows for a more agile response to market trends where the experimental budget is a scarce resource. While our empirical validation focuses on product bundling, the results provide a general blueprint for any high-dimensional design problem. The performance gains we observe demonstrate that whenever 'local' treatment levels share 'global' latent characteristics, centralized tensor designs offer a significant reduction in out-of-sample regret compared to unstructured independent testing.

\subsection{Implementation in Practice}

Several design choices in our procedure can be tailored using practitioner knowledge.

First, the relevant noise level should be interpreted as the variability of the estimated treatment effect under randomized exposure. In the simulation above, noise is introduced exogenously. In practice, however, it reflects the precision with which the platform can measure the outcome of a candidate intervention. Our results show that the method performs well across a wide range of noise levels, but higher-noise environments generally call for a larger Stage-I budget. The reason is that tensor completion requires a sufficiently accurate initial estimate to support reliable elimination. From a managerial perspective, the more uncertain the market environment or the noisier the performance metric, the more budget should be allocated to the first stage to stabilize the structural estimate before refinement.

Second, the choice of switch round $L_I$ should reflect how long the practitioner regards the low-rank approximation as credible along the elimination path. In the bundling application, we consider a $21\times 10\times 8$ tensor and set $L_I=2$, so that Stage I performs two rounds of coordinate-wise halving before switching to Stage II. This choice reflects a practical trade-off. If the manager believes that the low-rank approximation is plausible only for the original full tensor, then a more conservative design would switch earlier, thereby limiting reliance on structural extrapolation and moving sooner to Stage II, where no low-rank assumption is required. By contrast, if the manager believes that the reduced tensors obtained after early elimination continue to preserve the same latent structure, then a larger $L_I$ may yield additional savings by allowing the tensor stage to eliminate more unpromising levels before the final refinement step.

For theoretical transparency, we analyze the procedure under a fixed switch round $L_I$, and in the empirical study we set $L_I=2$. In practice, however, the switch can be selected adaptively. A natural data-driven approach is to reserve a small fraction $\alpha$ of the sampled cells in each tensor round as a validation set, fit the tensor-completion model on the remaining $1-\alpha$ fraction, and monitor out-of-sample prediction accuracy on the holdout cells. Specifically, in round $\ell$, let $V_\ell$ denote the validation set and define the normalized validation error as $\mathrm{RMSE}_\ell := \bigl(|V_\ell|^{-1}\sum_{i\in V_\ell}(Y_i-\widehat \cT_i^{(\ell)})^2\bigr)^{1/2}/\widehat\sigma_\ell$. The procedure then switches from the tensor stage to the vector stage when this validation error ceases to decline materially. One robust rule is to switch if $\mathrm{RMSE}_\ell \geq (1-\tau)\mathrm{RMSE}_{\ell-1}$ for two consecutive rounds, where $\tau$ is a small tolerance parameter, for example between $1\%$ and $3\%$. Requiring persistence over two rounds helps avoid reacting to transitory noise in the validation metric. We view this adaptive switching rule as a practically useful extension of the fixed-$L_I$ design studied in the theory. Establishing formal guarantees for such data-driven switch rules is an important direction for future research.

More broadly, the design offers a flexible template rather than a rigid rule. The structural stage uses low-rank estimation to compress a large action space into a manageable candidate set, while the second stage provides a robust model-agnostic refinement step. This separation makes the method practically attractive: domain knowledge can be incorporated through the Stage-I allocation and switching rule, while the final recommendation remains protected by a second-stage procedure that does not rely on perfect model specification.

%%%%%%%%%%%%%%%%%%%%%%%%%%%%%%%%%%%%%%%%%%%%%%%%%
\clearpage
\renewcommand*{\bibfont}{\footnotesize}
\printbibliography
\newpage

\ECSwitch
% \ECDisclaimer
%%%%%%%%%%%%%%%%%%%%%%%%%%%%%%%%%%%%%%%%%%%%%%%%%%%%%%%%%%
%%% Main head for the e‑companion
\ECHead{Electronic Companion for Policy-Aware Design of Large-Scale Factorial Experiments}
This electronic companion contains two parts. Section~\ref{sec:appx_alg}
illustrates the algorithm through a worked example. Section~\ref{sec:proof}
contains the proofs of the main theoretical results, including the
gap-independent bound in Theorem~\ref{thm:TwoStage-gapind} and the
gap-dependent bound in Theorem~\ref{thm:gap-dep}.

\section{Algorithm Example}\label{sec:appx_alg}
\subsection{A Representative Tensor-Completion Guarantee via Riemannian Gradient Descent}

This appendix summarizes one concrete route, based on Riemannian gradient descent \parencites{vandereycken2013lowrank,wei2016guarantees,cai2022provable,wang2023implicit,cai2023generalized}, for obtaining the tensor-completion guarantee used in Assumption~\ref{assump:completion-tail}. The goal here is not to reproduce the full technical development of the tensor-completion literature, but rather to show that the uniform entrywise control assumed in the main text is consistent with existing computationally efficient methods under standard structural conditions.

Consider an order-$m$ tensor
\(
\cT^\star \in \mathbb{R}^{d \times d \times \cdots \times d}
\)
with multilinear rank
\(
\fb=(r,\dots,r).
\)
We observe a subset of noisy entries
\[
\cY=\cP_{\Omega}(\cT^\star+\cE),
\]
where \(\cP_{\Omega}\) is the sampling operator associated with the observation set
\(
\Omega \subset [d]\times[d]\times\cdots\times[d].
\)
We assume a Bernoulli sampling model in which each entry is included independently with probability \(p\), so that the expected sample size is \(N = p d^m\). The noise tensor \(\cE\) has independent mean-zero sub-Gaussian entries with variance proxy \(\sigma^2\).

A standard low-rank tensor-completion estimator solves the constrained least-squares problem
\[
\min_{\cX \in \mathbb{R}^{d \times \cdots \times d}}
\frac{1}{2p}\bigl\|\cP_{\Omega}(\cX)-\cY\bigr\|_F^2
\qquad
\text{subject to}
\qquad
\rank(\cX)=\fb.
\]
One computationally efficient approach is Riemannian gradient descent (RGM) over the fixed-rank Tucker manifold. The orthogonal projection onto this tangent space admits a closed-form expression, which is used in each Riemannian gradient step. Assume
\[
\cX=\cG \times_1 X_1 \times_2 \cdots \times_m X_m \in \mathbb{M}_{\fr},
\]
where \(\fr=(r_1,\dots,r_m)\), each factor matrix \(X_k\in\mathbb{R}^{d\times r_k}\) has orthonormal columns, and the core tensor
\(
\cG\in\mathbb{R}^{r_1\times\cdots\times r_m}
\)
has full multilinear rank. Let \(\mathbb{M}_{\fr}\) denote the manifold of tensors with multilinear rank \(\fr\). By \textcite{koch2010dynamical}, the tangent space of \(\mathbb{M}_{\fr}\) at \(\cX\) is
\[
T_{\cX}
=
\left\{
\cC \times_1 X_1 \times_2 \cdots \times_m X_m
+
\sum_{k=1}^m
\cG \times_k W_k \times_{j\ne k} X_j
\,\middle|\,
\cC\in\mathbb{R}^{r_1\times\cdots\times r_m},\;
W_k\in\mathbb{R}^{d\times r_k},\;
W_k^\top X_k=0,\; k=1,\dots,m
\right\}.
\]

Let the iterate \(\cX^t\) admit the Tucker decomposition
\[
\cX^t=\cG^t \times_1 X_1^t \times_2 X_2^t \cdots \times_m X_m^t,
\]
where \(X_k^t \in \mathbb{R}^{d\times r}\) has orthonormal columns for each mode \(k\). Then one RGM step takes the form
\[
\cX^{t+1}
=
\operatorname{HOSVD}_{\fb}\!\left(
\cX^t-\frac{1}{p}\,\cP_{T_{\cX^t}}\!\bigl(\cP_{\Omega}(\cX^t)-\cY\bigr)
\right),
\]
where \(T_{\cX^t}\) denotes the tangent space of the Tucker-rank manifold at \(\cX^t\), \(\cP_{T_{\cX^t}}\) is the orthogonal projection onto that tangent space, and \(\operatorname{HOSVD}_{\fb}(\cdot)\) denotes truncation to multilinear rank \(\fb\) (Algorithm \ref{alg:hosvd}).

For completeness, if
\(
\cX=\cG\times_1 X_1\times_2 \cdots \times_m X_m,
\)
then the tangent-space projection of a tensor \(\cZ\) can be written as
\[
\cP_{T_{\cX}}(\cZ)
=
\cZ \times_1 X_1X_1^\top \times_2 \cdots \times_m X_mX_m^\top
+
\sum_{k=1}^m
\cG \times_k W_k \times_{j\ne k} X_j,
\]
where
\[
W_k
=
\bigl(I-X_kX_k^\top\bigr)\,
\cM_k\!\left(\cZ \times_{j\ne k} X_j^\top\right)\,
\cM_k^\dagger(\cG).
\]
Here \(\cM_k(\cdot)\) denotes mode-\(k\) matricization and \(\cM_k^\dagger(\cG)\) is the corresponding pseudoinverse term induced by the core tensor.

Let the true tensor admit the Tucker decomposition
\(
\cT^\star=\cS \times_1 U_1 \times_2 \cdots \times_m U_m,
\)
where \(U_k\in\mathbb{R}^{d\times r}\) contains the top-\(r\) left singular vectors of the mode-\(k\) unfolding \(\cM_k(\cT^\star)\), and \(\cS\in\mathbb{R}^{r\times\cdots\times r}\) is the core tensor. As in matrix completion, incoherence is needed to ensure that the signal is sufficiently spread out across coordinates. We define the incoherence level by
\[
\mu_0
:=
\frac{d}{r}\max_{1\le k\le m}\|U_k\|_{2,\infty}^2.
\]
A small value of \(\mu_0\) means that the singular vectors of the tensor unfoldings are not overly aligned with the canonical basis, so that a random subset of entries still contains enough information for recovery.

We now state a representative guarantee, adapted to the notation of the main text, from the recent literature on nonconvex tensor completion via Riemannian gradient descent; see, for example, \textcite{wang2023implicit}.

\begin{proposition}[Representative entrywise guarantee for RGM]
\label{prop:rgm-representative}
Suppose \(\cT^\star\) is \(\mu_0\)-incoherent, \(\Omega\) follows the Bernoulli sampling model with rate \(p\), and the noise tensor has independent mean-zero sub-Gaussian entries with variance proxy \(\sigma^2\). Let
\[
\lambda_{\min}
:=
\min_{1\le k\le m}\sigma_r\!\bigl(\cM_k(\cT^\star)\bigr),
\qquad
\lambda_{\max}
:=
\max_{1\le k\le m}\sigma_1\!\bigl(\cM_k(\cT^\star)\bigr),
\qquad
\kappa:=\lambda_{\max}/\lambda_{\min}.
\]
Assume further that
\[
d \gtrsim \mu_0^m r^{2m-1}\kappa^6,
\]
\[
p
\gtrsim
\max\!\left\{
\frac{\kappa^8\mu_0^{3m-1}r^{3m-3}\log^3 d}{d^{m/2}},
\;
\frac{\kappa^{16}\mu_0^{4m-5}r^{6m-6}\log^6 d}{d^{m-1}}
\right\},
\]
and
\[
\frac{\sigma}{\lambda_{\min}}
\lesssim
\min\!\left\{
\frac{\sqrt{p}}{d^{m/4}\sqrt{\log d}},
\;
\frac{\sqrt{p}}{\kappa\sqrt{d\,\mu_0^{m-1}r^{m-1}\log d}}
\right\}.
\]
Then, with high probability, the RGM iterates initialized by a spectral method satisfy an entrywise error bound of the form
\[
\infnorm{\cX^t-\cT^\star}
\lesssim
2^{-t}\left(\frac{\mu_0 r}{d}\right)^{m/2}\lambda_{\max}
+
\kappa^4\left(\frac{\mu_0 r}{d}\right)^{m/2}\lambda_{\max}\,\varepsilon_{\mathrm{noise}},
\]
uniformly over a polynomial number of iterations, where \[
\varepsilon_{\mathrm{noise}}
=
\frac{\sigma^2}{\lambda_{\min}^2}\,\frac{d^{m/2}\log d}{p}
\;+\;
\frac{\sigma \kappa}{\lambda_{\min}}
\sqrt{\frac{d\,\mu_0^{\,m-1} r^{\,m-1}\log d}{p}}.
\] is a noise term that decreases as \(p\) increases and \(\sigma\) decreases.
\end{proposition}

Proposition~\ref{prop:rgm-representative} shows that RGM converges geometrically to a noise-determined neighborhood of the true tensor. In the benchmark regime where \(r\), \(\mu_0\), and \(\kappa\) are all bounded by constants, the required sampling rate is roughly
\[
p \gtrsim d^{-m/2}\,\polylog(d),
\]
which corresponds to
\[
N = p d^m \gtrsim d^{m/2}\,\polylog(d).
\]
This is exactly the square-root scaling in the design-space size discussed in Remark~\ref{rem:conditions}. Under the same regime, the resulting error floor is consistent with the qualitative form assumed in Assumption~\ref{assump:completion-tail}: larger sample size and stronger signal improve recovery, while larger noise worsens it.

For the purposes of the main paper, the key implication is not the precise constant or exponent in Proposition~\ref{prop:rgm-representative}, but the broader message: under standard incoherence, sample-size, and signal-to-noise conditions, computationally efficient tensor-completion algorithms can provide the uniform entrywise control needed by our elimination analysis.

\subsubsection{Algorithmic components.}
For reference, one practical implementation consists of HOSVD truncation, spectral initialization with diagonal deletion, and Riemannian gradient descent.

\begin{algorithm}[!ht]
\caption{HOSVD truncation}
\label{alg:hosvd}
\DontPrintSemicolon
\KwIn{Tensor \(\cX\in\mathbb{R}^{d\times\cdots\times d}\), multilinear rank \(\fb=(r,\dots,r)\)}
\For{\(k=1,\dots,m\)}{
Compute \(X_k=\operatorname{SVD}_r(\cM_k(\cX))\), the top-\(r\) left singular vectors of the mode-\(k\) unfolding.\;
}
Set
\[
\cG=\cX\times_1 X_1^\top \times_2 \cdots \times_m X_m^\top.
\]
\KwOut{\(\cG\times_1 X_1\times_2 \cdots \times_m X_m\)}
\end{algorithm}

\begin{algorithm}[!ht]
\caption{Spectral initialization with diagonal deletion}
\label{alg:spectral_init}
\DontPrintSemicolon
\KwIn{Observed tensor \(Y=\cP_\Omega(\cT^\star+\cE)\), multilinear rank \(\fb=(r,\dots,r)\), sampling rate \(p\)}
\For{\(k=1,\dots,m\)}{
Form the mode-\(k\) unfolding \(\widehat T_k=\cM_k(Y)\).\;
Compute the top-\(r\) eigenspace of \(\cP_{\mathrm{off\mbox{-}diag}}(\widehat T_k\widehat T_k^\top)\), and denote the resulting factor matrix by \(X_k^1\).\;
}
Set
\[
\cG^1
=
p^{-1}Y\times_1 (X_1^1)^\top \times_2 \cdots \times_m (X_m^1)^\top.
\]
\KwOut{\(\cX^1=\cG^1\times_1 X_1^1\times_2 \cdots \times_m X_m^1\)}
\end{algorithm}

\begin{algorithm}[!ht]
\caption{Riemannian gradient descent on the Tucker manifold}
\label{alg:rgm}
\DontPrintSemicolon
\KwIn{Initialization \(\cX^1\), multilinear rank \(\fb=(r,\dots,r)\), sampling rate \(p\), observations \(Y\)}
\For{\(t=1,2,\dots\)}{
Set
\[
\cZ^t
=
\cX^t-\frac{1}{p}\,\cP_{T_{\cX^t}}\!\bigl(\cP_\Omega(\cX^t)-Y\bigr).
\]
\For{\(k=1,\dots,m\)}{
Compute \(X_k^{t+1}=\operatorname{SVD}_r(\cM_k(\cZ^t))\).\;
}
Set
\[
\cG^{t+1}
=
\cZ^t\times_1 (X_1^{t+1})^\top \times_2 \cdots \times_m (X_m^{t+1})^\top,
\]
and
\[
\cX^{t+1}
=
\cG^{t+1}\times_1 X_1^{t+1}\times_2 \cdots \times_m X_m^{t+1}.
\]
}
\end{algorithm}

\section{Proofs}\label{sec:proof}
This Online Appendix contains the proofs for the main results in the paper.
\begin{itemize}
\item \textbf{Section~\ref{sec:proof of theorem gapind}} proves the gap-independent bound in
Theorem~\ref{thm:TwoStage-gapind}.

\item \textbf{Section~\ref{sec:proof of gap-dep thm}} proves the gap-dependent guarantee in
Theorem~\ref{thm:gap-dep}.
\end{itemize}

For readability we assume $d$ is a power of two so that the median prune implements exact halving.
The non-dyadic case only changes floors/ceilings and constants and is omitted.
Set
\[
d_\ell\ :=\ d/2^{\ell-1},\qquad \ell=1,2,\ldots,
\]
and let $\cA_k^{(\ell)}\subseteq[d]$ be the surviving indices in mode $k\in[m]$ at the beginning of round $\ell$, with $\cA^{(\ell)}:=\bigotimes_{k=1}^m \cA_k^{(\ell)}$. For each mode $k$, round $\ell\in[\Lsw]$, index $i\in\Aellk$, define the (restricted) fiber-maxima
\[
\mu^{(\ell)}_{k,i}\;:=\;\max_{\bj\in \prod_{k'\neq k}\Aell_{k'}} \cT^{(\ell)}_{\,(i,\bj)},
\qquad
\hat\mu^{(\ell)}_{k,i}\;:=\;\max_{\bj\in \prod_{k'\neq k}\Aell_{k'}} \widehat{\cT}^{(\ell)}_{\,(i,\bj)}.
\]
Here, $\max\left(\cX\right)$ represents the largest value (not absolute value) of the entries of $\cX$ and $\bj\in \prod_{k'\neq k}\Aell_{k'}$ is a $m-1$ dimension vector. $\cT^{(\ell)}$ is the true tensor $\cT^\star$ restricted at design space $\Aell$. We set $\mu_{k,i} := \mu_{k,i}^{(1)}$ and define the mode-$k$ order statistics 
\[
\mu_{\star} = \mu_{k,(1)}\ge \cdots\ge \mu_{k,(d)}\quad\text{(mode-$k$ order statistics)},
\]
and the mode-wise gaps $\Delta_{k,t}:=\mu_{k,(1)}-\mu_{k,(t)}$. Define global gaps $\Delta_{t}:=\min_{k\in[m]}\Delta_{k,t}$.
Define the good-count functions
\[
g_k(\eps)\ :=\ \bigl|\{i\in[d]:\ \mu_{k,i}\ge \mu_{k,(1)}-\eps\}\bigr|,
\quad
g_{k,\ell}(\eps)\ :=\ \bigl|\{i\in\cA_k^{(\ell)}:\ \mu_{k,i}\ge \mu_{k,(1)}-\eps\}\bigr|.
\]
Let $g_{\max}(\eps):=\max_{k\in[m]} g_k(\eps)$ and $g_{\max,\ell}(\eps):=\max_{k\in[m]} g_{k,\ell}(\eps)$.

Finally, let \(\{\delta_{(t)}\}_{t=1}^{d^m}\) denote the decreasing order statistics of the \(d^m\) entries of \(\cT^\star\), viewed as a vector. We use \(\Ainfnorm{\widehat{\cT}-\cT^\star}\) as shorthand for the sup-norm over the current active set \(\cA^{(\ell)}\). In other words,
\[
\Ainfnorm{\widehat{\cT}-\cT^\star}
:=
\Ainfnorm{\widehat{\cT}^{(\ell)}-\cT^{(\ell)}}.
\]

\begin{assumption}[Roundwise conditional sup-norm tail on the active set]\label{assump:concentration-appendix}
Assume the Stage-I switch round satisfies \(\Lsw \le \lfloor \log_2(d/r)\rfloor\).  
For each Stage-I round \(\ell\in[\Lsw]\), let
\(
\cA^{(\ell)}=\cA_1^{(\ell)}\times\cdots\times \cA_m^{(\ell)}
\)
denote the active design set, and let \(\cT^{(\ell)}\) be the restriction of \(\cT^\star\) to \(\cA^{(\ell)}\). For each \(\ell\), let
\( \lambda_{\min}^{(\ell)} := \min_{k\in[m]} \sigma_r\!\big(\mathcal M_k(\cT^{(\ell)})\big), \)
and assume \(\lambda_{\min}^{(\ell)}\ge \lambda_{\min}\) for all \(\ell\in[\Lsw]\). 
Let \(d_\ell := |\cA_k^{(\ell)}|\), and let \(\widehat{\cT}^{(\ell)}\) be the tensor-completion estimate constructed from the round-\(\ell\) data. Let \(\{\mathcal F_\ell\}_{\ell\ge1}\) be a filtration such that \(\cA^{(\ell)}\) (equivalently, the round-\(\ell\) design) is \(\mathcal F_\ell\)-measurable.   Define the round-\(\ell\) sampling rate and effective degrees of freedom as
\[
p_\ell := \frac{N_\ell}{|\cA^{(\ell)}|} = \frac{N_\ell}{d_\ell^m},
\qquad
\df_\ell := m(d_\ell r-r^2)+r^m \asymp m d_{\ell}+r^m.
\]

Assume there exist absolute constants \(c,C_m>0\) and a deterministic constant \(\lambda_{\min}>0\) (a uniform lower bound on the relevant signal strength across Stage-I rounds) such that, for every \(\ell\in[\Lsw]\) and every \(t>0\),
\[
\bP\!\left(
\Ainfnorm{\widehat{\cT}^{(\ell)}-\cT^{(\ell)}} \ge t
\,\middle|\, \mathcal F_\ell
\right)
\le
C_m \exp\!\left\{
-\,c\,t^2\,
\frac{p_\ell\,\lambda_{\min}^{\,2}}
{\sigma^2\,\df_\ell\,\log d_\ell\,\infnorm{\cT^\star}^2}
\right\}
\;=:\;\Phi_\ell(t),
\]
where the restricted sup-norm is $\Ainfnorm{\cX}:=\max_{\boldsymbol{i}\in\Aell} |\cX_{\bidx}|$. Since $\mu_{k,(1)} \geq \mu_{k,(2)} \geq \cdots$, the gaps $\Delta_{k, t}=\mu_{k,(1)}-\mu_{k,(t)}$ are nondecreasing in $t$, the tail function $\Phi_{\ell}(t)$ is nonincreasing in $t$.
\end{assumption}

\subsubsection*{Proof sketch (Theorem~\ref{thm:TwoStage-gapind}).}
Fix $\epsilon>0$ and choose $t:=\epsilon/(4(\Lsw+1))$. The proof proceeds in three steps.

\smallskip
\noindent\textbf{Step 1: Stage~I uniform accuracy event.}
Define $F_\ell(t)$ to be the event that the round-$\ell$ tensor completion estimate is uniformly
accurate on the active set, and let $E_{\mathrm I}(t):=\bigcap_{\ell=1}^{\Lsw}F_\ell(t)$.
Using the roundwise conditional tail bound in Assumption~\ref{assump:concentration-appendix} and
iterated conditioning on the filtration $\{\mathcal F_\ell\}$, we show that
$\Pr(E_{\mathrm I}(t)^c)$ is at most a sum of $\Phi_\ell(t)$ terms, hence decays exponentially in
$t^2 p_{\min}$ up to logarithmic factors.

\smallskip
\noindent\textbf{Step 2: Value retention to the switch.}
On $E_{\mathrm I}(t)$, the entrywise errors satisfy
$|\widehat{\cT}^{(\ell)}_{\bidx}-\cT^{(\ell)}_{\bidx}|\le t$ uniformly over $\bidx\in\cA^{(\ell)}$.
With the tie-breaking convention stated in the proof of Lemma~\ref{lem:gap-indep-stage2},
the empirical maximizer $\hat{\bidx}^{(\ell)}\in\arg\max_{\bidx\in\cA^{(\ell)}}\widehat{\cT}^{(\ell)}_{\bidx}$
survives the coordinate-wise median prune, implying the recursive inequality
$v_{\ell+1}\ge v_\ell-2t$, where $v_\ell:=\max_{\bidx\in\cA^{(\ell)}}\cT^{(\ell)}_{\bidx}$.
Iterating yields $v_{\Lsw+1}\ge \mu_\star-2\Lsw t\ge \mu_\star-\epsilon/2$.

\smallskip
\noindent\textbf{Step 3: Stage~II Sequential Halving and tail-to-regret integration.}
Condition on the surviving arm set $\cA^{(\Lsw+1)}$. Since the best surviving mean is at least
$\mu_\star-\epsilon/2$ on $E_{\mathrm I}(t)$, a standard Sequential Halving guarantee implies that,
with Stage~II budget $\Nps$, the probability of returning an arm worse than $\epsilon$ from optimal
decays as $\exp\{-\Theta(\epsilon^2\Nps/(\sigma^2 \Kps \Lps))\}$ up to constants.
Combining the Stage~I and Stage~II tails yields a two-term bound of the form
$\min\{1,Ae^{-B\epsilon^2}\}+\min\{1,A'e^{-B'\epsilon^2}\}$, and integrating over
$\epsilon$ gives the stated simple-regret bound.

\subsection{Proof of Theorem \ref{thm:TwoStage-gapind}}\label{sec:proof of theorem gapind} 
By the definition of simple regret, we have:
\begin{align}\label{equ:def of simple regret}
  \mathrm{SReg}_N \;=\; \mathbb{E}[\mu_\star-\mu_{J_N}] \;=\; \int_{0}^{\infty} \bP\!\big(\mu_\star-\mu_{J_N}>\epsilon\big)\,\mathrm{d}\epsilon.
\end{align}

We now decompose the regret into contributions from the \TStage\ and \VS. Let $\tilde\mu_{\star} := \max_{\bidx\in\cA^{(\Lsw+1)}}\cT^{(\Lsw+1)}_{\bidx}$ denote the best entry of the true tensor restricted to the final Stage I design set $\cA_{(\Lsw+1)}$. For $\ell\in[\Lsw]$ and a threshold $t>0$, define the “good estimation” events:
\begin{align}\label{equ:def of events}
  \begin{aligned}
    F_\ell(t) &:= \Big\{\Ainfnorm{\widehat{\cT}-\cT^{\star}} \le t\Big\}, \qquad E_{\mathrm I}(t) :=\bigcap_{\ell=1}^{\Lsw} F_\ell(t),\\
    E_{\mathrm{II}}(\epsilon') &:=\{\mu_{J_N}<\tilde\mu_{\star}-\epsilon'\}. 
  \end{aligned}
\end{align}

Fix any $\epsilon>0$. By a standard event decomposition:
\begin{align}\label{eq:event-factorization-F}
  \begin{aligned}
    \bP\!\big(\mu_\star-\mu_{J_N}>\epsilon\big)
    &=\bP\!\big(\mu_\star-\mu_{J_N}>\epsilon,\ E_{\mathrm I}(t)\big) \,+\, \bP\!\big(\mu_\star-\mu_{J_N}>\epsilon,\ E_{\mathrm I}(t)^{c}\big) \\
    &\le \bP\!\big(\mu_\star-\mu_{J_N}>\epsilon\ \big|\ E_{\mathrm I}(t)\big) \,+\, \bP\!\big(E_{\mathrm I}(t)^{c}\big).
  \end{aligned}
\end{align}

The next two lemmas control the two terms on the right-hand side of \eqref{eq:event-factorization-F}, corresponding directly to the \TStage\ and \VS\ of the algorithm.

Define
\[
p_{\min}:=\min_{\ell\in[L_I]} p_\ell.
\]
\begin{lemma}\label{lem:gap-indep-stage1}
 Under Assumption~\ref{assump:concentration-appendix}, for all $\epsilon>0$ and setting $t=\epsilon/(4(\Lsw+1))$, we have:
  \begin{align*}
    \bP\!\big(E_{\mathrm I}(t)^{c}\big) \;\leq\; C_m \Lsw \exp \left\{-c \epsilon^2 \frac{p_{\min}\lambda_{\min}^2}{\sigma^2 \df\Lsw^2 \log d \infnorm{\cT^\star}^2} \right\}.
  \end{align*}
\end{lemma}
\noindent (Proof provided in Section~\ref{sec:proof of lemma gap-indep-stage1}.)

\begin{lemma}[Gap-independent Stage-II guarantee on the surviving set]
\label{lem:gap-indep-stage2}
Let
\[
\tilde\mu_\star
:=
\max_{\boldsymbol i\in \cA_{\Lsw+1}}
\cT^{(\Lsw+1)}_{\boldsymbol i}.
\]
Conditional on the realized Stage-I surviving set \(\cA_{\Lsw+1}\), Stage II is standard Sequential Halving run on \(\Kps\) surviving arms for \(\Lps\) rounds with total Stage-II budget \(\Nps\). Assume the Stage-II empirical arm evaluations are sub-Gaussian with variance proxy \(\sigma^2\). Then the following hold.

\begin{enumerate}
\item[(i)] For every \(x>0\),
\[
\bP\!\left(
\tilde\mu_\star-\mu_{J_N}>x
\,\middle|\,
E_{\mathrm I}(t)
\right)
\;\le\;
3\,\Lps\,
\exp\!\left\{
-\,x^2\frac{\Nps}{32\,\sigma^2\,\Kps\,\Lps}
\right\}.
\]

\item[(ii)] If \(2\Lsw t\le \epsilon/2\), then
\[
\bP\!\left(
\mu_\star-\mu_{J_N}>\epsilon
\,\middle|\,
E_{\mathrm I}(t)
\right)
\;\le\;
3\,\Lps\,
\exp\!\left\{
-\,\epsilon^2\frac{\Nps}{128\,\sigma^2\,\Kps\,\Lps}
\right\}.
\]
In particular, the condition \(2\Lsw t\le \epsilon/2\) holds for example when
\(t=\epsilon/(4(\Lsw+1))\).
\end{enumerate}
\end{lemma}
\noindent (Proof provided in Section~\ref{sec:proof of lemma gap-indep-stage2}.)

Combining Lemmas~\ref{lem:gap-indep-stage1} and \ref{lem:gap-indep-stage2} with \eqref{eq:event-factorization-F}, we obtain, for absolute constants $c, C_m>0$:
\begin{align}\label{equ:result of eps prob in stage 1}
  \bP\big(\mu_\star-\mu_{J_N}>\epsilon\big) \le \underbrace{C_m \Lsw \exp \left\{-c \epsilon^2 \frac{p_{\min}\lambda_{\min}^2}{\sigma^2 \df\Lsw^2 \log d \infnorm{\cT^\star}^2} \right\}}_{\TStage} \;+\; \underbrace{3\,\Lps \exp\!\left\{-\epsilon^2 \frac{\Nps}{128\,\sigma^2 \Kps\,\Lps} \right\}}_{\VS}.
\end{align}

For notational convenience, set $A_{\mathrm I}=C_m \Lsw$, $B_{\mathrm I}=c\,\frac{p_{\min}\lambda_{\min}^2}{\sigma^2\,\df\,\Lsw^2\,\log d\,\infnorm{\cT^\star}^2}$, $A_{\mathrm{II}}=3\,\Lps$, and $B_{\mathrm{II}}=\frac{\Nps}{128\,\sigma^2\,\Kps\,\Lps}$. Since probabilities are bounded by $1$, we write:
\[
  \bP\big(\mu_\star-\mu_{J_N}>\epsilon\big) \;\le\; \min\{1, A_{\mathrm I}e^{-B_{\mathrm I}\epsilon^2}\} \;+\; \min\{1, A_{\mathrm{II}}e^{-B_{\mathrm{II}}\epsilon^2}\}.
\]

Using the inequality $\min\{1,x+y\}\le \min\{1,x\} + \min\{1,y\}$ and applying the Gaussian tail integral bound term-wise (replacing $\log A$ with $\log(eA)$ to cover the case $A\le1$), we obtain from \eqref{equ:def of simple regret} that:
\[
  \mathrm{SReg}_N \;\le\; \underbrace{\sqrt{\frac{\log(eA_{\mathrm I})}{B_{\mathrm I}}} \;+\; \frac{1}{2\sqrt{B_{\mathrm I}\,\log(eA_{\mathrm I})}}}_{=:~S_{\mathrm I}} \;+\; \underbrace{\sqrt{\frac{\log(eA_{\mathrm{II}})}{B_{\mathrm{II}}}} \;+\; \frac{1}{2\sqrt{B_{\mathrm{II}}\,\log(eA_{\mathrm{II}})}}}_{=:~S_{\mathrm{II}}}.
\]

Plugging in our constants and simplifying:
\[
  S_{\mathrm I} \;=\; \frac{\sigma\infnorm{\cT^\star}}{\lambda_{\min}\sqrt{c}}\, \sqrt{\frac{\df\,\Lsw^2\,\log d}{p_{\min}}}\, \Bigg( \sqrt{\log\!\big(eC_m\Lsw\big)} \;+\; \frac{1}{2\sqrt{\log\!\big(eC_m\Lsw\big)}} \Bigg),
\]
and
\[
  S_{\mathrm{II}} \;=\; \sigma\, \sqrt{\frac{128\,\Kps\,\Lps}{\Nps}}\, \Bigg( \sqrt{\log\!\big(e\cdot 3\Lps\big)} \;+\; \frac{1}{2\sqrt{\log\!\big(e\cdot 3\Lps\big)}} \Bigg).
\]

Since $\log\!\big(eC_m\Lsw\big)=O(\log\Lsw)$ and $\log\!\big(e\cdot 3\Lps\big)=O(\log\Lps)$, the lower-order additive terms vanish asymptotically, and we conclude:
\[
  \mathrm{SReg}_N \;=\;  O\!\left( \sigma\left[ \sqrt{\frac{\df\,\Lsw^2\,\log d}{p_{\min}}\,\log \Lsw}\frac{\infnorm{\cT^\star}}{\lambda_{\min}} \;+\; \sqrt{\frac{\Kps\,\Lps}{\Nps}\,\log \Lps} \right] \right).
\]

\subsubsection{Proof of Lemma \ref{lem:gap-indep-stage1}}\label{sec:proof of lemma gap-indep-stage1}
\begin{proof}
For each $\ell\in[\Lsw]$, let \( \mathcal H_{\ell-1} \) denote the sigma-field generated by all randomness revealed up to the end of round $\ell-1$
(including the Stage-I active sets and all observations used in rounds $1,\dots,\ell-1$).
Let $\mathcal F_\ell$ be the sigma-field with respect to which the round-$\ell$ design is measurable
(as in Assumption~\ref{assump:concentration-appendix}); in particular, we assume
\(
\mathcal H_{\ell-1}\subseteq \mathcal F_\ell .
\)
Define
\[
B_{\ell-1}:=\bigcap_{i<\ell} F_i(t).
\]
Since each $F_i(t)$ depends only on the round-$i$ data, we have $F_i(t)\in\mathcal H_i\subseteq \mathcal H_{\ell-1}$ for $i<\ell$,
hence
\[
B_{\ell-1}\in\mathcal H_{\ell-1}\subseteq \mathcal F_\ell .
\]

Now fix $\ell$. If $\bP(B_{\ell-1})=0$, then the conditional probability
$\bP(F_\ell(t)^c\mid B_{\ell-1})$ is immaterial. Assume $\bP(B_{\ell-1})>0$.
Using $B_{\ell-1}\in\mathcal F_\ell$ and the tower property,
\begin{align*}
\bP\big(F_\ell(t)^c\cap B_{\ell-1}\big)
&=\bE\!\left[\mathbf 1_{B_{\ell-1}}\mathbf 1_{F_\ell(t)^c}\right] \\
&=\bE\!\left[\mathbf 1_{B_{\ell-1}}\,
    \bE\!\left[\mathbf 1_{F_\ell(t)^c}\mid \mathcal F_\ell\right]\right] \\
&=\bE\!\left[\mathbf 1_{B_{\ell-1}}\,
    \bP\!\left(F_\ell(t)^c\mid \mathcal F_\ell\right)\right].
\end{align*}
By Assumption~\ref{assump:concentration-appendix},
\[
\bP\!\left(F_\ell(t)^c\mid \mathcal F_\ell\right)\le \Phi_\ell(t)\qquad \text{a.s.}
\]
Therefore,
\[
\bP\big(F_\ell(t)^c\cap B_{\ell-1}\big)
\le \Phi_\ell(t)\,\bE[\mathbf 1_{B_{\ell-1}}]
= \Phi_\ell(t)\,\bP(B_{\ell-1}).
\]
Dividing by $\bP(B_{\ell-1})$ gives
\[
\bP\!\left(F_\ell(t)^c\mid \bigcap_{i<\ell}F_i(t)\right)
=\bP\!\left(F_\ell(t)^c\mid B_{\ell-1}\right)
\le \Phi_\ell(t).
\]

Hence,
\[
\bP\big(E_{\mathrm I}(t)^c\big)
=\bP\!\left(\bigcup_{\ell=1}^{\Lsw}F_\ell(t)^c\right)
\le \sum_{\ell=1}^{\Lsw}\bP\!\left(F_\ell(t)^c\mid \bigcap_{i<\ell}F_i(t)\right)
\le \sum_{\ell=1}^{\Lsw}\Phi_\ell(t),
\]
where we used $1-\prod_{\ell=1}^{\Lsw} a_\ell\le \sum_{\ell=1}^{\Lsw}(1-a_\ell)$ with
$a_\ell=\bP\!\big(F_\ell(t)\,\big|\,\cap_{i<\ell}F_i(t)\big)\in[0,1]$.
Using \cref{assump:concentration-appendix} and $\df_\ell\le\df$, $\log d_\ell\le\log d$, $t^2=\epsilon^2/(16(\Lsw+1)^2)$, we have
\begin{align*}
\sum_{\ell=1}^{\Lsw } \Phi_{\ell}(t) \leq C_m \Lsw  \exp \left\{-c \epsilon^2 \frac{p_{\min} \lambda_{\min}^2}{\sigma^2 \df\Lsw^2 \log d \infnorm{\cT^\star}^2} \right\}.
\end{align*}
\end{proof}

%%%%%%%%%%%%%%%%%%%%%%%%%%%%%%%%%%%

\subsubsection{Proof of Lemma \ref{lem:gap-indep-stage2}}
\label{sec:proof of lemma gap-indep-stage2}
For each round $\ell \le \Lsw$, define
\[
v_\ell \;:=\; \max_{\boldsymbol i\in \cA_\ell}\cT^{(\ell)}_{\boldsymbol i},
\qquad
\hat v_\ell \;:=\; \max_{\boldsymbol i\in \cA_\ell}\widehat{\cT}^{(\ell)}_{\boldsymbol i},
\qquad
\hat{\boldsymbol i}^{(\ell)}\in\arg\max_{\boldsymbol i\in \cA_\ell}\widehat{\cT}^{(\ell)}_{\boldsymbol i}.
\]
Here, $v_\ell$ denotes the maximum entry of the (round-$\ell$) true effect tensor over the round-specific candidate set $\cA_\ell$.
Likewise, $\hat v_\ell$ is the maximum entry of the estimated effect tensor $\widehat{\cT}^{(\ell)}$ over $\cA_\ell$. At each Stage-I round $\ell$, pick a deterministic maximizer $\hat{\boldsymbol{i}}^{(\ell)} \in \arg \max _{\boldsymbol{i} \in \cA^{(\ell)}}{\widehat{\cT_{\boldsymbol{i}}}}^{(\ell)}$ 
When pruning mode $k$, if ties at the cutoff occur, break ties so that $\hat{i}_k^{(\ell)} \in \cA_k^{(\ell+1)}$.
(This is always feasible because $\hat{\mu}_{k, \hat{i}_k^{(\ell)}}^{(\ell)}=\hat{v}_{\ell}$ is a top score; if $\hat{i}_k^{(\ell)}$ would be excluded only due to ties, swap it in with another tied index.)
	\begin{lemma}[Stage-I value retention via estimated maximizer]\label{lem:stage1-retain}
On the event $F_\ell(t)$ defined at \cref{equ:def of events}, we have
\[
v_{\ell+1}\ \ge\ v_\ell - 2t.
\]
Consequently, on $E_{\mathrm I}(t)=\bigcap_{\ell=1}^{\Lsw}F_\ell(t)$,
\[
v_{\Lsw+1}\ \ge\ \mu_\star - 2\Lsw\,t.
\]
\end{lemma}
Proof in Section \ref{sec:Proof of Lemma stage1-retain}. 
\begin{proof}
We prove the two parts in order.

\medskip
\noindent\textbf{Proof of (i).}
Fix a realization of the Stage-I history such that the surviving set is
\(\cA_{\Lsw+1}=A\), where \(|A|=\Kps\). Conditional on this realization,
Stage II is exactly the standard Sequential Halving procedure on the
deterministic arm set \(A\), using only the Stage-II samples. Its best arm value is
\[
\tilde\mu_\star(A)
:=
\max_{\boldsymbol i\in A}\cT^\star_{\boldsymbol i}.
\]
Since \(\cT^{(\Lsw+1)}\) is just the restriction of \(\cT^\star\) to \(A\), we also have
\[
\tilde\mu_\star(A)
=
\max_{\boldsymbol i\in A}\cT^{(\Lsw+1)}_{\boldsymbol i}.
\]

Therefore, by Theorem~7 of \cite{zhao2023revisiting}, for every \(x>0\),
\[
\bP\!\left(
\tilde\mu_\star(A)-\mu_{J_N}>x
\,\middle|\,
\cA_{\Lsw+1}=A
\right)
\;\le\;
3\,\Lps\,
\exp\!\left\{
-\,x^2\frac{\Nps}{32\,\sigma^2\,\Kps\,\Lps}
\right\}.
\]
The right-hand side is uniform over all realizations \(A\) with \(|A|=\Kps\).
Now take conditional expectation over the random survivor set \(\cA_{\Lsw+1}\)
given \(E_{\mathrm I}(t)\). Using the tower property,
\begin{align*}
\bP\!\left(
\tilde\mu_\star-\mu_{J_N}>x
\,\middle|\,
E_{\mathrm I}(t)
\right)
&=
\bE\!\left[
\bP\!\left(
\tilde\mu_\star-\mu_{J_N}>x
\,\middle|\,
\cA_{\Lsw+1}
\right)
\,\middle|\,
E_{\mathrm I}(t)
\right] \\
&\le
3\,\Lps\,
\exp\!\left\{
-\,x^2\frac{\Nps}{32\,\sigma^2\,\Kps\,\Lps}
\right\}.
\end{align*}
This proves part (i).

\medskip
\noindent\textbf{Proof of (ii).}
By Lemma~\ref{lem:stage1-retain}, on the event \(E_{\mathrm I}(t)\),
\[
v_{\Lsw+1}\ge \mu_\star-2\Lsw t.
\]
Since
\[
\tilde\mu_\star
=
\max_{\boldsymbol i\in \cA_{\Lsw+1}}
\cT^{(\Lsw+1)}_{\boldsymbol i}
=
v_{\Lsw+1},
\]
it follows that on \(E_{\mathrm I}(t)\),
\[
\tilde\mu_\star\ge \mu_\star-2\Lsw t.
\]
Under the assumption \(2\Lsw t\le \epsilon/2\), we obtain on \(E_{\mathrm I}(t)\),
\[
\tilde\mu_\star\ge \mu_\star-\frac{\epsilon}{2}.
\]
Hence, again on \(E_{\mathrm I}(t)\),
\[
\mu_\star-\mu_{J_N}>\epsilon
\quad\Longrightarrow\quad
\tilde\mu_\star-\mu_{J_N}
>
\epsilon-\frac{\epsilon}{2}
=
\frac{\epsilon}{2}.
\]
Equivalently,
\[
\{\mu_\star-\mu_{J_N}>\epsilon\}\cap E_{\mathrm I}(t)
\subseteq
\left\{\tilde\mu_\star-\mu_{J_N}>\frac{\epsilon}{2}\right\}\cap E_{\mathrm I}(t).
\]
Taking conditional probabilities given \(E_{\mathrm I}(t)\), we get
\[
\bP\!\left(
\mu_\star-\mu_{J_N}>\epsilon
\,\middle|\,
E_{\mathrm I}(t)
\right)
\le
\bP\!\left(
\tilde\mu_\star-\mu_{J_N}>\frac{\epsilon}{2}
\,\middle|\,
E_{\mathrm I}(t)
\right).
\]
Now apply part (i) with \(x=\epsilon/2\):
\begin{align*}
\bP\!\left(
\mu_\star-\mu_{J_N}>\epsilon
\,\middle|\,
E_{\mathrm I}(t)
\right)
&\le
3\,\Lps\,
\exp\!\left\{
-\,(\epsilon/2)^2\frac{\Nps}{32\,\sigma^2\,\Kps\,\Lps}
\right\} \\
&=
3\,\Lps\,
\exp\!\left\{
-\,\epsilon^2\frac{\Nps}{128\,\sigma^2\,\Kps\,\Lps}
\right\}.
\end{align*}
This proves part (ii).

Finally, if \(t=\epsilon/(4(\Lsw+1))\), then
\[
2\Lsw t
=
\frac{\epsilon\Lsw}{2(\Lsw+1)}
\le
\frac{\epsilon}{2},
\]
so the displayed condition is satisfied.
\end{proof}

\subsubsection{Proof of Lemma \ref{lem:stage1-retain}}\label{sec:Proof of Lemma stage1-retain}

\begin{proof}
The estimated maximizer survives the median prune in every mode.
For any mode $k$,
\[
\hat\mu^{(\ell)}_{k,\hat i^{(\ell)}_k}
=\max_{\bj}\widehat{\cT}^{(\ell)}_{(\hat i^{(\ell)}_k,\bj)}
=\widehat{\cT}^{(\ell)}_{\hat{\boldsymbol i}^{(\ell)}}
=\hat v_\ell,
\]
which is the largest entry in the list $\{\hat\mu^{(\ell)}_{k,i}: i\in\cA_k^{(\ell)}\}$ and hence at least its median. Therefore $\hat i^{(\ell)}_k\in\cA_k^{(\ell+1)}$ for all $k$, so $\hat{\boldsymbol i}^{(\ell)}\in\cA^{(\ell+1)}$.

On $F_\ell(t)$, Lemma~\ref{lem:max-Lip} yields $|\widehat{\cT}^{(\ell)}_{\boldsymbol i}-\cT^{(\ell)}_{\boldsymbol i}|\le t$ for all $\boldsymbol i$, hence
\[
\cT^{(\ell)}_{\hat{\boldsymbol i}^{(\ell)}}\ \ge\ \widehat{\cT}^{(\ell)}_{\hat{\boldsymbol i}^{(\ell)}}-t
=\max_{\boldsymbol i\in \cA_\ell}\widehat{\cT}^{(\ell)}_{\boldsymbol i}-t
\ \ge\ \max_{\boldsymbol i \in \cA_\ell}\cT^{(\ell)}_{\boldsymbol i}-2t
= v_\ell-2t.
\]
where we used $\max_{\boldsymbol i}\widehat{\cT}^{(\ell)}_{\boldsymbol i}\ge \max_{\boldsymbol i}\cT^{(\ell)}_{\boldsymbol i}-t$ and Lipschitzness of max function.
Since $\hat{\boldsymbol i}^{(\ell)}\in\cA^{(\ell+1)}$ and $\cT^{(\ell+1)}$ is the restriction of $\cT^\star$ to $\cA^{(\ell+1)}$,
\(
\cT^{(\ell)}_{\hat{\boldsymbol i}^{(\ell)}}=\cT^{(\ell+1)}_{\hat{\boldsymbol i}^{(\ell)}}
\),
so
\[
v_{\ell+1}
=\max_{\boldsymbol i\in\cA^{(\ell+1)}} \cT^{(\ell+1)}_{\boldsymbol i}
\ \ge\
\cT^{(\ell+1)}_{\hat{\boldsymbol i}^{(\ell)}}
=\cT^{(\ell)}_{\hat{\boldsymbol i}^{(\ell)}}
\ \ge\
v_\ell - 2t.
\]
Iterating over $\ell=1,\dots,\Lsw$ and noting $v_1=\mu_\star$ gives $v_{\Lsw+1}\ge \mu_\star-2\Lsw t$.
\end{proof}

\begin{lemma}[Fiber-max is 1-Lipschitz in the restricted sup-norm]\label{lem:max-Lip}
Fix a Stage-I round \(\ell\in[\Lsw]\), a mode \(k\in[m]\), and an index \(i\in \cA_k^{(\ell)}\). Define
\[
\hat\mu^{(\ell)}_{k,i}
:=\max_{\bj\in \prod_{k'\neq k}\cA_{k'}^{(\ell)}}
\widehat{\cT}^{(\ell)}_{(i,\bj)},
\qquad
\mu^{(\ell)}_{k,i}
:=\max_{\bj\in \prod_{k'\neq k}\cA_{k'}^{(\ell)}}
\cT^{(\ell)}_{(i,\bj)}.
\]
Then
\[
\bigl|\hat\mu^{(\ell)}_{k,i}-\mu^{(\ell)}_{k,i}\bigr|
\;\le\;
\Ainfnorm{\widehat{\cT}-\cT^{\star}}.
\]
In particular, on the event
\[
F_\ell(t):=\left\{\Ainfnorm{\widehat{\cT}-\cT^{\star}}\le t\right\},
\]
we have
\[
\bigl|\hat\mu^{(\ell)}_{k,i}-\mu^{(\ell)}_{k,i}\bigr|\le t.
\]
\end{lemma}

\subsubsection{Proof of Lemma \ref{lem:max-Lip}}
\begin{proof}
Let
\[
S:=\prod_{k'\neq k}\cA_{k'}^{(\ell)},
\qquad
A_{\bj}:=\widehat{\cT}^{(\ell)}_{(i,\bj)},
\qquad
B_{\bj}:=\cT^{(\ell)}_{(i,\bj)}
\quad (\bj\in S).
\]
Set
\[
\Delta:=\max_{\bj\in S}|A_{\bj}-B_{\bj}|.
\]
For any \(\bj\in S\),
\[
A_{\bj}\le B_{\bj}+|A_{\bj}-B_{\bj}|\le B_{\bj}+\Delta.
\]
Taking maxima over \(\bj\in S\) yields
\[
\max_{\bj\in S} A_{\bj}\le \max_{\bj\in S} B_{\bj}+\Delta.
\]
By symmetry (swapping \(A\) and \(B\)),
\[
\max_{\bj\in S} B_{\bj}\le \max_{\bj\in S} A_{\bj}+\Delta.
\]
Combining the two inequalities gives
\[
\left|\max_{\bj\in S}A_{\bj}-\max_{\bj\in S}B_{\bj}\right|\le \Delta.
\]
Substituting the definitions of \(A_{\bj}\) and \(B_{\bj}\), we obtain
\[
\bigl|\hat\mu^{(\ell)}_{k,i}-\mu^{(\ell)}_{k,i}\bigr|
\le
\max_{\bj\in S}
\bigl|\widehat{\cT}^{(\ell)}_{(i,\bj)}-\cT^{(\ell)}_{(i,\bj)}\bigr|.
\]
Since every \((i,\bj)\in \cA^{(\ell)}\), the right-hand side is bounded by
\[
\Ainfnorm{\widehat{\cT}-\cT^{\star}},
\]
which proves the claim. The final statement follows immediately from the definition of \(F_\ell(t)\).
\end{proof}

\newpage
\subsection{Proof of Theorem \ref{thm:gap-dep}}
\label{sec:proof of gap-dep thm}
\subsubsection*{Proof sketch (Theorem~\ref{thm:gap-dep}).}
Fix $\eps>0$ and a switch round $\Lsw\le \ell'(\eps)$.

\smallskip
\noindent\textbf{Stage~I (tensor rounds): safety of globally good rows.}
For each round $\ell$, define the separation margin $D_\ell(\eps)$ between any active
globally $\eps/2$-good row and the lower band of rows with stagewise rank $>d_{\ell+2}$.
Lemma~\ref{lem:stage-sep} shows that on $G_\ell$ (no good rows have been dropped so far),
\[
D_\ell(\eps)\ \ge\ \tfrac12\,\Delta_{d_{\ell+2}}\ \ge\ \eps/2.
\]
Define the roundwise accuracy event
\[
E_\ell:=\Big\{\Ainfnorm{\widehat{\cT}^{(\ell)}-\cT^{(\ell)}}\le \tfrac14\,D_\ell(\eps)\Big\}.
\]
Using the fiber-max Lipschitz property (Lemma~\ref{lem:max-Lip}), $E_\ell$ implies that empirical row
scores preserve the strict ordering between globally $\eps/2$-good rows and the lower band,
so the median prune cannot discard any globally $\eps/2$-good row. Formally,
Lemma~\ref{lem:stageI-safety} shows $G_\ell\cap E_\ell\subseteq G_{\ell+1}$, yielding the recursion
\[
\bP(G_{\Lsw+1}^c)\ \le\ \sum_{\ell=1}^{\Lsw}\bP(E_\ell^c\cap G_\ell).
\]
Assumption~\ref{assump:concentration-appendix} then bounds each term via the conditional tail function
$\Phi_\ell(\cdot)$ evaluated at $D_\ell(\eps)/4$, which is in turn controlled by
$\Delta_{d_{\ell+2}}/8$ and the Stage~I hardness $H_{2,\mathrm I}$, producing the Stage~I failure bound.

\smallskip
\noindent\textbf{Stage~II (vector rounds): reduction to standard Sequential Halving.}
On the event $G_{\Lsw+1}$, the global maximizer entry $x_\star$ remains in the post-switch active set
$\cA^{(\Lsw+1)}$. Conditional on $\mathcal F_{\Lsw+1}$, Stage~II observations are fresh and independent,
so Algorithm~\ref{alg:VectorSH} is a standard $K=|\cA^{(\Lsw+1)}|$-arm Sequential Halving instance.
We apply the general Sequential Halving tail bound \cite{zhao2023revisiting}
and upper bound its instance complexity by the row-based surrogates
$M_{\mathrm{row}}(\eps)$ and $H_{\mathrm{row}}(\eps)$. This comparison uses the structural
fact that $\eps$-good entries must lie in $\eps$-good rows (Lemma~\ref{lem:good-in-good-rows})
and the counting bound (Corollary~\ref{cor:counting}), together with the guarantee that row-goodness
survives Stage~I (Lemma~\ref{lem:row-good-induces-entry-good}).

\smallskip
\noindent\textbf{Combine Stage~I and Stage~II.}
Finally, we decompose
\[
\bP(\mu_\star-\mu_{J_N}>\eps)
\ \le\ 
\bP(G_{\Lsw+1}^c)\ +\ \bP(\mu_\star-\mu_{J_N}>\eps\mid G_{\Lsw+1}),
\]
and plug in the Stage~I and Stage~II bounds to obtain the claimed gap-dependent guarantee.

\subsubsection{Notation}
The dyadic construction yields the standard brackets
\[
d_{\ell_k'+3}\ \le\ g_k(\eps)\ \le\ d_{\ell_k'+2},
\qquad
d_{\ell'+3}\ \le\ g_{\max}(\eps)\ \le\ d_{\ell'+2},
\]
and the aligned gap bounds
\[
\Delta_{k,\,d_{\ell_k'+3}}\ \le\ \eps\ <\ \Delta_{k,\,d_{\ell_k'+2}},
\qquad
\Delta_{\,d_{\ell'+3}}\ \le\ \eps\ <\ \Delta_{\,d_{\ell'+2}}.
\]

Consider \TStage\ first. Fix a switch round \(\Lsw\) with \(\Lsw\le \ell'\).
(Here we use the convention that Stage-I keeps exactly the top \(d_\ell/2\) rows in each mode, with ties broken arbitrarily, so that \(|\cA_k^{(\ell)}|=d_\ell\) for all \(\ell\le \Lsw+1\).)

For each mode \(k\), let the sorted row scores at Stage~\(\ell\) be
\[
\mu_{k,(1)}^{(\ell)} \ge \mu_{k,(2)}^{(\ell)} \ge \cdots \ge \mu_{k,(d_\ell)}^{(\ell)}.
\]
For each mode \(k\) and tolerance \(\eta>0\), define the global and stagewise good-row counts
\[
g_k(\eta)
\;:=\;
\bigl|\{i\in[d]:\ \mu_{k,i}\ge \mu_{k,(1)}-\eta\}\bigr|,
\qquad
g_{\max}(\eta)\;:=\;\max_{k\in[m]} g_k(\eta),
\]
\[
g_{k,\ell}(\eta)
\;:=\;
\bigl|\{i\in\cA_k^{(\ell)}:\ \mu_{k,i}\ge \mu_{k,(1)}-\eta\}\bigr|,
\qquad
g_{\max,\ell}(\eta)\;:=\;\max_{k\in[m]} g_{k,\ell}(\eta),
\]
where \(\cA_k^{(\ell)}\subseteq[d]\) denotes the active index set in mode \(k\) at the beginning of round \(\ell\).
At initialization, \(\cA_k^{(1)}=[d]\) for all \(k\), hence
\[
g_{k,1}(\eta)=g_k(\eta),
\qquad
g_{\max,1}(\eta)=g_{\max}(\eta).
\]

For each mode \(k\) and stage \(\ell\), define the minimal margin between any \(\eps/2\)-good row and any
row in the lower band (stagewise ranks \(>d_{\ell+2}\)) at Stage~\(\ell\):
\[
D_{k,\ell}(\eps)
:=
\min_{\substack{1\le i\le g_{k,\ell}(\eps/2)\\ d_{\ell+2}+1\le j\le d_\ell}}
\big(\mu^{(\ell)}_{k,(i)}-\mu^{(\ell)}_{k,(j)}\big),
\qquad
D_\ell(\eps):=\min_{k\in[m]} D_{k,\ell}(\eps).
\]

For every globally $\eps / 2$-good row $i$ in mode $k$ that is active at round $\ell$, there exists a maximizer $x^{(i)}$ for $\mu_{k, i}$ whose coordinates are all globally $\eps / 2$-good (hence active under $G_{\ell}$ ). Equivalently, on $G_{\ell}, \mu_{k, i}^{(\ell)}=\mu_{k, i}$ for all globally $\eps/ 2$-good $i \in \cA_k^{(\ell)}$.
By Lemma~\ref{lem:stage-sep}, for \(\ell\le \ell'\) we have
\begin{equation}\label{eq:stageI-D-lower}
D_{k,\ell}(\eps)\ \ge\ \tfrac12\,\Delta_{k,d_{\ell+2}}
\ \ge\ \tfrac12\,\Delta_{d_{\ell+2}}
\ \ge\ \tfrac{\eps}{2},
\qquad
D_\ell(\eps)\ \ge\ \tfrac12\,\Delta_{d_{\ell+2}}.
\end{equation}

We now define the events that \TStage\ preserves all globally \(\eps/2\)-good rows.
For each round \(\ell\ge 1\), let
\[
E_\ell
\;:=\;
\Big\{
\Ainfnorm{\widehat{\cT}^{(\ell)}-\cT^{(\ell)}}
\ \le\ \tfrac14\,D_\ell(\eps)
\Big\},
\]
where \(\cT^{(\ell)}\) is the restriction of \(\cT^\star\) to the active set \(\cA^{(\ell)}\), and
\(\Ainfnorm{\cdot}\) denotes the sup-norm over \(\cA^{(\ell)}\).

Define inductively the “no-loss so far’’ events
\[
G_1 := \Omega,
\qquad
G_{\ell+1}
\;:=\;
G_\ell
\cap
\Big\{\text{all globally \(\eps/2\)-good rows in every mode survive round \(\ell\)}\Big\}.
\]
Equivalently, \(G_{\ell+1}\) is the event that, after completing rounds \(1,\dots,\ell\), every row
\(i\) with \(\mu_{k,i}\ge \mu_{k,(1)}-\eps/2\) in any mode \(k\) is still active, i.e.
\[
G_{\ell+1}
\quad\Longrightarrow\quad
g_{k,\ell+1}(\eps/2) = g_k(\eps/2)
\text{ for all }k\in[m],
\quad\text{and hence}\quad
g_{\max,\ell+1}(\eps/2) = g_{\max}(\eps/2).
\]

The next lemma shows that the roundwise restricted \(\ell_\infty\) control is sufficient to guarantee \(G_{\ell+1}\).
Recall for every mode $k\in[m]$, one has $\mu_{k,(1)}=\mu_\star$, where
$\mu_\star:=\max_{\bidx\in[d]^m}\cT^\star_{\bidx}$ and
$\mu_{k,i}:=\max_{\bj\in[d]^{m-1}}\cT^\star_{(i,\bj)}$.
Indeed, $\mu_{k,i}\le \mu_\star$ for all $i$, while choosing $i=x^\star_k$ from a maximizer
$x^\star\in\arg\max_{\bidx}\cT^\star_{\bidx}$ gives $\mu_{k,i}\ge \cT^\star_{x^\star}=\mu_\star$.

\begin{lemma}[Stage-I safety under restricted $\ell_\infty$ control]
\label{lem:stageI-safety}
Fix $1\le \ell\le \Lsw\le \ell'$.
Recall that Stage~I prunes each mode-$k$ active set $\cA_k^{(\ell)}$ by keeping exactly the
$\lceil|\cA_k^{(\ell)}|/2\rceil$ indices with the largest empirical scores
$\{\hat\mu^{(\ell)}_{k,i}: i\in \cA_k^{(\ell)}\}$, breaking ties arbitrarily so that
$|\cA_k^{(\ell+1)}|=d_{\ell+1}$.

Let $G_\ell$ denote the event that all {globally} $\eps/2$-good rows in every mode
are active at the beginning of round $\ell$, i.e.
\[
G_\ell
:=\bigcap_{k=1}^m
\Big\{\{i\in[d]:\ \mu_{k,i}\ge \mu_{k,(1)}-\eps/2\}\subseteq \cA_k^{(\ell)}\Big\},
\]
and let
\[
E_\ell
:=\Big\{\ \Ainfnorm{\widehat{\cT}^{(\ell)}-\cT^{(\ell)}}\le \tfrac14\,D_\ell(\eps)\ \Big\},
\]
where $\cT^{(\ell)}$ is the restriction of $\cT^\star$ to $\cA^{(\ell)}$ and
$D_\ell(\eps):=\min_{k\in[m]}D_{k,\ell}(\eps)$ with
$D_{k,\ell}(\eps)$ defined as in Lemma~\ref{lem:stage-sep}.
(Under $\ell\le\ell'$, Lemma~\ref{lem:stage-sep} ensures $D_\ell(\eps)>0$.)

Then on the event $G_\ell\cap E_\ell$, no globally $\eps/2$-good row is discarded at round
$\ell$ in any mode. Equivalently,
\[
G_\ell\cap E_\ell\ \subseteq\ G_{\ell+1}.
\]
Consequently,
\[
G_{\ell+1}^c\ \subseteq\ G_\ell^c\ \cup\ E_\ell^c,
\qquad\text{and}\qquad
G_{\ell+1}^c\cap G_\ell\ \subseteq\ E_\ell^c\cap G_\ell.
\]
In particular,
\begin{equation}\label{eq:stageI-event-inclusion}
\bP\!\left(G_{\ell+1}^c \mid G_\ell\right)
\ \le\
\bP\!\left(E_\ell^c \mid G_\ell\right).
\end{equation}
\end{lemma}

% ============================================================
% Standard recursion / union bound step (polished)
% ============================================================
By Lemma~\ref{lem:stageI-safety}, we have
\[
G_{\ell+1}^c\ \subseteq\ G_\ell^c\ \cup\ (E_\ell^c\cap G_\ell),
\]
and hence
\[
\bP(G_{\ell+1}^c)
\ \le\
\bP(G_\ell^c)\ +\ \bP(E_\ell^c\cap G_\ell).
\]
Iterating this recursion over $\ell=1,\dots,\Lsw$ yields the Stage-I failure union bound
\begin{equation}\label{eq:stageI-failure-union}
\bP(G_{\Lsw+1}^c)
\ \le\
\sum_{\ell=1}^{\Lsw}\bP(E_\ell^c\cap G_\ell).
\end{equation}

We now invoke the roundwise conditional tensor-completion tail bound in Assumption~\ref{assump:concentration-appendix}.
Let \(\{\mathcal F_\ell\}_{\ell\ge1}\) be the filtration from that assumption.
Since \(\cA^{(\ell)}\) is \(\mathcal F_\ell\)-measurable, both \(D_\ell(\eps)\) and \(G_\ell\) are \(\mathcal F_\ell\)-measurable.
Therefore, by the tower property, Assumption~\ref{assump:concentration-appendix} and Equation \eqref{eq:stageI-D-lower},
\begin{align}
\bP(E_\ell^c\cap G_\ell)
&=
\mathbb E\!\left[\mathbf 1_{G_\ell}\bP(E_\ell^c\mid \mathcal F_\ell)\right]\nonumber\\
&\le
\mathbb E\!\left[\mathbf 1_{G_\ell}\,
\Phi_\ell\!\left(\tfrac14 D_\ell(\eps)\right)\right]\nonumber\\
&\le
\Phi_\ell\!\left(\tfrac18 \Delta_{d_{\ell+2}} \right),\label{eq:stageI-cond-step}
\end{align}
where
\[
\Phi_\ell(t)
=
C_m \exp\!\left\{
-\,c\,t^2\,
\frac{p_\ell\,\lambda_{\min}^{\,2}}
{\sigma^2\,\df_\ell\,\log d_\ell\,\infnorm{\cT^\star}^2}
\right\}.
\]

Applying \eqref{eq:stageI-cond-step} with \(t=\tfrac18 \Delta_{d_{\ell+2}}\) and using \eqref{eq:stageI-D-lower} yields
\begin{align}
\bP(E_\ell^c\cap G_\ell)
&\le
C_m \exp\!\Bigg\{-\,\frac{c}{64}\,\Delta_{d_{\ell+2}}^{2}\,
\frac{p_\ell\lambda_{\min}^2}{\sigma^2\,\df_\ell\,\log d_\ell \,\infnorm{\cT^\star}^2}\Bigg\}.
\label{eq:stageI-E-ell-bound}
\end{align}

Define the Stage-I instance complexity
\[
H_{2,\mathrm{I}}
\;:=\;
\max_{2\le t\le d}\ \frac{t}{\Delta_t^{2}}.
\]
Then for every $\ell$ we have $\Delta_{d_{\ell+2}}^2 \ge d_{\ell+2}/H_{2,\mathrm{I}}$.
Since $d_{\ell+2}=d_\ell/4$ and
\[
\df_\ell \;=\; m(d_\ell r-r^2)+r^m \;\le\; mr\,d_\ell+r^m,
\]
it follows that
\begin{equation}\label{eq:ratio-df-lower}
\frac{d_{\ell+2}}{\df_\ell}
\;=\;
\frac{d_\ell/4}{\df_\ell}
\;\ge\;
\frac{d_\ell/4}{mr\,d_\ell+r^m}
\;=\;
\frac{1}{4\left(mr+\frac{r^m}{d_\ell}\right)}.
\end{equation}
Plugging $\Delta_{d_{\ell+2}}^2 \ge d_{\ell+2}/H_{2,\mathrm{I}}$ and
\eqref{eq:ratio-df-lower} into \eqref{eq:stageI-E-ell-bound} yields: there exists a universal
constant $c_0>0$ such that, for every $\ell\in[\Lsw]$,
\begin{equation}\label{eq:stageI-E-ell-bound-simple}
\bP(E_\ell^c\cap G_\ell)
\;\le\;
C_m\exp\!\left\{
-\,c_0\,
\frac{p_\ell\,\lambda_{\min}^2}
{\sigma^2\,H_{2,\mathrm{I}}\,\log d_\ell\,\infnorm{\cT^\star}^2}
\cdot
\frac{1}{\left(mr+\frac{r^m}{d_\ell}\right)}
\right\}.
\end{equation}

To obtain a uniform bound over rounds, note that $d_\ell\ge d_{\min}:=\min_{\ell\le\Lsw}d_\ell$
and $\log d_\ell\le \log d$ for all $\ell\le\Lsw$. Hence,
\[
mr+\frac{r^m}{d_\ell}
\ \le\
mr+\frac{r^m}{d_{\min}},
\qquad
\log d_\ell \ \le\ \log d,
\]
and therefore \eqref{eq:stageI-E-ell-bound-simple} implies the bound
\begin{equation}\label{eq:stageI-E-ell-bound-uniform}
\bP(E_\ell^c\cap G_\ell)
\;\le\;
C_m\exp\!\left\{
-\,c_0\,
\frac{p_\ell\,\lambda_{\min}^2}
{\sigma^2\,H_{2,\mathrm{I}}\,\log d\,\infnorm{\cT^\star}^2}
\cdot
\frac{1}{\left(mr+\frac{r^m}{d_{\min}}\right)}
\right\}.
\end{equation}

Let \(p_{\min}:=\min_{1\le \ell\le \Lsw} p_\ell\). Combining \eqref{eq:stageI-failure-union} with
\eqref{eq:stageI-E-ell-bound-uniform} gives
\begin{equation}\label{eq:stageI-final}
\bP\!\big(G_{\Lsw+1}^c\big)
\ \le\
C_m \Lsw\,
\exp\!\left\{
-\,c_0\,
\frac{p_{\min}\,\lambda_{\min}^2}
{\sigma^2\,H_{2,\mathrm{I}}\,\log d\,\infnorm{\cT^\star}^2}
\cdot
\frac{1}{\left(mr+\frac{r^m}{d_{\min}}\right)}
\right\}.
\end{equation}

Equivalently, if the Stage-I sampling rates satisfy
\begin{equation}\label{eq:stageI-budget-cond}
p_{\min}
\ \ge\
K\,\Big(\frac{\sigma}{\lambda_{\min}}\Big)^2\,
\bigl(\log d\bigr)\,
\infnorm{\cT^\star}^{\,2}\,
H_{2,\mathrm{I}}\,
\Bigl(mr+\frac{r^m}{d_{\min}}\Bigr)
\end{equation}
for a sufficiently large universal constant \(K>0\), then
\(\bP(G_{\Lsw+1}^c)\le C_m\Lsw\exp\{-\widetilde{\Theta}(1)\}\), with the exponent scaling as in
\eqref{eq:stageI-final}.

\subsubsection{Comparison to the vector case.}
In the classical vector best-arm identification setting, one observes noisy
rewards for $n$ independent arms with means $(\mu_i)_{i=1}^n$, and Stage-I
elimination procedures (such as successive rejects or sequential halving) operate
directly on arm-wise empirical means. To bound the probability that too many
$\eps$-good arms are eliminated, the analysis typically conditions on the
population order, enumerates all possible choices of (i) a set $A$ of $\eps$-good
arms that are mistakenly discarded and (ii) a set $B$ of suboptimal arms that are
incorrectly retained, and then applies a union bound over all such pairs
$(A,B)$. This combinatorial union produces binomial factors of the form
$\binom{g(\eps)}{t}\binom{n-q}{s-q+t}$ and the tail bounds are driven by
arm-wise concentration inequalities (often using independence across arms).

In our tensor setting, Stage~I does not estimate each row or arm separately.
Instead, we fit a single tensor completion model $\widehat{\cT}^{(\ell)}$ at each
round and assume a {global} $\ell_\infty$ tail bound on its error
(Assumption~\ref{assump:completion-tail}). By the Lipschitz property of the
row-wise marginals (Lemma~\ref{lem:max-Lip}), this global bound simultaneously
controls the deviations of {all} row scores $\hat\mu^{(\ell)}_{k,i}$ in all
modes. As a consequence, on the event that
$\infnorm{\widehat{\cT}^{(\ell)}-\cT^\star}$ is small, no $\eps/2$-good row can
be misordered below any lower-band row in any mode. Thus every $\eps/2$-good row
must survive Stage~$\ell$, and the Stage-I failure event $G_{\ell+1}^c$ is already
contained in the single ``large $\ell_\infty$ error'' event. This allows us to
bypass the combinatorial union over subsets $(A,B)$ used in the vector case and to control Stage~I directly via the global tensor tail bound.

%%%%%%%%%%%%%%%%%%%%%%%%%%%%%%%%%%%%%%%%%%%%%%%%%%%%%%%%%%%%%
\subsection{The Second Stage (vector rounds $\ell\ge \Lsw +1$)}\label{sec:stageII}

In this subsection, we analyze the entrywise (vector) Sequential Halving phase.
Throughout, we work on the high-probability event
\[
G_{\Lsw+1}
\ :=\
\Bigl\{\text{no globally $\eps/2$-good row in any mode is dropped during \TStage}\Bigr\},
\]
from \cref{lem:stageI-safety}.

\begin{corollary}[Row-based Stage-II guarantee for Sequential Halving]
\label{cor:stageII-row-based}
Work on the event $G_{\Lsw+1}$, so that the globally best entry $x_\star$ belongs to the
post-switch active set $\cA^{(\Lsw+1)}$. Assume that, conditionally on $\mathcal F_{\Lsw+1}$,
Stage-II pulls are fresh and independent, and each observation noise is $\sigma^2$-sub-Gaussian.

Recall the row gaps and row good-count:
\[
\Delta_t:=\min_{k\in[m]}\bigl(\mu_{k,(1)}-\mu_{k,(t)}\bigr),
\qquad
g_{\max}(\eta):=\max_{k\in[m]}\bigl|\{i:\ \mu_{k,i}\ge \mu_{k,(1)}-\eta\}\bigr|.
\]
Define the row-based complexities
\[
M_{\mathrm{row}}(\eps)
:=\max_{t\ \ge\ \left\lceil\left(g_{\max }(\eps / 2)+1\right)^{1 / m}\right\rceil}\ \frac{t^m}{\Delta_t^{\,2}},
\qquad
H_{\mathrm{row}}(\eps)
:=\frac{1}{g_{\max}(\eps/2)}\,M_{\mathrm{row}}(\eps).
\]
Then there exist universal constants $C_0,c_0>0$ such that if the Stage-II budget satisfies
\(
\Nps \ \ge\ C_0\, M_{\mathrm{row}}(\eps),
\)
the Sequential Halving output $J_N$ on $\cA^{(\Lsw+1)}$ obeys
\[
\bP\!\Big(\mu_\star-\mu_{J_N}>\eps \ \Big|\ G_{\Lsw+1}\Big)
\ \le\
\exp\!\Big(-\widetilde{\Theta}\big(\Nps/H_{\mathrm{row}}(\eps)\big)\Big),
\]
where $\widetilde{\Theta}(\cdot)$ hides absolute constants and logarithmic factors in
$|\cA^{(\Lsw+1)}|$.
\end{corollary}

\begin{proof}
Condition on $\mathcal F_{\Lsw+1}$ and on the event $G_{\Lsw+1}$.
Let
\[
\cS:=\cA^{(\Lsw+1)}, \qquad K:=|\cS|.
\]
For each $x\in\cS$, write its (true) mean as $\mu_x:=\cT^\star_x$, and let
\[
\mu_\star:=\max_{x\in\cS}\mu_x.
\]
(Under $G_{\Lsw+1}$, the global maximizer $x_\star$ lies in $\cS$, hence $\mu_\star$ is the global
optimum $\max_{x\in[d]^m}\cT^\star_x$.)
Define the arm gaps
\[
\delta_x:=\mu_\star-\mu_x \ \ge 0,
\]
and the post-switch good-count function
\[
g_{\post}(\eta):=\bigl|\{x\in\cS:\ \delta_x\le \eta\}\bigr|,\qquad \eta\ge 0.
\]
By the Stage-II sampling assumption, conditional on $\mathcal F_{\Lsw+1}$ the observations in Stage II
match the usual stochastic bandit model with $K$ arms, means $\{\mu_x\}_{x\in\cS}$, and
$\sigma^2$-sub-Gaussian noise.

Define the generalized inverse (gap-quantile) function
\[
q_{\post}(s):=\inf\{\eta\ge 0:\ g_{\post}(\eta)\ge s\},\qquad s\in\{1,\dots,K\}.
\]
This is the smallest gap threshold that covers at least $s$ post-switch arms.
(Equivalently, if one sorts post-switch gaps increasingly, $q_{\post}(s)$ equals the $s$-th smallest
gap, but we will not use that notation.)

Now define the post-switch “undivided” and “divided” hardness parameters
\[
M_{\post}(\eps):=\max_{s\ \ge\ g_{\post}(\eps)+1}\ \frac{s}{q_{\post}(s)^2},
\qquad
H_{\post}(\eps):=\frac{1}{g_{\post}(\eps/2)}\,M_{\post}(\eps).
\]

By \cite{zhao2023revisiting}, there exist universal constants
$c,C>0$ such that Sequential Halving with budget $\Nps$ satisfies
\[
\bP\!\Big(\mu_\star-\mu_{J_N}>\eps \ \Big|\ \mathcal F_{\Lsw+1}\Big)
\ \le\
\exp\!\Big(-\widetilde{\Theta}\big(\Nps/H_{\post}(\eps)\big)\Big)
\quad\text{provided that}\quad
\Nps \ \ge\ C\,M_{\post}(\eps).
\]
Thus it remains to show that (i) $M_{\post}(\eps)\le M_{\mathrm{row}}(\eps)$ and
(ii) $H_{\post}(\eps)\le H_{\mathrm{row}}(\eps)$ on $G_{\Lsw+1}$.
This will imply that the assumed budget condition $\Nps\ge C_0 M_{\mathrm{row}}(\eps)$
ensures $\Nps\ge C M_{\post}(\eps)$ (for $C_0$ large enough), and then the exponent
$\Nps/H_{\post}(\eps)$ is at least $\Nps/H_{\mathrm{row}}(\eps)$.

We will show that for every integer $s\ge 1$, with $t:=\lceil s^{1/m}\rceil$,
\begin{equation}\label{eq:qpost-lower}
q_{\post}(s)\ \ge\ \Delta_t.
\end{equation}

Fix $\eta\ge 0$. Define also the global row good-counts (as in your notation)
\[
g_k(\eta):=\bigl|\{i\in[d]:\ \mu_{k,i}\ge \mu_{k,(1)}-\eta\}\bigr|.
\]
By the structural fact “good entries live in good rows” (Lemma~\ref{lem:good-in-good-rows}
and Corollary~\ref{cor:counting}), the number of entries whose gap is at most $\eta$ is bounded by
\[
g_{\post}(\eta)\ \le\ \prod_{k=1}^m g_k(\eta).
\]
Now fix $t\in[d]$ and any $\eta<\Delta_t=\min_k(\mu_{k,(1)}-\mu_{k,(t)})$.
Then for each mode $k$, at most $t-1$ rows can have row-gap $\le \eta$, i.e. $g_k(\eta)\le t-1$.
Hence
\[
g_{\post}(\eta)\ \le\ \prod_{k=1}^m g_k(\eta)\ \le\ (t-1)^m.
\]
Take $t=\lceil s^{1/m}\rceil$, so that $(t-1)^m < s$. Then for every $\eta<\Delta_t$ we have
$g_{\post}(\eta) < s$, which by the definition of $q_{\post}(s)$ implies $q_{\post}(s)\ge \Delta_t$.
This proves \eqref{eq:qpost-lower}.

Let $s\ge g_{\post}(\eps)+1$ and set $t=\lceil s^{1/m}\rceil$.
By \eqref{eq:qpost-lower}, $q_{\post}(s)\ge \Delta_t$, hence
\[
\frac{s}{q_{\post}(s)^2}
\ \le\
\frac{s}{\Delta_t^2}
\ \le\
\frac{t^m}{\Delta_t^2}.
\]
Therefore,
\[
M_{\post}(\eps)
=\max_{s\ge g_{\post}(\eps)+1}\frac{s}{q_{\post}(s)^2}
\ \le\
\max_{t\ge \lceil (g_{\post}(\eps)+1)^{1/m}\rceil}\frac{t^m}{\Delta_t^2}.
\]
On $G_{\Lsw+1}$, we have (your Lemma~\ref{lem:row-good-induces-entry-good})
\[
g_{\post}(\eps)\ \ge\ g_{\post}(\eps/2)\ \ge\ g_{\max}(\eps/2),
\]
so $\lceil (g_{\post}(\eps)+1)^{1/m}\rceil \ge \lceil (g_{\max}(\eps/2)+1)^{1/m}\rceil$.
Thus the tail-max over $t\ge \lceil (g_{\post}(\eps)+1)^{1/m}\rceil$ is bounded by the larger
tail-max over $t\ge \lceil (g_{\max}(\eps/2)+1)^{1/m}\rceil$, yielding
\[
M_{\post}(\eps)\ \le\ M_{\mathrm{row}}(\eps)
\qquad\text{on }G_{\Lsw+1}.
\]

Again on $G_{\Lsw+1}$, Lemma~\ref{lem:row-good-induces-entry-good} gives
$g_{\post}(\eps/2)\ge g_{\max}(\eps/2)$, hence
\[
H_{\post}(\eps)
=\frac{1}{g_{\post}(\eps/2)}\,M_{\post}(\eps)
\ \le\
\frac{1}{g_{\max}(\eps/2)}\,M_{\mathrm{row}}(\eps)
=
H_{\mathrm{row}}(\eps).
\]

Choose $C_0$ large enough so that $\Nps\ge C_0 M_{\mathrm{row}}(\eps)$ implies
$\Nps\ge C M_{\post}(\eps)$. Then the Sequential Halving theorem yields, conditional on
$\mathcal F_{\Lsw+1}$,
\[
\bP\!\Big(\mu_\star-\mu_{J_N}>\eps \ \Big|\ \mathcal F_{\Lsw+1}\Big)
\ \le\
\exp\!\Big(-\widetilde{\Theta}\big(\Nps/H_{\post}(\eps)\big)\Big)
\ \le\
\exp\!\Big(-\widetilde{\Theta}\big(\Nps/H_{\mathrm{row}}(\eps)\big)\Big).
\]
Finally, since $G_{\Lsw+1}\in\mathcal F_{\Lsw+1}$, conditioning further on $G_{\Lsw+1}$ preserves the
same bound, proving the corollary.
\end{proof}

%%%%%%%%%%%%%%%%%%%%%%%%%%%%%%%%%%%%%%%%%%%%% Proof of Lemma %%%%%%%%%%%%%%%%%%%%%%%%%%%%%%%%%%%%%%%%%%%%%
\subsubsection{Proof of Lemma \ref{lem:stageI-safety}}\label{sec:proof of lemma stageI-safety}
\begin{proof}
Fix a mode \(k\in[m]\) and a round \(\ell\le \Lsw\). Work on the event \(G_\ell\cap E_\ell\).

Let \(i\in[d]\) be any globally \(\eps/2\)-good row in mode \(k\), i.e.
\[
\mu_{k,i}\ \ge\ \mu_{k,(1)}-\eps/2.
\]
Since \(G_\ell\) holds, this row is still active, so \(i\in \cA_k^{(\ell)}\).

By definition of \(\mu_{k,i}\), there exists an index tuple
\[
x^{(i)}=(x_1^{(i)},\dots,x_m^{(i)})\in[d]^m
\quad\text{with}\quad
x_k^{(i)}=i
\]
such that
\[
\mu_{x^{(i)}}=\mu_{k,i}.
\]
Since \(\mu_{k,i}\ge \mu_\star-\eps/2\), the entry \(x^{(i)}\) is \(\eps/2\)-good:
\[
\mu_\star-\mu_{x^{(i)}}\le \eps/2.
\]
For each mode \(h\in[m]\),
\[
\mu_{h,x_h^{(i)}}\ \ge\ \mu_{x^{(i)}}
\quad\Longrightarrow\quad
\mu_\star-\mu_{h,x_h^{(i)}}\ \le\ \mu_\star-\mu_{x^{(i)}}\ \le\ \eps/2.
\]
Hence every coordinate \(x_h^{(i)}\) is a globally \(\eps/2\)-good row in its mode. Since \(G_\ell\) holds, all such rows survive to round \(\ell\), so
\[
x^{(i)}\in \cA^{(\ell)}.
\]
Therefore the maximizer \(x^{(i)}\) is still feasible in the restricted row maximum, and thus
\[
\mu_{k,i}^{(\ell)}
=
\max_{\bj\in \prod_{k'\neq k}\cA_{k'}^{(\ell)}} \cT^\star_{(i,\bj)}
=
\mu_{k,i}.
\]

On the other hand, for any active row \(u\in \cA_k^{(\ell)}\),
\[
\mu_{k,u}^{(\ell)}
\le
\mu_{k,u},
\]
since the restricted maximization is over a subset of the full index set.

It follows that the globally \(\eps/2\)-good active rows in mode \(k\) are contained among the top
\(g_{k,\ell}(\eps/2)\) rows in the stagewise ordering
\[
\mu_{k,(1)}^{(\ell)} \ge \mu_{k,(2)}^{(\ell)} \ge \cdots \ge \mu_{k,(d_\ell)}^{(\ell)}.
\]
(Indeed, each globally \(\eps/2\)-good active row has restricted score at least \(\mu_{k,(1)}-\eps/2\), while any active row that is not globally \(\eps/2\)-good has restricted score strictly below \(\mu_{k,(1)}-\eps/2\).)

Fix any stagewise rank \(q\ge d_{\ell+2}+1\), and let \(j_q\in \cA_k^{(\ell)}\) denote a row index attaining the \(q\)-th stagewise score:
\[
\mu_{k,j_q}^{(\ell)} = \mu_{k,(q)}^{(\ell)}.
\]
From previous result, the row \(i\) has stagewise rank at most \(g_{k,\ell}(\eps/2)\). Hence, by the definition of \(D_{k,\ell}(\eps)\),
\[
\mu_{k,i}^{(\ell)}-\mu_{k,j_q}^{(\ell)}
\ \ge\
D_{k,\ell}(\eps).
\]
By \eqref{eq:stageI-D-lower},
\[
D_{k,\ell}(\eps)\ \ge\ D_\ell(\eps)\ >\ 0.
\]

By Lemma~\ref{lem:max-Lip} (applied on the active set),
\[
\big|\hat\mu^{(\ell)}_{k,u}-\mu^{(\ell)}_{k,u}\big|
\;\le\;
\Ainfnorm{\widehat{\cT}^{(\ell)}-\cT^{(\ell)}}
\qquad\text{for all }u\in \cA_k^{(\ell)}.
\]
On \(E_\ell\),
\[
\Ainfnorm{\widehat{\cT}^{(\ell)}-\cT^{(\ell)}} \le \frac{D_\ell(\eps)}{4}.
\]
Therefore, for every \(q\ge d_{\ell+2}+1\),
\begin{align*}
\hat\mu^{(\ell)}_{k,i}-\hat\mu^{(\ell)}_{k,j_q}
&\ge
\big(\mu^{(\ell)}_{k,i}-\mu^{(\ell)}_{k,j_q}\big)
-\big|\hat\mu^{(\ell)}_{k,i}-\mu^{(\ell)}_{k,i}\big|
-\big|\hat\mu^{(\ell)}_{k,j_q}-\mu^{(\ell)}_{k,j_q}\big|\\
&\ge
D_{k,\ell}(\eps)-2\cdot \frac{D_\ell(\eps)}{4}\\
&\ge
D_\ell(\eps)-\frac{D_\ell(\eps)}{2}
\;=\;
\frac{D_\ell(\eps)}{2}
\;>\;0.
\end{align*}
Thus every globally \(\eps/2\)-good active row in mode \(k\) has strictly larger empirical score than every row whose stagewise true rank is \(q\ge d_{\ell+2}+1\).

Hence any globally \(\eps/2\)-good active row can be outranked (in empirical score) only by rows among the top \(d_{\ell+2}\) stagewise true ranks. Therefore each such row has empirical rank at most \(d_{\ell+2}\), and in particular it is among the top \(d_{\ell+1}=d_\ell/2\) rows retained by the pruning step (since \(d_{\ell+2}=d_\ell/4 < d_{\ell+1}\)).

Because the choice of mode \(k\) and globally \(\eps/2\)-good row \(i\) was arbitrary, no globally \(\eps/2\)-good row is discarded in any mode at round \(\ell\). This proves
\[
G_\ell\cap E_\ell \subseteq G_{\ell+1}.
\]
Taking complements gives
\[
G_{\ell+1}^c \subseteq G_\ell^c \cup E_\ell^c.
\]
Finally, intersecting both sides with \(G_\ell\) yields
\[
G_{\ell+1}^c\cap G_\ell \subseteq E_\ell^c\cap G_\ell.
\]
If \(\bP(G_\ell)>0\), divide by \(\bP(G_\ell)\) to obtain \eqref{eq:stageI-event-inclusion}. (If \(\bP(G_\ell)=0\), the conditional inequality is trivial.)
\end{proof}

\begin{lemma}[Stagewise separation of good rows from the lower band]
\label{lem:stage-sep}
Fix a mode $k\in[m]$ and a Stage-I round $\ell$ such that $\ell\le \ell'$
(where $\ell'$ is defined so that $\min_{h\in[m]}\Delta_{h,d_{\ell+2}}>\eps$).
Let the active row set in mode $k$ at the beginning of round $\ell$ be
$\cA_k^{(\ell)}\subseteq[d]$ with $|\cA_k^{(\ell)}|=d_\ell$.

Define the {global} row scores
\[
\mu_{k,i}:=\max_{\bj\in[d]^{m-1}} \cT^\star_{(i,\bj)},
\qquad i\in[d],
\]
and the {restricted (stagewise)} row scores
\[
\mu_{k,i}^{(\ell)}:=\max_{\bj\in\prod_{h\neq k}\cA_h^{(\ell)}} \cT^\star_{(i,\bj)},
\qquad i\in\cA_k^{(\ell)}.
\]
Let $\mu_{k,(1)}\ge\cdots\ge\mu_{k,(d)}$ be the order statistics of the global scores
$\{\mu_{k,i}:i\in[d]\}$, and define the global mode-$k$ gaps
\[
\Delta_{k,t}:=\mu_{k,(1)}-\mu_{k,(t)},\qquad t=1,\dots,d.
\]

Let $\mu_{k,(1)}^{(\ell)}\ge\cdots\ge \mu_{k,(d_\ell)}^{(\ell)}$
denote the order statistics of the {restricted} scores
$\{\mu_{k,i}^{(\ell)}: i\in\cA_k^{(\ell)}\}$.

For $\eta>0$, define the {active good-row count measured by global scores}
\[
g_{k,\ell}(\eta):=\bigl|\{i\in\cA_k^{(\ell)}:\ \mu_{k,i}\ge \mu_{k,(1)}-\eta\}\bigr|.
\]

Let $G_\ell$ be the event that {all globally $\eps/2$-good rows in every mode
are still active at the beginning of round $\ell$}, i.e.
for every mode $h\in[m]$,
\[
\{i\in[d]:\ \mu_{h,i}\ge \mu_{h,(1)}-\eps/2\}\ \subseteq\ \cA_h^{(\ell)}.
\]

Assume $g_{k,\ell}(\eps/2)\ge 1$ (this holds on $G_\ell$).
Define the stagewise separation margin
\[
D_{k,\ell}(\eps)
:=
\min_{\substack{1\le i\le g_{k,\ell}(\eps/2)\\ d_{\ell+2}+1\le j\le d_\ell}}
\Big(\mu_{k,(i)}^{(\ell)}-\mu_{k,(j)}^{(\ell)}\Big).
\]
Since $\ell \leq \ell^{\prime}(\eps)$ implies $\Delta_{k, d_{\ell+2}}>\eps$ for every $k$, any $\eps / 2$-good row must have rank at most $d_{\ell+2}-1$. Hence on $G_{\ell}$,
\begin{align*}
g_{k, \ell}(\eps / 2)=g_k(\eps / 2) \leq d_{\ell+2}-1 \text {, }
\end{align*}
so the definition of $D_{k, \ell}(\eps)$ is nondegenerate and $D_{k, \ell}(\eps)>0$.
If $t \geq d_{\ell+2}$, then $\Delta_{k, t} \geq \Delta_{k, d_{\ell+2}}>\eps>\eps / 2$, so rank $t$ cannot be $\eps / 2$-good.
Then on the event $G_\ell$,
\[
D_{k,\ell}(\eps)\ \ge\ \Delta_{k,d_{\ell+2}}-\frac{\eps}{2}
\ \ge\ \frac12\,\Delta_{k,d_{\ell+2}}
\ \ge\ \frac{\eps}{2}.
\]
Consequently, with $\Delta_t:=\min_{h\in[m]}\Delta_{h,t}$ and
$D_\ell(\eps):=\min_{h\in[m]} D_{h,\ell}(\eps)$, we have on $G_\ell$,
\[
D_\ell(\eps)\ \ge\ \Delta_{d_{\ell+2}}-\frac{\eps}{2}
\ \ge\ \frac12\,\Delta_{d_{\ell+2}}
\ \ge\ \frac{\eps}{2}.
\]
\end{lemma}

\begin{proof}
Fix $k$ and $\ell\le \ell'$, and work on the event $G_\ell$.

\medskip
Let $i\in\cA_k^{(\ell)}$ be globally $\eps/2$-good in mode $k$, i.e.
$\mu_{k,i}\ge \mu_{k,(1)}-\eps/2$.
Choose a maximizer $x^{(i)}\in[d]^m$ for $\mu_{k,i}$ with $x_k^{(i)}=i$ so that
\[
\cT^\star_{x^{(i)}}=\mu_{k,i}.
\]
For any mode $h\in[m]$,
\[
\mu_{h,x_h^{(i)}}
=\max_{\bj\in[d]^{m-1}}\cT^\star_{(x_h^{(i)},\bj)}
\ \ge\ \cT^\star_{x^{(i)}}
=\mu_{k,i}
\ \ge\ \mu_{k,(1)}-\eps/2.
\]
Since $\mu_{k,(1)}=\mu_\star$ and likewise $\mu_{h,(1)}=\mu_\star$,
this implies $\mu_{h,x_h^{(i)}}\ge \mu_{h,(1)}-\eps/2$ for every $h$,
so each coordinate $x_h^{(i)}$ is a globally $\eps/2$-good row in mode $h$.
By $G_\ell$, all such rows are active at round $\ell$, hence
$x^{(i)}\in \cA^{(\ell)}$.
Therefore the maximizer $x^{(i)}$ is feasible for the restricted maximization, and
\[
\mu_{k,i}^{(\ell)}
=\max_{\bj\in\prod_{h\neq k}\cA_h^{(\ell)}} \cT^\star_{(i,\bj)}
\ \ge\ \cT^\star_{x^{(i)}}=\mu_{k,i}.
\]
Since always $\mu_{k,i}^{(\ell)}\le \mu_{k,i}$ (restriction can only decrease a max),
we conclude $\mu_{k,i}^{(\ell)}=\mu_{k,i}$ for every globally $\eps/2$-good active row $i$.

\medskip
By definition of $g_{k,\ell}(\eps/2)$, there are at least
$g_{k,\ell}(\eps/2)$ indices $i\in\cA_k^{(\ell)}$ with
$\mu_{k,i}\ge \mu_{k,(1)}-\eps/2$.
By Step~1, each such index satisfies $\mu_{k,i}^{(\ell)}=\mu_{k,i}\ge \mu_{k,(1)}-\eps/2$.
Hence at least $g_{k,\ell}(\eps/2)$ restricted scores are $\ge \mu_{k,(1)}-\eps/2$,
which implies for every $1\le i\le g_{k,\ell}(\eps/2)$,
\begin{equation}\label{eq:top-restricted-lb}
\mu_{k,(i)}^{(\ell)}\ \ge\ \mu_{k,(1)}-\frac{\eps}{2}.
\end{equation}

\medskip
For $i\in\cA_k^{(\ell)}$ define $a_i:=\mu_{k,i}^{(\ell)}$ and $b_i:=\mu_{k,i}$.
Then $a_i\le b_i$ for all $i$.
Therefore, the $j$-th largest value among $\{a_i:i\in\cA_k^{(\ell)}\}$ is at most
the $j$-th largest value among $\{b_i:i\in\cA_k^{(\ell)}\}$, which is at most
the $j$-th largest value among $\{b_i:i\in[d]\}$, i.e.
\[
\mu_{k,(j)}^{(\ell)}\ \le\ \mu_{k,(j)}.
\]
Hence for any $j\ge d_{\ell+2}+1$,
\begin{equation}\label{eq:lower-restricted-ub}
\mu_{k,(j)}^{(\ell)}
\ \le\ \mu_{k,(j)}
\ \le\ \mu_{k,(d_{\ell+2}+1)}
\ \le\ \mu_{k,(d_{\ell+2})}
\ =\ \mu_{k,(1)}-\Delta_{k,d_{\ell+2}}.
\end{equation}

\medskip
Take any admissible pair
$1\le i\le g_{k,\ell}(\eps/2)$ and $d_{\ell+2}+1\le j\le d_\ell$.
Using \eqref{eq:top-restricted-lb} and \eqref{eq:lower-restricted-ub},
\[
\mu_{k,(i)}^{(\ell)}-\mu_{k,(j)}^{(\ell)}
\ \ge\
\Big(\mu_{k,(1)}-\frac{\eps}{2}\Big)
-\Big(\mu_{k,(1)}-\Delta_{k,d_{\ell+2}}\Big)
=
\Delta_{k,d_{\ell+2}}-\frac{\eps}{2}.
\]
Taking the minimum over all admissible pairs gives
\[
D_{k,\ell}(\eps)\ \ge\ \Delta_{k,d_{\ell+2}}-\frac{\eps}{2}.
\]

Finally, since $\ell\le \ell'$ implies $\Delta_{k,d_{\ell+2}}>\eps$ (by definition of $\ell'$),
we have
\[
\Delta_{k,d_{\ell+2}}-\frac{\eps}{2}
\ \ge\ \frac12\,\Delta_{k,d_{\ell+2}}
\ \ge\ \frac{\eps}{2}.
\]
This proves the first claim. Minimizing over $k$ yields the stated bound for $D_\ell(\eps)$.
\end{proof}

\begin{lemma}[Row-good rows induce surviving good entries]\label{lem:row-good-induces-entry-good}
On the event \(G_{\Lsw+1}\), for every mode \(k\in[m]\),
\[
g_{\post}\!\left(\tfrac{\eps}{2}\right)\ \ge\ g_k\!\left(\tfrac{\eps}{2}\right).
\]
Consequently,
\[
g_{\post}\!\left(\tfrac{\eps}{2}\right)\ \ge\ g_{\max}\!\left(\tfrac{\eps}{2}\right),
\qquad
g_{\post}(\eps)\ \ge\ g_{\post}\!\left(\tfrac{\eps}{2}\right)\ \ge\ g_{\max}\!\left(\tfrac{\eps}{2}\right).
\]
\end{lemma}

\begin{proof}
Fix \(k\in[m]\), and work on \(G_{\Lsw+1}\). For each row index
\(i\in[d]\) with \(\mu_{k,i}\ge \mu_{k,(1)}-\eps/2\), by the definition of \(\mu_{k,i}\) there exists an entry
\[
x^{(i)}=(x_1^{(i)},\dots,x_m^{(i)})\in[d]^m
\quad\text{with}\quad
x_k^{(i)}=i
\quad\text{and}\quad
\mu_{x^{(i)}}=\mu_{k,i}\ge \mu_\star-\eps/2.
\]
Hence \(x^{(i)}\) is an \(\eps/2\)-good entry in the full tensor.

By Lemma~\ref{lem:good-in-good-rows}, every coordinate of \(x^{(i)}\) lies in an \(\eps/2\)-good row (in its corresponding mode). Since \(G_{\Lsw+1}\) is the event that all globally \(\eps/2\)-good rows survive Stage I, all coordinates of \(x^{(i)}\) remain active at the switch, and therefore
\[
x^{(i)}\in \cA^{(\Lsw+1)}.
\]
Thus \(x^{(i)}\) is a post-switch \(\eps/2\)-good entry, i.e.
\[
\delta_{x^{(i)}}\le \eps/2.
\]

Moreover, the map \(i\mapsto x^{(i)}\) is injective for fixed mode \(k\), because the \(k\)-th coordinate of \(x^{(i)}\) equals \(i\). Therefore the number of post-switch \(\eps/2\)-good entries is at least the number of \(\eps/2\)-good rows in mode \(k\):
\[
g_{\post}(\eps/2)\ \ge\ g_k(\eps/2).
\]
Since \(k\) was arbitrary, taking the maximum over \(k\) gives
\(g_{\post}(\eps/2)\ge g_{\max}(\eps/2)\). Finally \(g_{\post}(\eps)\ge g_{\post}(\eps/2)\) by monotonicity in \(\eps\).
\end{proof}

\begin{lemma}[Good entries live in good rows]\label{lem:good-in-good-rows}
Fix a mode \(k\in[m]\) and any index vector \(\fx=(x_1,\ldots,x_m)\in [d]^m\).
Then there exists an index \(t\in[d]\) such that \(\mu_{k,x_k}=\mu_{k,(t)}\), and hence
\[
\delta_{\fx}
\;=\;
\mu_\star-\mu_{\fx}
\;\ge\;
\mu_\star-\mu_{k,x_k}
\;=\;
\mu_{k,(1)}-\mu_{k,(t)}
\;=\;
\Delta_{k,t}.
\]
Consequently, if \(\delta_{\fx}\le \eps\), then necessarily \(\Delta_{k,t}\le \eps\).
In words, every \(\eps\)-good entry must lie in a mode-\(k\) row whose gap is at most \(\eps\).
\end{lemma}

\begin{corollary}[Counting bound]\label{cor:counting}
Let
\[
H(\eps)
\;:=\;
\bigl|\{x:\ \delta_x\le \eps\}\bigr|,
\qquad
g_k(\eps)
\;:=\;
\bigl|\{i\in[d]:\ \mu_{k, i} \geq \mu_{k,(1)}- \eps\}\bigr|.
\]
Then, for any instance and any \(\eps\ge 0\),
\[
H(\eps)\ \le\ \prod_{k=1}^m g_k(\eps).
\]
\end{corollary}

\begin{proof}
If \(\delta_x\le\eps\), then by Lemma~\ref{lem:good-in-good-rows} we must have
\(\Delta_{k,i_k}\le\eps\) for all modes \(k\), where \(x=(i_1,\dots,i_m)\).
Therefore \(x\) can only use indices from the sets
\(\{i:\Delta_{k,i}\le\eps\}\) in each mode \(k\), of cardinalities \(g_k(\eps)\)
respectively. The Cartesian product of these sets has size \(\prod_k g_k(\eps)\),
which bounds \(H(\eps)\).
\end{proof}

%\bibliographystyle{ormsv080R}
%\bibliography{ref}

\end{document}